\documentclass[nohyperref]{article}

\usepackage[accepted]{icml2023}
\usepackage{natbib}

\usepackage[normalem]{ulem}

\usepackage{dsfont}
\usepackage{mathtools}
\usepackage{amssymb,bm}
\usepackage{cancel}
\usepackage{amsthm}
\usepackage{amsmath}
\usepackage[colorlinks = true,
            linkcolor = blue,
            urlcolor  = blue,
            citecolor = blue]{hyperref}
\usepackage{thmtools}
\usepackage{thm-restate}
\usepackage{caption}
\usepackage{subcaption}
\usepackage{comment}
\usepackage{booktabs}
\usepackage{multirow}
\usepackage{xcolor}

\DeclarePairedDelimiter\abs{\lvert}{\rvert}%
\newcommand{\norm}[1]{\left\lVert#1\right\rVert}

\DeclareMathOperator*{\argmin}{arg\,min}

\theoremstyle{plain}
\newtheorem{theorem}{Theorem}[section]
\newtheorem{proposition}[theorem]{Proposition}
\newtheorem{lemma}[theorem]{Lemma}

\theoremstyle{definition}
\newtheorem{definition}[theorem]{Definition}
\newtheorem{assumption}{Assumption}
\theoremstyle{remark}

\usepackage{ulem}
\usepackage{mathtools}
\usepackage{amssymb}

\newcommand{\R}{{\mathbb{R}}}

\usepackage[capitalize,noabbrev]{cleveref}

\begin{document}

%

%

\twocolumn[

\icmltitle{Learning the Dynamics of Sparsely Observed Interacting Systems 
}

\icmlsetsymbol{equal}{*}

\begin{icmlauthorlist}
\icmlauthor{Linus Bleistein}{equal,inria,crc,evry}
\icmlauthor{Adeline Fermanian}{equal,mines,curie,inserm,califrais}
\icmlauthor{Anne-Sophie Jannot}{inria,aphp}
\icmlauthor{Agathe Guilloux}{inria,crc}
\end{icmlauthorlist}

\icmlaffiliation{inria}{Inria Paris, F-75015 Paris, France}
\icmlaffiliation{crc}{Centre de Recherche des Cordeliers, INSERM, Université de Paris, Sorbonne Université, F-75006 Paris, France}
\icmlaffiliation{mines}{MINES ParisTech, PSL Research University, CBIO, F-75006 Paris, France}
\icmlaffiliation{curie}{Institut Curie, PSL Research University, F-75005 Paris, France}
\icmlaffiliation{inserm}{INSERM, U900, F-75005 Paris, France}
\icmlaffiliation{evry}{LaMME, UEVE and UMR 8071, Paris Saclay University, F-91042, Evry, France}
\icmlaffiliation{aphp}{AP-HP, Paris, France}
\icmlaffiliation{califrais}{LOPF, Califrais’ Machine Learning Lab, Paris, France}

\icmlcorrespondingauthor{Linus Bleistein}{linus.bleistein@inria.fr}

\icmlkeywords{Signature Methods, Time series, Dynamical Systems, Applications to Healthcare.}

\vskip 0.3in
]



\printAffiliationsAndNotice{\icmlEqualContribution} 

\begin{abstract}
We address the problem of learning the dynamics of an unknown non-parametric system linking a target and a feature time series. The feature time series is measured on a sparse and irregular grid, while we have access to only a few points of the target time series. Once learned, we can use these dynamics to predict values of the target from the previous values of the feature time series. We frame this task as learning the solution map of a controlled differential equation (CDE). By leveraging the rich theory of signatures, we are able to cast this non-linear problem as a high-dimensional linear regression. We provide an oracle bound on the prediction error which exhibits explicit dependencies on the individual-specific sampling schemes. Our theoretical results are illustrated by simulations which show that our method outperforms existing algorithms for recovering the full time series while being computationally cheap. We conclude by demonstrating its potential on real-world epidemiological data. 
\end{abstract}

\section{Introduction}

Time series are ubiquitous in many areas such as finance, economics, robotics, agriculture, and healthcare. One is typically interested in modelling the evolution of a target quantity through time, which is known to be affected by a set of time-evolving features. For example, pollution levels in a city are driven by quantities such as temperature, pressure, traffic, or economic activity measured through time. Mathematically, one wishes to model the evolution of a quantity $y_t \in \R^p$, $p \geq 1$, as a function of some time evolving features $x_t \in \R^d$, $d \geq 1$, for $t \in [0,1]$. In other words, the goal is to learn the dynamics that link the target to the features.

Such an interaction is typically modelled via differential equations, which are a common choice of model in natural sciences~\citep{zwillinger1998handbook}. In this article, we assume that there exists a function $\mathbf{G}: \R^p \times \R^d \to \R^p$ such that
\begin{equation} \label{eq:ode}
     y_t = y_0 + \int_0^t \mathbf{G}(y_s, x_s)ds
\end{equation}
or equivalently
\[
dy_t = \mathbf{G}(y_t,x_t)dt, \quad y_0 \in \mathbb{R}^p.
\]
The value $y_t$ depends on the trajectory of the features time series $x_s$ up to time $t$. Learning the dynamics of the system can be framed as learning the solution map of~\eqref{eq:ode}, i.e., a function $\Psi$ which, given a time $t$, an initial point $y_0 \in \mathbb{R}^p$, and the history of the path up to time $t$, denoted by $x_{[0,t]} = (x_s)_{s \in [0,t]}$, outputs the value of $y$ at time $t$.

If we know $\Psi$, we gain access to the values of $y$ at any point in time provided we know the values of $x$ up to this point ; this encompasses many tasks such as forecasting or interpolating between points of $y$. We specifically have in mind applications where we have an easy access to $x$ but a limited one to $y$.

This problem is extremely common in healthcare. For example, in obstetrics, the lactic acidosis (LA) of the fetus, which is a proxy for fetal distress, is a quantity of high medical interest for predicting complications in the first hours after birth. This biomarker cannot be measured during pregnancy but only at birth because the measurement is highly invasive. Some vitals such as heart rate and fetal movement are however easy to measure during pregnancy. In this case, $x$ is the non-invasive measurements made during pregnancy, while $y$ is the invasive measurement of LA at birth. Predicting the value of $y$ at any time $t$ (both before and at birth) would allow for early diagnosis. Similarly, after surgery, patients are often monitored to detect hemorrhage. While some vitals such as heart rate of saturation are monitored in continuous time, haemoglobin---which is highly predictive of hemorrhage---is only measured by blood samples taken a few times a day, which can significantly delay hemorrhage diagnostic.

\paragraph{Irregular data.} In practice, the functions $x$ and $y$ are measured on discrete grids and take the form of time series. These often present a lot of heterogeneity, both within and across individuals.
\begin{itemize}
    \item[$(i)$] For every individual, the time between any two measurements can vary, and thus individuals may not be recorded on the same grid.
    \item[$(ii)$] The number of total sampling points might vary between individuals. 
    \item[$(iii)$] Each measurement in time might be corrupted by measurement noise.
\end{itemize}
 Mathematically, we consider $n$ pairs of functions $\{(x^{1}, y^{1}), \dots, (x^{n}, y^{n})\}$. Each $x^i$ deterministically produces a specific $y^i$ through the Ordinary Differential Equation (ODE)~\eqref{eq:ode}. We call $x^{i}$ the feature path and $y^{i}$ the target path. Both $x^{i}$ and $y^{i}$ are only observed at a finite set of times specific to every individual. We denote by 
\[D^{i} = \left(t^i_1,\dots,t^i_{k_i} \right), \quad i= 1, \dots, n, \] 
the sampling grid of $x^{i}$ and by $\bar{D}^{i}$ the sampling grid of $y^{i}$. We stress that both the number of sampling times $k_i$ and the sampling times $t^i_{1},\dots,t^i_{k_i}$ themselves are individual specific, as described in $(i)$ and $(ii)$. Moreover, the observations are corrupted by additive noise, such that we observe \[X^i_t = x^i_t + \xi^i_t\] for all $t \in D^i$, and similarly $Y_t^i = y^i_t + \varepsilon^i_t$ for every $t \in \bar{D}^i$, where the $\xi^i_t$ and $\varepsilon^i_t$ are sub-gaussian i.i.d. random vectors. Each input may therefore be written as a matrix $\mathbf{X}^i = (X^i_t)_{t \in D^{i}} \in \mathbb{R}^{k_i \times d}$ which we call the feature time series. Similarly, the quantity of interest is a matrix $\mathbf{Y}^{i}=(Y^i_t)_{t \in \bar{D}^{i}} \in \mathbb{R}^{m_i \times p }$ (where $m_i$ is the length of $\bar{D}^{i}$) and is called the target time series. The grid $\bar{D}^{i}$ is assumed to be a subset of $D^i$: in our setup $y^{i}$ is hard to sample and therefore measured at only a few points (and sometimes only one) while $x^{i}$ is easy to access and measured at high frequency. Our goal is to approximate the dynamics linking  $x$ and $y$ from the irregular, heterogeneous, and fuzzy data $\mathbf{X}^{i}$ and $\mathbf{Y}^{i}$.

Such heterogeneity is difficult to handle by classical machine learning algorithms such as Long short-term memory networks~\citep[LSTM,][]{hochreiter1997long} which assume that the data is regularly sampled. Some more recent approaches ~\citep{rubanova2019latent,de2019gru,kidger2020neural,herrera2021neural} have adapted these models by introducing continuously evolving hidden states to account for the irregular spacing between observation times. 

We build upon the approach of Neural Controlled Differential Equations ~\citep[Neural CDE,][]{kidger2020neural,morrill2021neuralrough}, which have proven to be very successful for time series classification and online prediction tasks \citep{morrill2021neural}. The key idea of Neural CDE is that under some fairly mild assumptions, any general ordinary differential equation of the form~\eqref{eq:ode} can be rewritten as
\begin{equation} \label{eq:cde}
    y_t = y_0 + \int_0^t \mathbf{F}(y_s) dx_s,
\end{equation}
where $\mathbf{F}: \R^p \to \R^{p \times d}$ is a matrix-valued vector field, such that the right-hand-side of~\eqref{eq:cde} is a matrix-vector product~\citep[see, e.g., ][Proposition 2, for a proof]{fermanian2021framing}. The function $x$ is often called the driver of the CDE. In a Neural CDE setting, the driver $x$ is a continuous interpolation of the feature times series, $y$ corresponds to a continuously-evolving state, and $\mathbf{F}$ is chosen to be a neural network. This network is then trained such that the values of $(y_t)$ can be used as features for classification or regression tasks. While Neural CDE have been shown to outperform other architectures with limited memory usage, their training time is considerable and no statistical guarantees exist.

\paragraph{Model.} We model the interactions between the target and the feature paths through a CDE of the form \eqref{eq:cde}. This modelling choice encapsulates a broad variety of settings, since the vector field $\mathbf{F}$ can be any (regular enough) function. A priori, the solution map $\Psi$ of this CDE is a complex function of time and the history of $x$ up to $t$; however, by linearizing the model, we are able to approximate $\Psi$ by a simple scalar product between a deterministic transformation of the history of $x$, called the signature of $x$ at order $N \geq 1$ and denoted by $S_N(x_{[0,t]})$, and a time independent matrix $\theta_N^\ast$.  Informally, we have 
\[\Psi\big( x_{[0,t]}, t \big) \approx S_N\big(x_{[0,t]} \big)^\top \theta_N^\ast.\] Two striking features of this linearized model are \textit{(i)} that $\theta^*_N$ can be learned on any time horizon $[0,t]$ since it is independent of time, and \textit{(ii)} that once it has been learned, the model can be called at any time $t$. 

\paragraph{Contributions.} Our contributions are threefold. First, we frame the task of learning the interactions between two time series as learning the flow of a CDE, which can be linearized in the signature space. While the connection between CDEs and signatures is well-known, this is the first time CDEs are used as a statistical model. We then leverage this linearization to derive statistical guarantees on the prediction error with an explicit dependence on both sampling irregularities and the noise affecting measurements. To our knowledge, this is the first bound of this type for signature-based models, allowing for better understanding of the dependencies between prediction performance and sampling roughness. Finally, the resulting algorithm, called SigLasso, is shown to be computationally cheap and competitive compared to existing baselines on a wide range of simulated data and a real-world example of hospitalization growth rate prediction during the Covid pandemic.

\paragraph{Related works.} Signatures originated as a prominent tool in stochastic analysis~\citep{chen1958integration, lyons2007differential, friz2010multidimensional} and have proven to be a powerful feature extraction method in machine learning in various domains such as healthcare \citep{morrill2020utilization, wang2020learning}, human action recognition \citep{yang2022leveraging}, or financial modelling~\citep{ lyons2014feature, buehler2020data}. Their appealing properties include a capacity to handle irregular data, to capture dependence between coordinates, and their links with the theory of CDE. We refer to \citet{lyons2022signature} for a recent survey on their use cases. 
However, the statistical properties of signatures based algorithms have received little attention so far, with a few notable exceptions~\citep{papavasiliou2011parameter,lemercier2021distribution,fermanian2022functional}.

On the other hand, the interplay between dynamical systems and machine learning has received considerable attention in the recent years. A first line of work has focused on approximating the solution of ODE and Partial Differential Equations (PDE) with neural networks ~\citep{lagaris1998artificial, han2018solving, zubov2021neuralpde} and directly learning dynamical systems ~\citep{long2018pde,fattahi2019learning}. Recent approaches have been interested in combining deep learning algorithms with physical knowledge~\citep{greydanus2019hamiltonian, brunton2020machine, willard2020integrating}. Finally, dynamical systems, seen as continuous versions of neural network architectures, have also been a great source of inspiration for analysing and designing machine learning algorithms in the recent years ~\citep{chen2018neural,fermanian2021framing, marion2022scaling}. We refer to \citet{kidger2022neural} for an extensive review.

We stress that our problem is different in nature from most problems encountered in the time series literature, since we seek to model the relationship between two sparsely observed systems with heterogeneous sampling. We do not model the evolution of the feature time series, and take it as an input, contrarily to methods such as Gaussian Process. Most models either focus on the case where one time series is observed and forecasted, or on regular sampling, or on univariate time series. Our problem bears close resemblance to frameworks encountered in sequence-to-sequence learning ~\citep{sutskever2014sequence, gehring2017convolutional} and functional regression ~\citep{ramsay1991some, marx1999generalized}. 

\paragraph{Overview.} Section \ref{section:model} introduces the CDE model for interacting systems, the mathematical context and the learning procedure. Our main theoretical result is presented in Section \ref{section:oracle_ineq}. We conclude by an empirical study on synthetic and real-world data in Section \ref{section:experiments}. All proofs are postponed to the appendix and the code to reproduce the experiments is available at \href{https://github.com/LinusBleistein/SigLasso}{https://github.com/LinusBleistein/SigLasso}.

\section{Model and Assumptions}
\label{section:model}

A summary table of all notations introduced in the main body of this article is provided to the reader in Appendix \ref{appendix:definitions}.

\subsection{A CDE-Based Model on the Dynamics}

We start by describing our assumptions on the feature and target paths, which are linked by Equation~\eqref{eq:cde}. To correctly define the integral of Equation~\eqref{eq:cde}, we need to impose some conditions on the $x^i$ and on $\mathbf F$. Note that we consider that the $x^{i}$ are defined on $[0,1]$ but our results extend easily to any compact time interval $[a,b]$. 
\begin{assumption} \label{assump:X_bv}
All paths $(x^i)_{1 \leq i \leq n}$ are continuous and there exists $ 0 < L < 1$ such that, for all $i=1,\dots,n$,
    \[
        \big\|x^i\big\|_{\textnormal{1-var},[0,1]} = \sup_D \sum\limits_{k} \big\|x^i_{t_{k+1}}-x^i_{t_k}\big\| \leq L,
    \]
    where $\norm{\cdot}$ is the Euclidean norm and the supremum is taken over all finite dissections $D = \left\{0 = t_1 < \dots < t_k = 1 \right\}$. 
\end{assumption}
We write $C_L^{\textnormal{1-var}}([0,1],\mathbb{R}^d)$ for the set of continuous paths of total variation bounded by $L$. Outside the statistical context, when referring to general paths, we will drop the superscripts $i$ and simply write $x$ and $y$ to alleviate notations. 

We assume that the target path $y$ is the solution of the ODE~\eqref{eq:ode}. This modelization choice means that the evolution of $y$ is governed by a dynamical system whose dynamics itself are allowed to vary with the current value of the feature path. Observe that this model can be seen as a generalized form of a non-autonomous system~\citep{lyons2007differential}, which we recover by taking $x_t = t$. 
Since Equation~\eqref{eq:ode} can be rewritten as a CDE, the starting point of our work is to assume that the true dynamics of the data follow such a CDE, as stated in the following assumption.

\begin{assumption}
\label{assump:Y_CDE}
There exists a smooth vector field $\mathbf{F}: \R^p \to \R^{p \times d}$ such that, for all $i=1,\dots,n$, $y^{i}$ is the solution of the CDE~\eqref{eq:cde} driven by $x^{i}$ with initial condition $y^i_0 = y_0 \in \mathbb{R}^p$ homogeneous amongst individuals.
\end{assumption}

By ``smooth'' we mean that each coordinate of $\mathbf{F}$ is infinitely differentiable, that is, is $\mathcal{C}^\infty$. The vector field $\mathbf{F}$ and the initial condition $y_0$ are common to all individuals, which can be seen as homogeneity assumptions on our sample. On the other hand, since every individual $i$ has her own feature path $x^{i}$, the target paths $y^{i}$ are individual specific. In other words, there exists a solution map $\Psi$ that depends only on $y_0$ and $\mathbf{F}$ and is such that, for any $t \in [0,1]$, $\Psi(x^{i}_{[0,t]}, t) = y^{i}_t$.  

The vector field $\mathbf{F}$ encapsulates the common physical dynamics governing the evolution of $y^{i}$, which are affected by the changes in $x^{i}$. Note that there is no parametric model on $\mathbf{F}$ (although some strong smoothness requirements will be needed) contrarily to functional or traditional time series models~\citep{ramsay2005functional,morris2015functional}.

\subsection{Linearizing the CDE with Signatures}

Before defining the Taylor expansion of the CDE~\eqref{eq:cde}, which will allow us to linearize the solution map $\Psi$, we need to introduce the notion of signature, which have emerged as a powerful tool to model time series~\citep{levin2013learning,kidger2019deep}.

From now on, for any feature path $x$ with values in $\mathbb R^d$, we denote by $x^{(j)}$ its $j$th coordinate,  for $j = 1, \ldots,d$.

\begin{definition}
\label{def:signature}
Let $x \in C^{\textnormal{\textnormal{1-var}}}_L([0,1],\mathbb{R}^d).$ Take a word of length $k$ from the alphabet $\{1,\dots,d \}$, that is, an element $I=(i_1,\dots,i_k) \in \{ 1,\dots, d \}^k$. For all $t \in [0,1]$, the signature coefficient associated to this word is the scalar
\begin{align*}
    S^{I}\big(x_{[0,t]}\big) &= \idotsint\limits_{0 < u_1 < \dots < u_k < t} dx^{(i_1)}_{u_1}\cdots dx^{(i_k)}_{u_k}.
\end{align*}
\end{definition}

We introduce a series of notation on signature coefficients, grouping them by the length $k$ of the words. For any $k \in \mathbb{N}$ and $t \in [0,1]$, the signature of order $k$ of $x_{[0,t]}$ is
\begin{align*}
    \mathbb{X}_{k, [0,t]} = \big(S^{I}(x_{[0,t]})\big) _{I \in \{1,\dots,d \}^k} \in \mathbb{R}^{d^k},
\end{align*}
where the words $(i_1, \dots, i_k)$ are in lexicographic order. We then denote the full signature by
\begin{align*}
    S(x_{[0,t]}) = \big(1,\mathbb{X}^1_{[0,t]},\dots,\mathbb{X}^k_{[0,t]},\dots\big)
\end{align*}
and, for any $N \geq 1$, the signature truncated at order $N$ by
\begin{align*}
    S_N(x_{[0,t]}) = \big(1,\mathbb{X}^1_{[0,t]},\dots,\mathbb{X}^N_{[0,t]} \big)^\top.
\end{align*}
The size of the signature truncated at order $N$ grows exponentially with $N$, since, for $d \geq 2$, it is equal to
\begin{align*}
    s_d(N) =1 + d + d^2 + \dots + d^N = \frac{d^{N+1}-1}{d-1}.
\end{align*}
When computing signatures, it is common practice to add time as a coordinate to the path~\citep{fermanian2021embedding}, that is, consider the path $(t, x_t)^\top$. From now on, we assume that the first dimension
of $x$ always corresponds to time, so that $d \geq 2$. 

Signatures encode geometric properties of paths and have numerous appealing properties as a feature set for time series. We refer to \citet{chevyrev2016primer,fermanian2021embedding,lyons2022signature} for more detailed introductions to signatures. In order to provide supplementary intuition, we first give three insights on signatures.
\paragraph{A geometric insight.} Signatures are a geometric alternative to representations based on frequency such as the Fourier transform. Indeed, consider a function $f \in C^\infty([0,1],[0,1])$ and the two dimensional path $(t, f(t))$. For simplicity assume that $f(0)=0$. Then, the first order signature coefficients are equal to the last positions $t$ and $f(t)$. The second order signature coefficients are equal to $\int_{0}^t f(s)ds, \frac{1}{2} t^2, \frac{1}{2}f(t)^2$, and $f(t)t-\int_0^t f(s)ds$, and capture how the area under the curve evolves with time.
\paragraph{A computational insight.}Consider the linear path $x_t = a t + b$ for $t \in [0,1]$, where $a = (a_1,\dots,a_d)^\top \in \mathbb{R}^d$. For any word $(i_1,\dots,i_k)$ of size $k$ the associated signature coefficient is \[\frac{a_{i_1}\dots a_{i_k}}{k!}t^k.\] In the linear case, signatures are therefore simply polynomials in $t$ with path-specific coefficients. This result generalizes nicely to piecewise linear paths (via a result known as Chen's Lemma) and allows for computational efficiency when computing the signature. We refer to \citet{kidger2020signatory} for further computational details. 
\paragraph{A functional insight.} Recall that we are interested in approximating functions $f(x_t,t)$ which depend both on time and on the values of the feature path $(x_t)$. When computing signatures, we always consider the time-augmented path $(t,x_t)$. The coefficients will thus be divided in two parts: a first set of coefficients related to the time dimension, and a set of coefficients related to the path dimensions. The time-specific coefficients are simply
    \begin{align}
        t, \frac{t^2}{2!}, \frac{t^3}{3!}, \dots, \frac{t^N}{N!}
    \end{align}
    and thus form a polynomial basis. Roughly speaking, these coefficients approximate the time dependant part of the function $f$. The path-specific coefficients, which can be thought of as polynomials of a path, approximate the part of $f$ depending on $(x_t)$. This highlights that the Taylor development of a CDE, which is the cornerstone of our work, is very similar in nature to the approximation of a function by its classical Taylor development (see Appendix \ref{appendix:analogy_taylor}).

With these insights in hand, we are now ready to properly define the Taylor expansion of a CDE.

\begin{definition}
\label{def:taylor_expansion}
Let $N \geq 1$. The Taylor expansion of order $N$ of the solution $y$ of Equation~\eqref{eq:cde} is defined by 
\begin{align}
    \label{eq:taylor_expansion}
    \overline{y}_{N,t} = y_0 +\sum\limits_{k=1}^N \sum\limits_{I \in \{1, \dots, d\}^k } S^{I}\left(x_{[0,t]}\right) \times \Phi^{I}_{\mathbf{F}}\left(y_0\right),
\end{align}
where $\Phi^{I}_{\mathbf{F}}(\cdot) \in \R^p$ is the differential product of the vector field $\mathbf{F}$ along $I$, whose definition is postponed to Appendix~\ref{appendix:diff_prod}.
\end{definition}
The differential products $\Phi^{I}_{\mathbf{F}}(y_0)$ are essentially a combination of multiplication and summation of derivatives of the different components of $\mathbf{F}$, evaluated at $y_0$. 

The Taylor expansion crucially allows to write the solution map as a product between a time-varying term (the signature of the feature path) and a constant-over-time term (the differential product). This is similar to a regular Taylor expansion; we discuss this analogy in Appendix \ref{appendix:analogy_taylor}.
Note that Equation \eqref{eq:taylor_expansion} may also be written as a matrix-vector product
\begin{align}
    \label{eq:linear_problem}
    \overline{y}_{N,t}{}^\top &= S_N(x_{[0,t]})^\top \theta^\ast_N  \in \mathbb{R}^p,
\end{align}
where $\theta^\ast_N \in \R^{s_d(N) \times p}$ is a matrix collecting all differential products $\Phi^I_{\mathbf{F}}(y_0)$ up to order $N$ and the offset $y_0$. Since $\theta^\ast_N$ depends neither on $x$ nor on $t$, the Taylor expansion of $y$ at order $N$ is simply a linear function of the truncated signature.

For this expansion to become exact in the sens of pointwise convergence, we need quite strong regularity assumption on $\mathbf{F}$. This is the price to pay for the non-parametric nature of our model. 
\begin{assumption}
    \label{assump:decay_derivatives_F}
    The vector field $\mathbf{F}$ has fast decaying derivatives. In other words, defining
    \begin{align*}
    \Lambda_k(\mathbf{F}) = \sup\limits_{I \in \{1, \dots, d\}^k} \norm{\Phi^{I}_{\mathbf{F}}\left(y_0\right)},
\end{align*}
we assume that $\sum_{k=0}^{\infty} d^k \Lambda_k(\mathbf{F}) / k!  < \infty$.
\end{assumption}
Let $y$ be the solution of the CDE~\eqref{eq:cde}, and let $(\overline{y}_{N,t})_{t \in [0,1]}$ be its $N$-th Taylor expansion. Under Assumptions~\ref{assump:X_bv},~\ref{assump:Y_CDE}, and~\ref{assump:decay_derivatives_F}, we then have, for any $t \in [0,1]$,
\begin{align}\label{eq:cvg_taylor}
    \norm{y_t - \overline{y}_{N,t}} \underset{N \to + \infty}{\longrightarrow} 0.
\end{align}
We refer the reader to \citet{friz2008euler} and \citet[][Proposition 4]{fermanian2021framing} for more details on this result.



While we impose relatively strong regularity assumptions on the vector field $\mathbf{F}$, the assumptions on $x$ are mild (see Assumption~\ref{assump:X_bv}). Therefore, our assumptions, while enforcing a fairly high amount of regularity on the dynamical structure, still accommodate most of the real world data. Also note that \citet{fermanian2021framing} give conditions under which Assumption~\ref{assump:decay_derivatives_F} is satisfied when $\mathbf{F}$ is a layer of a neural network with smooth activation functions.

\subsection{The Learning Problem}

\begin{figure*}
    \centering
    \includegraphics[width=0.85\textwidth]{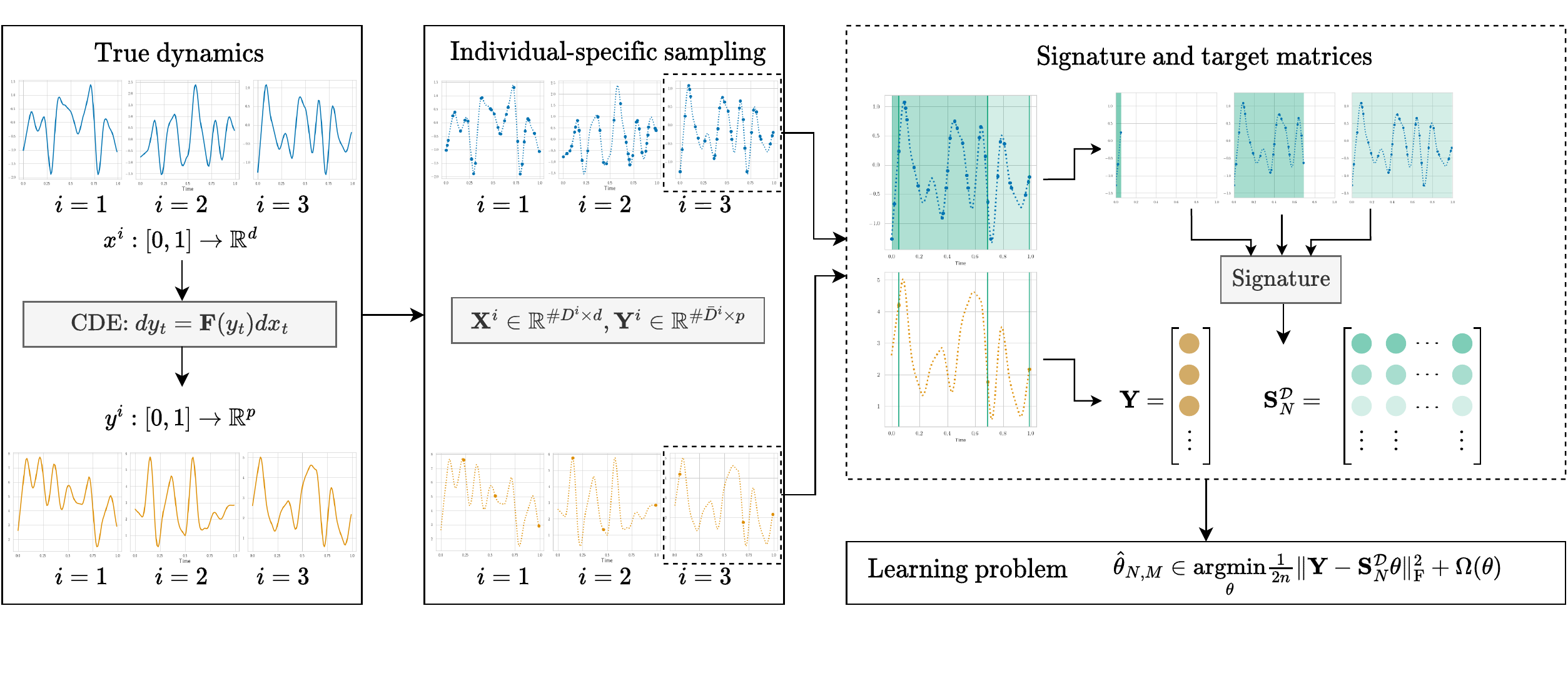}
    \caption{The workflow of our model. Starting from the left, the first panel describes our modelling hypothesis: the target and feature time series are linked through an unknown CDE. The second panel shows the observed data. The third panel shows how every observation of the target is mapped to the signature of the corresponding path, how to construct the dataset $\mathbf{S}_N^\mathcal{D}, \mathbf{Y}$ from this data and finally how to learn the SigLasso estimator.}
    \label{fig:workflow}
\end{figure*}

We go back to the statistical learning problem. Recall that in practice, we do not observe the continuous paths $\{(x^1,y^1),\dots,(x^n,y^n) \}$ but their discretized and fuzzy counterparts $\{(\mathbf{X}^1,\mathbf{Y}^1),\dots, (\mathbf{X}^n,\mathbf{Y}^n)\}$ measured on a set of individual grids $\mathcal{D}=\{D^1,\dots,D^n \}$ and $\bar{\mathcal{D}}=\{\bar{D}^1,\dots,\bar{D}^n \}$. We allow for high data heterogeneity: two individuals can be sampled at very different frequencies and at different time-points, and therefore have different observation grids. 


The meshsize of a sampling grid $D$, denoted by $\abs{D}$ is defined as the largest gap between two successive sampling times, that is, 
\begin{equation*}
    \abs{D} = \max\limits_{t_i \in D} \, \abs{t_{i+1}-t_i}.
\end{equation*}
Its cardinality, denoted by $\#D$, is the number of sampling points in $D$. We make the following assumption on the sampling procedure.
\begin{assumption}
\label{assump:sampling_grid}
There exists $\eta \in [0, 1]$ such that for all $i=1, \dots, n$, one has
\[0 \in D^{i}, \quad  \#D^{i} \geq 2, \quad t^i_{k_i} \geq \eta \quad \text{and} \quad \bar{D}^{i} \subset D^{i}.\]
\end{assumption}
Let us briefly comment on this assumption. We require that the measurements on all individuals start at $0$. We also do not allow for arbitrarily short time series: for every individual $i$, the last observation time $t^i_{k_i}$ must be greater than $\eta$. Finally, we require the sampling grid of $\mathbf{Y}^i$ to be coarser than the sampling grid of $\mathbf{X}^i$. We also let $\abs{\mathcal{D}} = \max_i \abs{D^i}$ be the biggest gap between two successive sampling times within the whole set of individual grids and $\# \mathcal{D} = \sum\limits_{i=1}^n \#D^i$ the total number of sampling points of the feature time series. We recall that the random vectors $\xi_t^i$, $t \in D^i$, are the noises affecting the measurements of $x^i$. The random vectors $\varepsilon^i_t$ affect the measurements of $y^i$. We end with assumptions on the law of these measurement noises.

\begin{assumption}[Noise on the feature time series]
\label{assump:feature_noise}
The noises $(\xi^i_t)_{i\in \{1, \dots, n\}, t \in D^{i}}$ are i.i.d.~$v_\xi$-subgaussian random vectors.
\end{assumption}

\begin{assumption}[Noise on the target time series]
\label{assump:target_noise}
The noises $(\varepsilon^i_t)_{i\in \{1, \dots, n\}, t \in \bar{D}^{i}}$ are i.i.d.~$v_\varepsilon$-subgaussian random vectors and independent from $(\xi^i_t)_{i\in \{1, \dots, n\}, t \in D^{i}}$.

\end{assumption}


The goal of the learning procedure is to learn the solution map $\Psi$, in order to infer the value of $y$ at any time $t$, given observations of $x$ up to time $t$. We have seen previously that this problem boils down to estimating the matrix $\theta_N^\ast$ via Equation \eqref{eq:linear_problem}. The truncation order $N \geq 1$ is an hyperparameter and will be selected by cross-validation.

\paragraph{Single target measurement.} For the sake of simplicity, we first present the case where $Y^{i}$ is measured only at the end of the observation period, as in the example of LA measurements in obstetrics. In this case, for all $i=1,\dots,n$, we have $m_i=1$ such that $M=n$,
\[D^{i} = (t^{i}_1, \dots, t^{i}_{k_i}), \quad \text{and} \quad \bar{D}^{i} = (t^{i}_{k_i}). \]
Then $\mathbf{Y}^{i} \in \R^p$, and we denote by
\begin{equation*}
\mathbf{Y} = \begin{bmatrix}  \mathbf{Y}^{1} , \ldots , \mathbf{Y}^{n}\end{bmatrix}^\top \in \R^{n \times p}
\end{equation*}
the matrix containing all target measurements. Since we do not have access to the feature paths $x^{i}$ but only to the discrete measurements $\mathbf{X}^{i}$, we compute the signature of its linear interpolation normalized by its total variation, sampled up to final time $t^i_{k_i}$, and denote by $\mathbf{S}^{\mathcal{D}}_N \in \mathbb{R}^{n \times s_d(N)}$ the matrix of stacked signatures. 
Note that signatures of piecewise linear functions are fast to compute with packages such as \texttt{signatory}~\citep{kidger2020signatory} or \texttt{iisignature}~\citep{reizenstein2021algorithm}. The complexity to compute the signature truncated at order $N$ of the $i$th feature time series is of the order $\mathcal{O}(\#D^{i} d^N)$.
Finally, we approximate $\theta_N^\ast \in \mathbb{R}^{s_d(N) \times p}$ by solving the optimisation problem 
\begin{align}
\label{eq:optimization_problem}
    \min_{\theta \in \mathbb{R}^{s_d(N) \times p}} \frac{1}{2n}\norm{\mathbf{Y} -\mathbf{S}_N^\mathcal{D} \theta }_{\textnormal{F}}^2 + \Omega(\theta),
\end{align}
where $\Omega: \mathbb{R}^{s_d(N) \times p} \to \R^{+}$ is a regularization term and $\norm{\cdot}_{\textnormal{F}}$ is the Frobenius norm. We have reduced the complex problem of learning the solution map $\Psi$ to a simple penalized linear regression in the signature space. This linear model on the signature is close to the one studied by \citet{levin2013learning,lyons2014feature,fermanian2022functional}. 

\paragraph{Multiple target measurements.} We also cover the case when $\# \bar{D}^i > 1$, that is, the target is measured at multiple times for every individual, as in the example of hemorrhage detection. In this case, we have $\mathbf{Y}^{i} \in \R^{m_i \times p}$ and we stack the different measurement matrices $\mathbf{Y}^i$ to obtain a matrix $\mathbf{Y}$ of size $M \times p$, where $M = m_1 + \dots + m_n$. For any $i=1,\dots,n$ and every $t \in \bar{D}^i$, we predict $Y_t^i$ using the signature of the linear interpolation of the normalized $(X^{i}_{0}, \dots, X^{i}_{t})$. In this manner, we will be able to predict $Y_t^i$ at every point where $X^i_t$ is sampled. The exact workflow of our model is described in Figure \ref{fig:workflow}.

\section{Theoretical Guarantees}
\label{section:oracle_ineq}

\subsection{Mathematical Setup}

We consider a general multiple target measurements setting. To simplify the exposure of our results, we consider a univariate target path, i.e., $p=1$. In this case, the true parameter $\theta_N^*$ is a vector of size $s_d(N)$ and not a matrix. The general case $p \geq 1$, which our algorithm handles as running $p$ Lasso regressions in parallel, is considered in Appendix \ref{appendix:proofs}, and all theoretical results are proved in this general case. In addition, to lighten the presentation of the oracle inequality, we also focus in this section on the case of $\omega$-Lipschitz feature paths, that is, for every $i=1,\dots,n$ and for all $s,t \in [0,1]$, $\norm{x^i_t - x^i_s} \leq \omega \abs{t-s}$. We stress that our results are valid for continuous paths of bounded variations. We let $\mathbf{y} \in \mathbb{R}^{M}$ be the matrix collecting all unobserved values of the target paths at measurement times such that $\mathbf{y} = \mathbb{E}\big(\mathbf{Y}\big),$
where the expectation is taken over the noises $\varepsilon^{i}_t$, and define $\widehat{\theta}_{N,M}$ as
\begin{align}\label{eq:thetahatdef}
     \widehat  \theta_{N,M} \in \argmin\limits_{\theta \in \mathbb{R}^{s_d(N)}} \frac{1}{2M}\norm{\mathbf{Y}-\mathbf{S}_N^{\mathcal{D}}\theta}_{2}^2 + \Omega(\theta).
 \end{align}

For $\delta \in (0,1)$, we define the set 
\begin{align}
\label{eqn:a_xi}
    A_\xi(\delta) =  \Big\{\max\limits \norm{\xi^i_t} \leq \underbrace{v_\xi\sqrt{d} +  v_\xi \sqrt{\frac{1}{c}\log \frac{\# \mathcal{D}}{\delta}}}_{=: C_\delta} \Big\} 
\end{align}
where the maximum is taken on all $i=1,\dots,n$ and $t\in D^i$, and $c$ is a universal constant.
This set is of probability greater than $1-\delta$ under Assumption \ref{assump:feature_noise} (see Appendix \ref{appendix:subgaussian}). Similarly, for $k \geq 0$ and $\bar \delta \in (0,1)$, let 
\[C_k(\bar{\delta}) = \sqrt{v_\varepsilon \log(2Nd^k / \Bar{\delta})} \] 
and define 
\begin{align}
\label{eq:A_epsilon}
    A_\varepsilon (\Bar{\delta})
    &= \bigcap_{k=0}^N \Big\{ \big\|\bm{\varepsilon}^\top\mathbf{S}^\mathcal{D}_{\cdot,[k]}\big\|_\infty \leq \frac{M^{\frac{1}{2}}C_k(\bar{\delta})}{k!} \Big\} ,
\end{align}
where $\mathbf{S}^\mathcal{D}_{\cdot,[k]}$ is the sub-matrix of size $M \times d^k$ of $\mathbf{S}^\mathcal{D}_N$ associated to the signature coefficients of order $k$, and $\bm{\varepsilon} \in \R^M$ is a vector of i.i.d.~noise terms satisfying Assumption \ref{assump:target_noise} (see Appendix \ref{appendix:proof_notations}). Under Assumptions \ref{assump:feature_noise} and \ref{assump:target_noise}, $A_\varepsilon (\Bar{\delta})$ is of probability at least $1-\Bar{\delta}$, and $A_\xi(\delta) \cap A_\varepsilon (\Bar{\delta})$  is of probability at least $(1-\delta)(1-\Bar{\delta})$. Let 
\begin{equation}
\label{eq:penalty}
    \Omega(\theta) = \sum\limits_{k=0}^N \frac{C_k(\bar{\delta})}{k!\sqrt{M}} \big\|\theta_{[k]}\big\|_1,
\end{equation}
where $\theta_{[k]}$ is the subvector of size $d^k$ that collects all elements of $\theta$ associated to words of size $k$ (see Appendix \ref{appendix:proof_notations}). This penalization can be implemented by rescaling the feature matrix $\mathbf{S}_N^\mathcal{D}$ and solving a standard $\ell_1$-penalized regression problem (see Appendix \ref{appendix:layer_penalty}). Our result extends to more general penalties by adapting existing techniques from~\citet{chesneau2008some} for the group-lasso, or~\citet{lederer2019oracle} for the hierarchical lasso. 



\subsection{Main Results}

The error made when learning $\theta_{N}^\ast$ by $\widehat \theta_{N,M}$ comes from three different sources. $(i)$ Truncating the signature used in the regression at depth $N \geq 1$ results in a truncation bias. $(ii)$ Discretization of the feature path and the noise affecting each measurement point induce a discretization error. In particular, there is a trade-off between sampling frequency and variance of the noise. $(iii)$ The measurement error on $\mathbf{y}^{i}$ and the finite-sample setting induce a classical estimation error.

\begin{table*}[t!]
	\centering
 \begin{tiny}
	\caption{Performance of SigLasso, GRU and Neural CDE in different simulation settings, averaged over 10 iterations. In every setting, $n=50$,  $\# \bar{D}^i = 5$ for all $i=1, \dots, n$ (and therefore $M=250$).}
	\label{tab:performance}
	\begin{tabular}{lcccccc}
	\toprule
    & \multicolumn{3}{c}{$L_2$ error}& \multicolumn{3}{c}{MSE on last point} \\
    \cmidrule(lr){2-4}
    \cmidrule(lr){5-7}
    Setting &  SigLasso & GRU & Neural CDE &  SigLasso & GRU & Neural CDE  \\
    \midrule 
    Well-specified & \textbf{0.13 $\pm$ 0.07} & 1.05 $\pm$ 0.42 & 0.61 $\pm$ 0.38  & \textbf{0.73 $\pm$ 0.56} & 3.32 $\pm$ 1.60 &  1.46 $\pm$ 1.20   \\
    Ill-specified & \textbf{0.15 $\pm$ 0.02} & 0.24 $\pm$ 0.11 & 0.29 $\pm$ 0.15 &  \textbf{0.09 $\pm$ 0.05} & 0.19 $\pm$ 0.09 &  0.22 $\pm$ 0.15 \\
    OU  & \textbf{0.01 $\pm$ 0.02} & 0.05 $\pm$ 0.06 & 0.17 $\pm$ 0.12 & 0.018 $\pm$ 0.025 & 0.014 $\pm$ 0.020 & \textbf{0.013 $\pm $ 0.016 }  \\
    Tumor growth & \textbf{0.16 $\pm$ 0.02} & 0.66 $\pm$ 0.09 & 5.29 $\pm$ 1.38 &  \textbf{0.35 $\pm$ 0.12 } & 2.00 $\pm$ 0.38 & 8.76 $\pm$ 9.26\\
	\bottomrule           
	\end{tabular}
 \end{tiny}
\end{table*}

The following lemmas bound each of those errors.
We first bound the variance of the estimator with arguments borrowed from~\citet{bickel2009simultaneous}. 
\begin{lemma}
\label{lemma:estimation_error}
Under Assumptions \ref{assump:X_bv} and \ref{assump:Y_CDE}, on the set $A_\varepsilon(\bar{\delta}) \cap A_\xi(\delta)$, the prediction error
\begin{equation*}
  \frac{1}{2M}\big\|\mathbf{y}-\mathbf{S}_N^{\mathcal{D}} \widehat{\theta}_{N,M}\big\|_{2}^2  
\end{equation*}
is bounded above by
\begin{align*}
  &\frac{1}{2M}\norm{\mathbf{y}-\mathbf{S}_N^{\mathcal{D}} \theta^\ast_N}_{2}^2 + \frac{2C_N(\bar{\delta})}{\sqrt{M}}\sum_{k=0}^N \frac{d^k \Lambda_k(\mathbf F)}{k!}.
\end{align*}
\end{lemma}

See Appendix \ref{proof:bias_variance} for a proof. This inequality decomposes the error into a bias and a variance term. We denote the signature matrix of the unobserved paths $x^{i}$ by $\mathbf{S}_N \in \mathbb{R}^{M \times s_d(N)}$.

As for the bias term, notice that one can write
\begin{align*}
    \frac{1}{2M}\norm{\mathbf{y}-\mathbf{S}_N^{\mathcal{D}} \theta^\ast_N}_{2}^2 & \leq \frac{1}{M} \underbrace{\norm{\mathbf{y} - \mathbf{S}_N \theta^\ast_N}_{2}^2}_{\text{Truncation bias}}  \\
    & + \frac{1}{M}\underbrace{\norm{\mathbf{S}_N \theta^\ast_N - \mathbf{S}_N^{\mathcal{D}} \theta^\ast_N }_{2}^2}_{\text{Discretization error}}.
\end{align*}
Bounding each of these terms corresponds to, respectively, Lemmas~\ref{lemma:bias_bound} and~\ref{lemma:disc_error}. We stress that the truncation bias is of a different nature than the discretization error since it depends on a choice of hyperparameter while the latter is inherent to the data at hand.

\begin{lemma}
\label{lemma:bias_bound}
Under Assumptions \ref{assump:X_bv} and \ref{assump:Y_CDE}, for any $N \geq 1$,
\begin{align*}
    \frac{1}{M}\norm{\mathbf{y}-\mathbf{S}_N \theta^*_N}^2_{2} \leq  \Bigg(\frac{d^{N+1}\Lambda_{N+1}(\mathbf{F})}{(N+1)!}\Bigg)^2.
\end{align*}
\end{lemma}
Under Assumption \ref{assump:decay_derivatives_F}, the right-hand-side decays exponentially fast with $N$.
This lemma is an immediate consequence of \citet{fermanian2021framing} (see Appendix~\ref{proof:bias_bound}). We now turn to the error induced by the discretization of the feature path.
\begin{lemma}
\label{lemma:disc_error}
Under Assumptions \ref{assump:X_bv},  \ref{assump:sampling_grid}, and \ref{assump:feature_noise}, on the set $A_{\xi}(\delta)$, one has
\[
\frac{1}{M}\norm{(\mathbf{S}_N - \mathbf{S}^\mathcal{D}_N)\theta^\ast_N}^2_{2} \leq C_{\mathcal{D},N}(\delta) \sum \limits_{k=0}^N  \frac{d^k \Lambda_k(\mathbf{F})^2}{k!^2},
\]
where $C_{\mathcal{D},N}(\delta)$ is equal to
\begin{align*}
& 4e^2 L^2 N!^2 \times\\ & \Big(\omega \abs{\mathcal{D}} + C_\delta + \frac{1-L+ 2\# \mathcal{D} C_\delta}{\eta}\big(\norm{x_0}+L+C_\delta\big)\Big)^2.
\end{align*}  
\end{lemma}

This lemma relies on a fine analysis of the distance between two signature layers. The dependence of $C_{\mathcal{D},N}(\delta)$ on sampling mechanisms and noise is of particular interest. First, the term $\omega |\mathcal{D}|$ refers to the longest time between two observations amongst individuals. Not sampling an individual during a long period of time causes a loss in information, which is bounded by the Lipschitz control $\omega$ of the feature path. 

The second part of $C_{\mathcal{D},N}(\delta)$ is a consequence of the noises $\xi^{i}_t$ affecting the measurement points of the feature time series and does not vanish with $n$. The emergence of such a bias is a well studied phenomenon in errors-in-variable models, and cannot be corrected without precise knowledge of the noise's variance~\citep{loh2011high}. 

The last term corresponds to a bias coming from the interplay of the normalisation of the signatures by the total variation of the path and the noise, which is a standard practice \cite{morrill2020generalised}. We found that this normalization performs best empirically. Note that this bias term is equal to $0$ if the path has total variation exactly equal to $1$ and is observed without noise.

From Lemmas~\ref{lemma:estimation_error}  to~\ref{lemma:disc_error} and the definitions of $A_\xi$ and $A_\varepsilon$ finally we get the following oracle inequality. The proof is given in Appendix \ref{proof:oracle_bound}. 
\begin{theorem}[An oracle inequality for learning with signatures]
\label{thm:oracle_inequatility}
Under Assumptions \ref{assump:X_bv}, \ref{assump:Y_CDE}, \ref{assump:sampling_grid}, \ref{assump:feature_noise}, and \ref{assump:target_noise}, with probability at least  $(1-\delta)(1-\bar{\delta})$, the prediction error
is bounded above by
\begin{align}
\frac{1}{2M}\norm{\mathbf{y} - \mathbf{S}_N^{\mathcal{D}}\widehat \theta_{N,M}}_2^2 & \leq \Bigg(\frac{d^{N+1}\Lambda_{N+1}(\mathbf{F})}{(N+1)!}\Bigg)^2  \label{term_1}\\
& +C_{\mathcal{D},N}(\delta)  \sum \limits_{k=0}^N  \frac{d^k \Lambda_k(\mathbf{F})^2}{k!^2} \label{term_2}\\
& + \frac{2C_N(\Bar{\delta})}{\sqrt{M}}\sum_{k=0}^N \frac{d^k \Lambda_k(\mathbf F)}{k!} \label{term_3}.
\end{align}
\end{theorem}




All terms depend on the regularity of the vector field $\mathbf{F}$ via the constants $\Lambda_k(\mathbf{F})$: the bigger these constants, the faster the vector field $\mathbf{F}$ may vary, making the CDE harder to predict. The convergence speed in $1/ \sqrt{M}$ is classical. Also note that our bound is non-asymptotic and is valid for any $M \geq 1$. 

The dependence of the bound on $N$ is highly non-trivial and requires an in-depth analysis of the regularity of $\mathbf{F}$ in order to bound $\Lambda_k(\mathbf{F})$, which is out of the scope of this paper. The asymptotic behaviour of this oracle inequality is discussed in Appendix \ref{appendix:assymptotics}. 


\section{Experiments}
\label{section:experiments}

We study the performance of SigLasso obtained by solving the optimization problem \eqref{eq:thetahatdef}, where $\Omega(\theta)$ is defined by Equation~\eqref{eq:penalty}. All details are given in Appendix \ref{appendix:experiments}.

\subsection{Simulations}

We consider several settings of data generation. First, in the well-specified setting, the data is generated from a model with regular feature paths $x$ (piecewise polynomials) and target paths $y$ wich are solutions to the CDE $$dy_t = \text{tanh}(Ay_t )dx_t,$$
where $A$ is a randomly drawn matrix.  In the ill-specified setting, the target $y$ is equal to
$$y_t = \log \| \sum_{h=1}^{10} x_{t-h} \|$$ for any $t \in [0,1]$. In the third setting, called OU setting, the feature paths are realizations of Brownian motions and the target paths are  Ornstein-Uhlenbeck processes~\citep{borodin2015handbook} driven by the feature paths. The last setting corresponds to the tumor growth model from \citet{simeoni2004predictive}. The feature path represents the concentration of a treatment drug, generated as the squared value of the smooth paths used in the well-specified setting, and the target path $y$, the weight of the tumor, is governed by a system of differential equations given in Appendix \ref{appendix:details_tumor_growth}.

\begin{figure}[ht]
    \centering
    \includegraphics[width=0.4\textwidth]{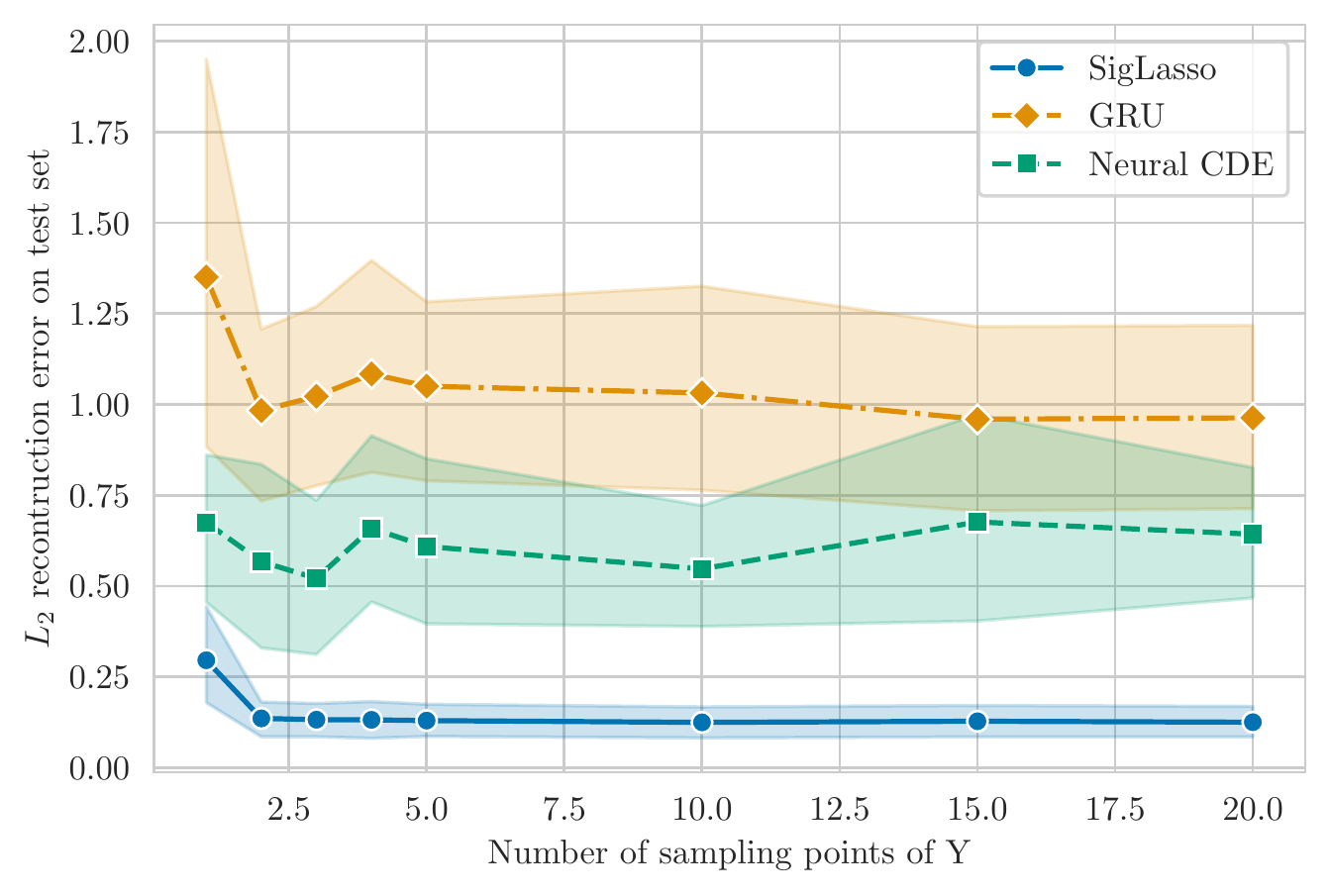}
    \caption{$L_2$ reconstruction error of SigLasso, GRU and Neural CDE in the well-specified setting, for varying number of target samples.}
    \label{fig:sampling_study_smooth}
\end{figure}

We compare SigLasso to a GRU and Neural CDE \citep{kidger2020neural}. We measure the performance of the models with two metrics on a test set: the mean squared error for predicting the last observation point of the target paths and the $L_2$ error for predicting the full path on a fine grid.

The results are shown in Table \ref{tab:performance} and Figure \ref{fig:sampling_study_smooth}. In Figure \ref{fig:sampling_study_smooth} we consider the well-specified setting and vary the number of sampling points of the target paths between 1 and 20. In Table \ref{tab:performance} it is fixed to 5 but the simulation settings change. Sampling of both the target and the feature time series is highly irregular. SigLasso outperforms Neural CDE and GRU models in generalizing from a few learning points of the target to its full trajectory in all settings. We conduct supplementary experiments with RNN and LSTM, which are also outperformed by SigLasso (see Table \ref{table:supplementary_results} in Appendix \ref{appendix:additional_results}). 
An additional byproduct of SigLasso's simple form is its training speed: its is approximately 10 times faster than GRU and 100 times faster than Neural CDE, including cross-validation to select $N$ and regularization strength (see Appendix \ref{appendix:additional_results}).

\subsection{Forecasting the Growth Rate of Hospitalizations in France During the Covid-19 Pandemic}

\begin{figure}
    \centering
    \includegraphics[width=0.40\textwidth]{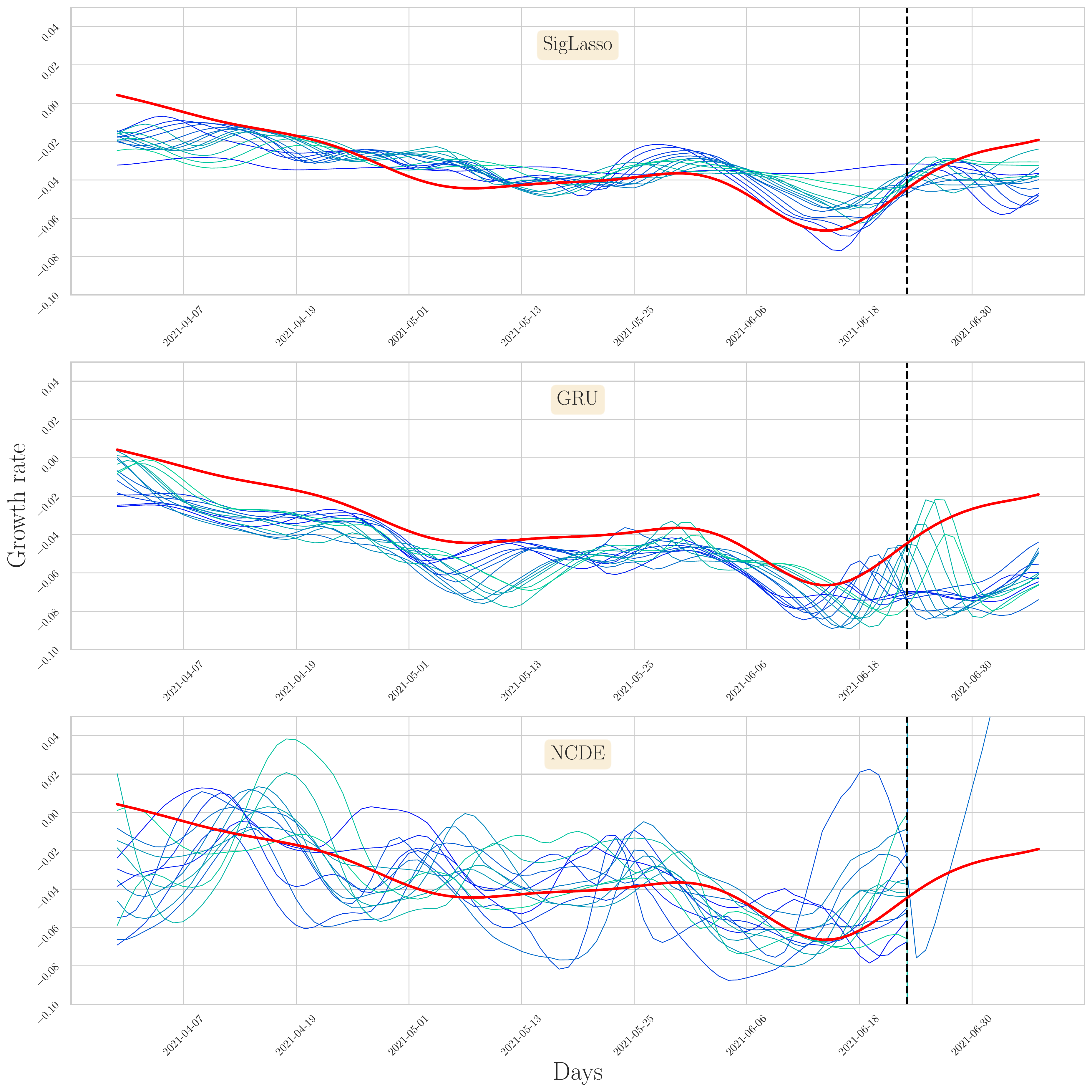}
    \caption{Interpolation (left of dotted line) and prediction (right of dotted line) of HGR in region Île de France for SigLasso, GRU and NCDE. The lighter the blue, the smaller the horizon $h$. Ground truth is in red. NCDE overfits and is unable to predict the HGR during the test period.}
    \label{fig:covid_example}
\end{figure}

Forecasting hospitalizations in real time during the Covid-19 pandemic is a notably difficult task. In this experiment, we train our model to learn the dynamics linking population data related to mobility, vaccination, and weather, and the hospitalization growth rate (HGR) in each of the 9 metropolitan regions of France based on the data of \citet{paireau2022ensemble}. The feature time series is regularly sampled and specific to each region and $12$-dimensional. We consider prediction horizons $h=1,\dots,14$, meaning that we predict the hospital saturation at time $t$ using the history of the feature time series up to $t-h$. Using our notations, we have for each region $d=12$, $p=1$, $n=1$. Concretely, this means that we train one Siglasso model per region and per horizon. The target is sampled every day during the training period left of the dotted line in Figure \ref{fig:covid_example}, and the models are fitted to those values: on this time span, the models learns to interpolate the target time series. The model is then asked to predict the HGR on the days right of the dotted line seing only the feature time series. This means that it performs a prediction task on this time span.

Both GRU and SigLasso learn smooth and precise dynamics, which generalize well above the learning horizon and yield similar prediction performance (see Appendix~\ref{appendix:french_covid}). Our model is slightly outperformed by GRU for $h\geq 7$ for the prediction task, the difference in MSE being always less than $0.08$. It performs similarly or better for most values of $h$ for the interpolation task. Neural CDE and the original method proposed by~\citet{paireau2022ensemble} perform poorly. Figure \ref{fig:covid_example} shows an example of reconstruction and prediction of HGR obtained with SigLasso, GRU and NCDE.

\section{Conclusion}

We have introduced a novel CDE-based model for interacting systems. Drawing on the theory of signatures, we derive an oracle bound that depends explicitly on the roughness of the data sampling. We illustrate the high performance of our approach on synthetic and real-world data.


The obtained theoretical guarantees rely on strong regularity assumptions on the vector field $\mathbf{F}$. The exact approximation properties of this class of vector field are a very interesting direction for future work. Considering other penalties that take into account the underlying structure of $\theta_N^*$ would also be an interesting extension of our work.  
  
\paragraph{Acknowledgement.} We thank anonymous ICML reviewers for their remarks, which helped us improve this paper. LB thanks Gérard Biau and Claire Boyer for supervising a previous internship which sparked his interest in signatures. LB and AF thank the Sorbonne Center for Artificial Intelligence (SCAI) and its team. 

\bibliography{refs}
\bibliographystyle{icml2023}
\clearpage
\appendix
\onecolumn 

\begin{center}
    \Large \textbf{Supplementary Material} 
\end{center}
\bigskip 

\section{Summary of used notations}
\label{appendix:definitions}
The following table provides an exhaustive list of notations used in the main body of the paper. Notations are grouped by subsections.

\begin{table}[h!]
\centering
\footnotesize
\begin{tabular}{cll}
\toprule
\textbf{Notation} & \textbf{Definition} & \textbf{Reference} \\
\midrule 
$y_t$ & Target path & Introduction, page 1\\
$x_t$ & Feature path & ---  \\
$Y_t$ & Target time series & Introduction, page 2\\
$X_t$ & Feature time series & --- \\
$D^i$ & Sampling grid of the features of individual $i$ & ---\\
$\Bar{D}^i$ & Sampling grid of the target of individual $i$ & ---\\
$\xi^i_t$ & Noise on the feature time series & ---\\
$\varepsilon^i_t$ & Noise on the target time series & --- \\
$m_i$ & Number of sampling points of the target for individual $i$ & --- \\
$n$ & Number of samples & --- \\
$M$ & Total number of sampling points of the target & --- \\
$\mathbf{X}^i$ & Matrix of measurements of the feature time series for individual $i$ & ---\\
$\mathbf{Y}^i$ & Matrix of measurements of the target time series for individual $i$ & ---\\
\midrule 
$\norm{x}_{\textnormal{1-var},[0,t]}$ & Total variation of the path $x$ on the interval $[0,t]$ & Section 2, page 3, Assumption \ref{assump:X_bv} \\
$C_L^{\textnormal{1-var}([0,1],\mathbb{R}^d)}$ & Set of continuous paths of total variation bounded by $L$ & --- \\
$\mathbf{F}$ & Unknown smooth generative vector field & Section 2, page 3, Assumption \ref{assump:Y_CDE} \\ 
\midrule 
$S^I(x_{[0,t]})$ & Signature coefficient of the path $x$ on $[0,t]$ associated to the word $I$ & Section 2, page 4, Definition \ref{def:signature}\\
$\mathbb{X}_{k,[0,t]}$ & Signature of order $k$ & ---\\
$S(x_{[0,t]})$ & Full signature at order $N$ & ---\\
$S_N(x_{[0,t]})$ & Truncated signature at order $N$ & ---\\
$s_d(N)$ & Size of the signature truncated at order $N$ of a $d$ dimensional path & ---\\
$\Bar{y}_{N,t}$ & Taylor expansion of the solution of a CDE of order $N$ evaluated in $t$ & Section 2, page 4, Definition \ref{def:taylor_expansion}\\
$\Phi^I_\mathbf{F}$ & Differential product of the vector field $\mathbf{F}$ along $I$ & Appendix \ref{appendix:diff_prod}, page 15, Definition \ref{def:diff_prod}\\
$\theta^\star_N$ & Matrix collecting all differential products up to order $N$ & Section 2, page 5, Equation \eqref{eq:linear_problem}\\
$\Lambda_k(\mathbf{F})$ & Norm on the differential product of $\mathbf{F}$ & Section 2, page 5, Assumption \ref{assump:decay_derivatives_F}\\
\midrule 
$\eta$ & Minimal sampling time & Section 2, page 5, Assumption \ref{assump:sampling_grid} \\
$\mathcal{D}$ & Set of individual specific sampling grids of features & Section 2, page 5\\
$\Bar{\mathcal{D}}$ & Set of individual specific sampling grids of targets & --- \\
$\abs{D}$ & Meshisze of a sampling grid $D$ & ---\\
$\#D$ & Number of sampling points in a sampling grid & ---\\
$\mathbf{Y}$ & Matrix collecting all target measurements of the sample & --- \\
$\mathbf{S}^\mathcal{D}_N$ & Matrix of stacked signatures & Section 2, page 6\\
\midrule 
$\mathbf{y}$ & Expectation of $\mathbf{Y}$ & Section 3, page 6\\
$\Hat{\theta}_{N,M}$ & Siglasso estimator & Section 3, page 6, Equation \eqref{eq:optimization_problem}\\
$A_\xi(\delta)$ & Set bounding the magnitude of noises $(\xi_t^i)$ & Section 3, page 6, Equation \eqref{eqn:a_xi}\\
$C_k(\Bar{\delta})$ & Constant & Section 3, page 6\\
$A_\varepsilon(\delta)$ & Set bounding the magnitude of the noises $(\varepsilon_t^i)$ & Section 3, page 7, Equation \eqref{eq:A_epsilon}\\
$\Omega(\theta)$ & Penalty evalutated at $\theta$ & Section 3, page 7, Equation \eqref{eq:penalty}\\
\bottomrule
\end{tabular}
\end{table}

\section{Mathematical details}

\subsection{The Riemann-Stieltjes integral}
\label{appendix:RS_integral}

We fitst recall two key properties on the Riemann-Stieljes integral. For a general presentation of the Riemann-Stieltjes integral, we refer to \citet{friz2010multidimensional}.




\begin{proposition}
Let $x \in C_L^{\textnormal{1-var}}([0,1],\mathbb{R}^d)$ and $y:[0,1] \rightarrow \mathbb{R}^d$ be a continuous path. Then 
\begin{align*}
    \norm{\int_0^t y_s dx_s} \leq \norm{y}_{\infty,[0,t]} \norm{x}_{\textnormal{1-var}, [0,t]}
\end{align*}
\end{proposition}
We refer the reader to \citet[][Proposition 2.2]{friz2010multidimensional} for a proof.

\begin{proposition}[Integration by parts] \label{prop:int_by_part_riemann}
Let $x,y \in C_L^{\textnormal{1-var}}([0,1],\mathbb{R}^d)$. Then 
\begin{align*}
    \int_s^t y_u dx_u + \int_s^t x_u dy_u = y_t x_t - y_sx_s
\end{align*}
\end{proposition}

See \citet[][Proposition 2.4]{friz2010multidimensional} for a proof.

\subsection{The truncated tensor algebra}
\label{appendix:truncated_algebra}
This section introduces notations and definitions on the space in which signatures are defined, namely, the tensor algebra. While for the exposition of our main results, the truncated signature of a path $x \in C_L^{\textnormal{1-var}}([0,1],\mathbb{R}^d)$ at depth $N \geq 1$ can be assimilated to an element of $\mathbb{R}^{s_d(N)}$, it is often useful to place ourselves in the tensor algebra to obtain finer bounds or technical results.\\

Let $x \in C_L^{\textnormal{1-var}}([0,1],\mathbb{R}^d)$ be a path of bounded variation. For a word $I = (i_1,\dots,i_k) \in \left\{ 1,\dots,d\right\}^k$ of size $k$, the signature coefficient $S^I(x_{[0,1]})$ can be seen as an element of the $k$-th tensor product of $\mathbb{R}^d$ with itself, denoted by $\big(\mathbb{R}^d\big)^{\otimes k}$. For instance, the coefficients of order $k=1$ can be written as a vector and the coefficients of order $k=2$ as a matrix, and so on, i.e.,
\begin{align*}
    \mathbb{X}_{[0,1]}^1 = \begin{bmatrix} \int_0^1 dx_s^{(1)}\\
    \vdots\\
    \int_0^1 dx_s^{(d)}
    \end{bmatrix}
\quad 
\text{and} \quad 
    \mathbb{X}_{[0,1]}^2 = \begin{bmatrix} \int_0^1 dx_s^{(1)}dx_s^{(1)} & \dots & \int_0^1 dx_s^{(1)}dx_s^{(d)}  \\
    \vdots & & \vdots\\
    \int_0^1 dx_s^{(d)}dx_s^{(1)} & \dots & \int_0^1 dx_s^{(d)}dx_s^{(d)}
    \end{bmatrix}.
\end{align*}

We now define a norm on $(\mathbb{R}^d)^{\otimes k}$. Let $a \in (\mathbb{R}^d)^{\otimes k}$ and $(e_1,\dots,e_d)$ be the canonical basis of $\mathbb{R}^d$. Then $(e_{i_1}\otimes \dots \otimes e_{i_k})_{(i_1,\dots,i_k) \in \{1,\dots,d\}^k}$ is a basis of $(\mathbb{R}^d)^{\otimes k}$. We can thus write $a$ as $a = (a^I)_{I \in \left\{1,\dots,d \right\}^k}$.  For every $k \geq 0$, the vector space $(\mathbb{R}^d)^{\otimes k}$ is naturally endowed with the norm 
\[
\norm{a}_{(\mathbb{R}^d)^{\otimes k}}^2 = \sum\limits_{I \in \left\{1,\dots,d \right\}^k} \big(a^I \big)^2.
\]

Remark that this norm satisfies for any $x \in (\mathbb{R}^d)^{\otimes k}$ and $y\in (\mathbb{R}^d)^{\otimes m}$, 
\begin{align}
\label{eq:admissibility}
\norm{x \otimes y}_{(\mathbb{R}^d)^{\otimes (k+m)}} = \norm{x}_{(\mathbb{R}^d)^{\otimes k}} \norm{y}_{(\mathbb{R}^d)^{\otimes m}}.
\end{align}
We refer to \citet[]{fermanian2021framing} for further deails.
The signature truncated at depth $N \geq 1$ collects elements from $\mathbb{R}, \big( \mathbb{R}^d\big)^{\otimes 2}, \dots, \big( \mathbb{R}^d\big)^{\otimes N}$. It can thus be seen as an element of the truncated tensor algebra
\[
T_N(\mathbb{R}^d) = \mathbb{R} \oplus \big( \mathbb{R}^d\big)^{\otimes 2} \oplus \dots \oplus \big( \mathbb{R}^d\big)^{\otimes N}.
\]

Let $a = (a_0,\dots,a_N) \in T_N(\mathbb{R}^d)$, where every $a_k \in (\mathbb{R}^d)^{\otimes k}$. We define the norm 
\begin{align*}
    \norm{a}_{T_N(\mathbb{R}^d)} = \Big(\sum\limits_{k=0}^N \norm{a_k}_{(\mathbb{R}^d)^{\otimes k}}^2 \Big)^{1/2}.
\end{align*}

To clarify, if we consider the truncated signature of $x$ at depth $N \geq 1$, which is an element of $T_N(\mathbb{R}^d)$, then 
\begin{align*}
    \norm{S_N(x_{[0,t]})}_{T_N(\mathbb{R}^d)} = \Big(\sum\limits_{k=0}^N \norm{\mathbb{X}_{[0,t]}^k}^2_{(\mathbb{R}^d)^{\otimes k}}\Big)^{1/2} = \Big(\sum\limits_{k=0}^N \quad \sum\limits_{I \in \left\{1,\dots,d \right\}^k} S^I(x_{[0,t]}) ^2\Big)^{1/2}.
\end{align*}

Note that this norm is exactly equivalent to the Euclidian norm of $\mathbb{R}^{s_d(N)}$, which is the space we consider in the exposition of our main results for the sake of simplicity.

We are now ready to define the tensor product on the truncated tensor algebra. For two elements $a = (a_0,\dots,a_N)$ and $b = (b_0,\dots,b_N)$ both in $T_N(\mathbb{R}^d)$, we define
\[
a \otimes b = (c_0,\dots,c_j,\dots, c_N), \quad \text{ where } c_j = \sum\limits_{k=0}^j a_k \otimes b_{j-k}.
\]
For any $k=0, \dots ,N$, we let $\pi_k : T_N(\mathbb{R}^d) \to (\mathbb{R}^d)^{\otimes k}$ be the canonical projection of $T_N(\mathbb{R}^d)$ onto $(\mathbb{R}^d)^{\otimes k}$. More precisely, for every $a = (a_0,\dots,a_N) \in T_N(\mathbb{R}^d)$, 
\[
\pi_k(a) = a_k
\]
We also define 
the canonical projection $\Pi_k : T_N(\mathbb{R}^d) \to T_k(\mathbb{R}^d)$ defined by
\[
\Pi_k(a) = \left(a_0,\dots,a_k\right).
\]

\subsection{The differential product}
\label{appendix:diff_prod}

We first define the differential product, in order to give a precise statement of Assumption~\ref{assump:decay_derivatives_F}.

\begin{definition}
\label{def:diff_prod}
Let $F,G:\mathbb{R}^p \rightarrow \mathbb{R}^p$ be two smooth vector fields, i.e., each of their components is $\mathcal{C}^\infty$. Denote by $J(\cdot)$ the Jacobian matrix. The differential product $F \star G: \mathbb{R}^p \rightarrow \mathbb{R}^p$ is the smooth vector field defined for any $h \in \mathbb{R}^p$
\begin{align*}
    (F \star G)(h) = \sum\limits_{j=1}^e \frac{\partial G}{\partial h_j}(h)F_j(h) = J(G)(h)F(h).
\end{align*}
\end{definition}

The differential product is not associative. We therefore use the convention to evaluate it from right to left, that is,
\begin{align*}
    F^{1}\star F^{2} \star F^{3} =  F^{1}\star \left(F^{2} \star F^{3}\right).
\end{align*}

Let $F: \mathbb{R}^p \rightarrow \mathbb{R}^{p \times d}$ be a smooth vector field. We write $F^1,\dots, F^d$ the columns of $\mathbf{F}$. Every $F^i$, for $i=1,\dots,d$, can thus be seen as a map from $\mathbb{R}^p$ to $\mathbb{R}^p$. Recall that $y_0 \in  \mathbb{R}^p$ is the initial condition of the CDE defined in Assumption \ref{assump:Y_CDE}. Let $I = (i_1,\dots,i_k)  \in \left\{ 1,\dots,d\right\}^k$. We now define
\begin{align*}
    \Phi^I_\mathbf{F}(y_0) = \left(F^{i_1}\star \dots \star F^{i_k}\right)(y_0) \in \mathbb{R}^p.
\end{align*}

We refer to \citet{fermanian2021framing} for greater details on the differential product. We now define for all $k \geq 1$
\begin{align} \label{eq:def_Lambda_k}
    \Lambda_k(\mathbf{F}) = \sup\limits_{1 \leq i_1,\dots,i_k \leq d} \norm{\Phi^I_\mathbf{F}(y_0)} \in \mathbb{R},
\end{align}
and use the convention $\Lambda_0(\mathbf{F}) = \norm{y_0}$. Remark that Assumption \ref{assump:decay_derivatives_F} implies that 
\[
\frac{d^{N+1}\Lambda_{N+1}(\mathbf{F})}{(N+1)!} \underset{N \to + \infty}{\longrightarrow} 0.
\]
As an immediate consequence, the truncation bias
\[
\Bigg(\frac{d^{N+1}\Lambda_{N+1}(\mathbf{F})}{(N+1)!}\Bigg)^2
\]
introduced in Lemma \ref{lemma:bias_bound}, vanishes as $N$ grows.

\subsection{Analogy with the Taylor extension}
\label{appendix:analogy_taylor}

The definition of the Taylor expansion of a CDE exposed in the previous subsection is technical. However, it can simply be thought of as a generalization of the the classical Taylor expansion of a $\mathcal{C}^\infty$ function $f: \mathbb{R} \rightarrow \mathbb{R}$. Recall that in this case, the Taylor expansion at $0$ evaluated at $t \in \mathbb{R}$ of order $N \in \mathbb{N}$ writes as a power series
\begin{align}
\label{eq:vanillaTaylor}
    f(t) \approx f(0) + \frac{f'(0)}{1!}t + \dots + \frac{f^{(N)}(0)}{N!} t^N.
\end{align}
Every element of this power series is a product of two terms: the derivatives of $f$ encode some information about the regularity of $f$ at the initial point $0$ and do not depend on $t$, while the polynomial terms $t^k$ allow this linearized form to evolve with time $t$. Remark that these polynomial terms do not depend on $f$.

Similarly, Equation~\eqref{eq:taylor_expansion} is also a sum of products of two terms. On the one hand, the evolving nature of the system, instead of being handled by the polynomial terms $t^k$, are now captured by the signature coefficients $S^{I}(x_{[0,t]})$. As the polynomial terms, they do not depend on $\mathbf{F}$. On the other hand, the information about the initial value of the system at time $t=0$ and the dynamics of $\mathbf{F}$ are summarized by the differential product $\Phi^{I}_{\mathbf{F}}\left(y_0\right)$, which play the same role as the successive derivatives in Equation~\eqref{eq:vanillaTaylor}. To capture the multivariate nature of the paths, the Taylor expansion is summed over multi-indexes, or words, $I=(i_1,\dots,i_k) \in \{1,\dots,d\}^k$ of size $k$ for $k \in \mathbb{N}$.

\subsection{Properties of subgaussian random vectors}
\label{appendix:subgaussian}

We start with the definition of a subgaussian random variable, see \citet{vershynin2010introduction} for more details.

\begin{definition}
\label{def:subgaussian}
    A real-valued random variable $X$ is said to be $\sigma^2$-subgaussian if for all $t > 0$ 
    \[
    \mathbb{P}(X > t) \leq \exp(-t^2/\sigma^2),
    \]
    or, equivalently, if for all $t \in \mathbb R $
      \[
    \mathbb{
    E}(e^{tX}) \leq \exp(-c t^2\sigma^2),
    \]
    where $c$ is an universal constant.
    A random vector $Z$ is subgaussian if, for any vector $c$ of norm $1$, $\langle Z, c \rangle$ is subgaussian.
\end{definition}

The norm of a sequence of $d$ subgaussian random variables concentrates around $\sqrt{d}$, as stated by the following lemma. 

\begin{lemma}
\label{lemma:simple_subgaussian_norm}
    Let $X_1,\dots,X_n$ be a sequence of i.i.d.~$\sigma^2$-subgaussian random variables. Let $X=(X_1,\dots,X_d) \in \mathbb{R}^d$. There exists a universal constant $c$ such that for all $t >0$
    \[
    \mathbb{P}(\norm{X}_2 \geq t+ \sigma \sqrt{d}) \leq \exp(-ct^2/\sigma^2).
    \]
\end{lemma}

\begin{proof}
    We refer to \citet[Theorem 3.1.1]{vershynin-high-dimensional-2018} for a proof.
\end{proof}

We can use this lemma to bound the maximum of $n$ sequences of $d$ subgaussian random variables with high probability.

\begin{lemma}
\label{lemma:concentration_max_subgaussian}
    Let $X_1,\dots,X_n$ be a sequence of i.i.d.~$\sigma^2$-subgaussian random variables, such that for all $i=1,\dots,n$, $X_i = (X_{i1},\dots,X_{id})$. Then there exists a universal constant $c$ such that for all $\delta \in (0,1)$
    \[
    \mathbb{P} \Big(\max\limits_{i=1,\dots,n}\norm{X_i} \leq \sigma \sqrt{d} + \sigma \sqrt{\frac{1}{c}\log(n/\delta)}\Big) \geq 1-\delta.
    \] 
\end{lemma}
\begin{proof}
    Using Lemma \ref{lemma:simple_subgaussian_norm} and a union bound, we have 
    \begin{align*}
        \mathbb{P} \Big(\max\limits_{i=1,\dots,n}\norm{X_i} \geq \sigma \sqrt{d} + \sigma \sqrt{\frac{1}{C}\log(n/\delta)}\Big) & = \mathbb{P} \Big(\bigcup\limits_{i=1}^n \Big\{ \norm{X_i}_2 \geq \sigma \sqrt{d} + \sigma \sqrt{\frac{1}{C}\log(n/\delta)} \Big\} \Big) \\
        & \quad \leq \sum\limits_{i=1}^n \mathbb{P} \Big( \norm{X_i}_2 \geq \sigma \sqrt{d} + \sigma \sqrt{\frac{1}{c}\log(n/\delta)} \Big)\\
        & \quad \leq \delta,
    \end{align*}
    which yields the desired inequality.
\end{proof}

Notice that the universal constant is identical between both lemmas. As a consequence of this last lemma, under Assumption \ref{assump:feature_noise}, the set 
\begin{align}
    A_\xi(\delta) =  \Big\{ \max\limits_{i=1,\dots,n,t\in D^i}\norm{\xi^i_t} \leq v_\xi \sqrt{d} + v_\xi  \sqrt{\frac{1}{c}\log(\# \mathcal{D} / \delta)} \Big\}  
\end{align}
where $\# \mathcal{D} = \sum\limits_{i=1}^n \#D^i$ is of probability at least $1-\delta$.

We also need the following lemma. 

\begin{lemma}
\label{lemma:sum_subgaussian}
     Let $X_1,\dots,X_n$ be a sequence of i.i.d.~$\sigma^2$-subgaussian random variables. Let $Z_1,\dots,Z_n$ be random variables such that for all $i=1,\dots,n$, $|Z_i| \leq \alpha$ almost surely. Then $\sum\limits_{i=1}^n X_i Z_i$ is $n \sigma^2 \alpha^2$-subgaussian. 
\end{lemma}
\begin{proof}
    We use the characterization of subgaussian random variables by their characteristic function. For all $t > 0$,
    \begin{align*}
        \mathbb{E}\bigg[e^{t \sum_{i=1}^{n} X_i Z_i}\bigg] = \mathbb{E}\bigg[ \prod_{i=1}^n \mathbb{E} \big[ e^{t X_i Z_i} \,|\, Z_1,\dots,Z_n\big] \bigg] \leq \mathbb{E} \bigg[\prod_{i=1}^n \mathbb{E} \big[e^{tX_i \alpha}\big]\bigg] \leq \mathbb{E} \big[e^{c t^2 n\alpha^2\sigma^2}\big].
    \end{align*}
    This finally yields that 
    \begin{align*}
        \mathbb{E}\bigg[e^{t \sum X_i Z_i}\bigg] \leq \mathbb{E}\big[ e^{ct^2n\alpha^2 \sigma^2}\big],
    \end{align*}
    which concludes the proof.
\end{proof}

\section{Proofs}

\label{appendix:proofs}

\subsection{Preliminary notations}
\label{appendix:proof_notations}
Let $(E, \norm{\cdot}_E)$ be a normed vector space and $x:[0,1]\to E$. The supremum norm of $x$ is defined for all $t\in [0,1]$ as 
\[
\norm{x}_{\infty,[0,t]} = \sup_{s \in [0,t]} \norm{x_s}_E.
\]

When referring to the total variation $\norm{x}_{\textnormal{1-var},[0,1]}$ of a path $x:[0,1]\to \mathbb{R}^d$ over the whole domain, depending on the mathematical context,  we will sometimes drop the time subscript and simply write $\norm{x}_{\textnormal{1-var}}$.

When referring to a matrix $A = (A_{ij}) \in \mathbb{R}^{n \times p}$, we define classicaly the infinite and Frobenius norms by 
\[
\norm{A}_{\infty} = \max_{\substack{i=1,\dots,n \\ j=1,\dots,p}} \abs{A_{ij}} \quad \text{and} \quad \norm{A}_F = \sqrt{\sum\limits_{\substack{i=1,\dots,n \\ j=1,\dots,p}} \abs{A_{ij}}^2}.
\]

We now introduce some notations to take advantage of the structure of $\theta^*_N$. The true parameter of the Taylor expansion of the model CDE, defined in Equation \eqref{eq:linear_problem}, can be written in block notation as 

\begin{align} 
\label{eq:theta_matrix}
\theta^*_N =  \left[\begin{array}{ccc}
    \begin{matrix}
        \theta^*_{[0],1} & \cdots & \theta^*_{[0],p} \\
        \hline
        &&\\
        \theta^*_{[1],1} & \cdots & \theta^*_{[1],p} \\
        &&\\
        \hline \\
        && \\
        & & \\
        \theta^*_{[2],1} & \cdots & \theta^*_{[2],p} \\
        && \\
        &&\\
        \hline 
        && \\
        & \vdots &\\
        && \\
        \hline 
        &&\\
        && \\
        &&\\
        \theta^*_{[N],1} & \cdots & \theta^*_{[N],p} \\
        && \\
        && \\
        &&\\
    \end{matrix} 
\end{array}\right] \in \mathbb{R}^{s_d(N) \times p}, \quad \text{where} \quad \theta^\ast_{[k], \ell} \in \R^{d^k \times 1}, k=0, \dots, N, \, \ell=1, \dots, p.
\end{align}

Every column of $\theta_N^*$ corresponds to a dimension of the target, while blocks of lines correspond to signatures layers. Thus for every $k=0,\dots, N$ and $\ell = 1, \dots, p$, $\theta^*_{[k], \ell}$ is a column vector of size $d^k$. 

Similarly, for a general $\theta \in \mathbb{R}^{s_d(N) \times p}$ and the SigLasso estimator $\widehat{\theta}_{N,M}$, we will refere to the blocks forming these matrices as respectively $\theta_{[k],\ell}$ and $\widehat{\theta}_{[k],\ell}$, for $k=0,\dots, N$ and $\ell=1,\dots p$.

Likewise, the signature feature matrix $\mathbf{S}_N^\mathcal{D}  \in \mathbb{R}^{M \times s_d(N)}$ can be written in block notation as 
\[ \mathbf{S}_N^\mathcal{D} = \begin{bmatrix}
    \, 1 & \vline & \mathbf{S}^\mathcal{D}_{\cdot,[1]} & \vline & \mathbf{S}^\mathcal{D}_{\cdot,[2]} & \vline & \cdots & \vline & \mathbf{S}^\mathcal{D}_{\cdot,[N]}
  \end{bmatrix} = 
\begin{bmatrix}
    \, 1 & \vline & \mathbf{S}^\mathcal{D}_{1,[1]} & \vline && \mathbf{S}^\mathcal{D}_{1,[2]} && \vline &  & \vline &&& \mathbf{S}^\mathcal{D}_{1,[N]} &&&  \\
    \, \vdots & \vline & \vdots & \vline && \vdots && \vline & \cdots & \vline &&& \vdots &&& \\
    \, 1 & \vline & \mathbf{S}^\mathcal{D}_{n,[1]} & \vline && \mathbf{S}^\mathcal{D}_{n,[2]} && \vline &  & \vline &&& \mathbf{S}^\mathcal{D}_{n,[N]} &&  \end{bmatrix},\]
where for any $k=1, \dots, N$, $\mathbf{S}^\mathcal{D}_{\cdot,[k]} \in \R^{M \times d^k}$ and, for every individual $i=1,\dots,n$,  $\mathbf{S}^\mathcal{D}_{i,[k]} \in \R^{m_i \times d^k}$ (recall that $m_i$ is the number of measurements of the target path $y^i$). More precisely, given her target sampling grid $\Bar{D}^i = (\Bar{t}_1^i,\dots, \Bar{t}^i_{m_i})$, the individual-specific signature block of depth $k$ is equal to
\[
\mathbf{S}^\mathcal{D}_{i,[k]} = \begin{bmatrix}
    1 & S^{(1)}(X^i_{[0,\Bar{t}^i_1]}) & \cdots & S^{(d,\dots,d)}(X^i_{[0,\Bar{t}^i_1]})\\
    \vdots & \vdots & & \vdots 
    \\
    1 & S^{(1)}(X^i_{[0,\Bar{t}^i_{m_i}]}) & \cdots & S^{(d,\dots,d)}(X^i_{[0,\Bar{t}^i_{m_i}]})\\
\end{bmatrix},
\]
where the path $t \to X^{i}_t$ is a linear interpolation of the observed time series $\mathbf{X}^{i}$. The same notations will be used for the true signature feature matrix $\mathbf{S}_N$. We use the bracket notation $[\cdot]$ both in $\theta^*_N$ and $\mathbf{S}_N^\mathcal{D}$ to emphasise that both the columns of the feature matrix and the lines of learned parameter correspond to words of the alphabet $\{1,\dots,d\}$.

The unobserved matrix of true values of the target writes as

\begin{align}
\label{eq:true_target_matrix}
\mathbf{y} = \begin{bmatrix}
    \mathbf{y}^1 \\
    \vdots \\
    \mathbf{y}^n
\end{bmatrix} =
\begin{bmatrix}
    \mathbf{y}^1_{1} & \cdots & \mathbf{y}^1_{p}\\
    \vdots & & \vdots \\
     \mathbf{y}^n_{1} & \cdots & \mathbf{y}^n_{p}
\end{bmatrix} = 
\begin{bmatrix}
    y^1_{1,\Bar{t}^1_1} & \cdots & y^1_{p,\Bar{t}^1_1}\\
    \vdots & & \vdots \\
    y^1_{1,\Bar{t}^1_{m_1}} & \cdots & y^1_{p,\Bar{t}^1_{m_1}}\\
    \vdots & & \vdots\\
    y^n_{1,\Bar{t}^n_{1}} & \cdots & y^n_{p,\Bar{t}^n_{1}}\\
    \vdots & & \vdots \\
    y^n_{1,\Bar{t}^n_{m_n}} & \cdots &  y^n_{p,\Bar{t}^n_{m_n}} 
\end{bmatrix} \in \mathbb{R}^{M \times p}
\end{align}

and the measurement matrix $\mathbf{Y} \in \mathbb{R}^{M \times p} $ can be written in a similar fashion as 
\begin{align}
\label{eq:target_matrix}
\mathbf{Y} = \begin{bmatrix}
    \mathbf{Y}^1 \\
    \vdots \\
    \mathbf{Y}^n
\end{bmatrix} =
\begin{bmatrix}
    \mathbf{Y}^1_{1} & \cdots & \mathbf{Y}^1_{p}\\
    \vdots & & \vdots \\
     \mathbf{Y}^n_{1} & \cdots & \mathbf{Y}^n_{p}
\end{bmatrix} = 
\begin{bmatrix}
    Y^1_{1,\Bar{t}^1_1} & \cdots & Y^1_{p,\Bar{t}^1_1}\\
    \vdots & & \vdots \\
    Y^1_{1,\Bar{t}^1_{m_1}} & \cdots & Y^1_{p,\Bar{t}^1_{m_1}}\\
    \vdots & & \vdots\\
    Y^n_{1,\Bar{t}^n_{1}} & \cdots & Y^n_{p,\Bar{t}^n_{1}}\\
    \vdots & & \vdots \\
    Y^n_{1,\Bar{t}^n_{m_n}} & \cdots & Y^n_{p,\Bar{t}^n_{m_n}} 
\end{bmatrix} = 
\begin{bmatrix}
    y^1_{1,\Bar{t}^1_1} +  \varepsilon^1_{1,\Bar{t}^1_1} & \cdots & y^1_{p,\Bar{t}^1_1} + \varepsilon^1_{p,\Bar{t}^1_1}\\
    \vdots & & \vdots \\
    y^1_{1,\Bar{t}^1_{m_1}} + \varepsilon^1_{1,\Bar{t}^1_{m_1}}& \cdots & y^1_{p,\Bar{t}^1_{m_1}} + \varepsilon^1_{p,\Bar{t}^1_{m_1}}\\
    \vdots & & \vdots\\
    y^n_{1,\Bar{t}^n_{1}} +  \varepsilon^n_{1,\Bar{t}^n_{1}}& \cdots & y^n_{p,\Bar{t}^n_{1}} + \varepsilon^n_{p,\Bar{t}^n_{1}}\\
    \vdots & & \vdots \\
    y^n_{1,\Bar{t}^n_{m_n}} +  \varepsilon^n_{1,\Bar{t}^n_{m_n}}& \cdots &  y^n_{p,\Bar{t}^n_{m_n}}  + \varepsilon^n_{p,\Bar{t}^n_{m_n}} 
\end{bmatrix}
\end{align}

\subsection{Proof of Lemma~\ref{lemma:estimation_error}}
\label{proof:bias_variance}






Using the definition of $\Lambda_k(\mathbf{F})$ (see Equation \eqref{eq:def_Lambda_k}), we get the following proposition which allows to obtain an explicit dependence of the oracle bound on the regularity of $\mathbf{F}$. 
\begin{proposition}\label{prop:norm_theta_star}
Let $\theta^\ast_N$ be defined as in Equation ~\eqref{eq:linear_problem}. Then
\label{lemma:bound_norm_theta}
\[
    \norm{\theta^*_N}_\textnormal{F}^2 \leq \sum\limits_{k=0}^N d^k \Lambda_k(\mathbf{F})^2,
\]
and, for all $k =0,\dots,N$ and $\ell=1,\dots,p$,
\[
\big\|\theta^*_{[k],\ell}\big\|_1 \leq d^k\Lambda_k(\mathbf{F}).
\]
\end{proposition}
\begin{proof}
By definition, 
\[
\norm{\theta^*_N}_\textnormal{F}^2 = \sum \limits_{k=0}^N \quad \sum\limits_{1 \leq i_1,\dots,i_k \leq d} \norm{F^{i_1}\star \dots \star F^{i_k} (y_0)}_2^2.
\]

Since for all $(i_1,\dots,i_k) \in \left\{1,\dots,d \right\}^k$,
\[
\norm{F^{i_1}\star \dots \star F^{i_k} (y_0)}_2^2 \leq \Lambda_k(\mathbf{F})^2,
\]
we get
\begin{align*}
    \sum \limits_{k=0}^N \quad \sum\limits_{1 \leq i_1,\dots,i_k \leq d} \norm{F^{i_1}\star \dots \star F^{i_k} (y_0)}_2^2 & \leq \sum \limits_{k=0}^N \quad \sum\limits_{1 \leq i_1,\dots,i_k \leq d} \Lambda_k(\mathbf{F})^2\\
    & \leq \sum \limits_{k=0}^N  d^k \Lambda_k(\mathbf{F})^2.
\end{align*}

We now turn to the second inequality. For $k=0$, the inequality holds by definition. For $k=1,\dots,N$ and $\ell=1,\dots,p$, by definition of the $\ell_1$ norm, 
\begin{align*}
    \norm{\theta^*_{[k],\cdot}}_1 =  \sum\limits_{1 \leq i_1,\dots,i_k \leq d} \norm{\Phi^{I}_\mathbf{F}(y_0)}_1,
\end{align*}
This yields
\[
\norm{\theta^*_{[k],\cdot}}_1 \leq d^k \Lambda_k(\mathbf{F})
\] and thus
\[
\norm{\theta^*_{[k],\ell}}_1 \leq d^k \Lambda_k(\mathbf{F})
\]
for $\ell=1,\dots,p$.
\end{proof}

The following lemma is needed to leverage classical proof techniques to bound the prediction error of the Lasso estimator. 

\begin{lemma}
\label{lemma:word_bound}
Let $x \in C^{\textnormal{1-var}}_L([0,1],\mathbb{R}^d)$. Then conditionally on $A_\xi(\delta)$, for a given signature layer $k \geq 1$, the maximum among all signature coefficients and individuals is bounded from above, that is
\[
\norm{\mathbf{S}^\mathcal{D}_{\cdot, [k]}}_{\infty} \leq \frac{1}{k !}.
\]
\end{lemma}
\begin{proof}
It is well known \citep[see, e.g.,][Proposition 3]{fermanian2022functional} that if $\mathbb{X}^k$ is the signature of a path $x \in C_L^{\textnormal{1-var}}([0,1], \mathbb{R}^d)$, then
\begin{align*} \norm{\mathbb{X}^k}_{(\mathbb{R}^d)^{\otimes k}} \leq \frac{\norm{x}_{\textnormal{1-var}}^k}{k!}.
\end{align*}
As a consequence, for every word $I$ of size $k$, one gets 
\begin{align*}
    \big|S^I(x))\big| \leq \frac{\norm{x}_{\textnormal{1-var}}^k}{k !}. 
\end{align*}
The matrix $\mathbf{S}_N^\mathcal{D}$ is constructed by taking signatures of linear interpolations of the $\mathbf{X}^{i}$s normalized by their total variation. It therefore contains only signatures of paths of total variation bounded by 1. Taking the maximum on $I \in \{1,\dots,d \}^k$ and individuals $i=1,\dots,n$, we get 
\[
\norm{\mathbf{S}^\mathcal{D}_{\cdot, [k]}}_{\infty} \leq \frac{1}{k !}.
\]

\end{proof}

This final inequality being stated, we can now go back to the proof of Lemma~\ref{lemma:estimation_error}. We prove it in full generality for $p \geq 1$. In this proof, we make extensive use of the notations introduced in Subsection \ref{appendix:proof_notations} and refer the reader to it if a notation is unclear.

\begin{proof}
In all the proof, we place ourselves on the set $A_\xi(\delta)$ defined by Equation~\eqref{eqn:a_xi}, which ensures that the matrix $\mathbf{S}_N^{\mathcal{D}}$, seen as a random quantity, is well defined. Recall that we have two sources of randomness: the feature noises $\xi^{i}_t$ on the $\mathbf{X}^{i}$s and the target noises $\varepsilon^{i}_t$ on the $\mathbf{Y}^{i}$s. The feature noises appear only in $\mathbf{S}_N^{\mathcal{D}}$ and make it a random quantity. 
For $\mathbf{S}_N^{\mathcal{D}}$ to be well-defined, we then need the total variation of the linear interpolation of the feature time series $\mathbf{X}^{i}$ to be finite. This holds on the set $A_\xi(\delta)$ since all noises are then bounded.

Recall that we have defined $\widehat{\theta}_{N,M}$ as
\begin{equation*}
    \widehat{\theta}_{N,M} \in \argmin\limits_{\theta\in \mathbb{R}^{  s_d(N) \times p}} \frac{1}{2M}\norm{\mathbf{Y}-\mathbf{S}_N^{\mathcal{D}}\theta}_\textnormal{F}^2 + \Omega(\theta).
\end{equation*}
Note that 
\begin{equation*}
    \frac{1}{2M}\norm{\mathbf{Y}-\mathbf{S}_N^{\mathcal{D}}\theta}_\textnormal{F}^2 + \Omega(\theta) = \sum_{\ell=1}^p \frac{1}{2M}\norm{\mathbf{Y}_{\ell}-\mathbf{S}_N^{\mathcal{D}}\theta_{[\cdot],\ell}}_2^2 + \Omega(\theta_{[\cdot],\ell}),
\end{equation*}
where $\mathbf{Y}_{\ell} \in \mathbb R^M$ is the $\ell$-th column of the target measurement matrix defined in Equation \eqref{eq:target_matrix}. The quantity $\theta_{[\cdot],\ell}\in  \mathbb{R}^{s_d(N)}$ is the $\ell$-th column of the parameter matrix defined in Equation \eqref{eq:theta_matrix}.

By definition, for any $\theta \in \R^{s_d(N)}$, we have
\begin{align*}
\norm{\mathbf Y_{\ell}-\mathbf{S}_N^{\mathcal{D}} \widehat{\theta}_{[\cdot],\ell}}_2^2 \leq \norm{\mathbf Y_{\ell}-\mathbf{S}_N^{\mathcal{D}} \theta_{[\cdot],\ell} }_2^2 + \Omega(\theta_{[\cdot],\ell}) - \Omega(\widehat{\theta}_{[\cdot],\ell}).
\end{align*}
Moreover, letting $\bm{\varepsilon}_\ell = (\varepsilon^1_{\ell, \Bar{t}^1_1},\dots, \varepsilon^n_{\ell, \Bar{t}^n_{m_n}})^\top \in \mathbb R^M$ be a vector of i.i.d.~noises (see Equation \eqref{eq:target_matrix}), we have $\mathbf{Y}_\ell = \mathbf{y}_\ell + \bm{\varepsilon}_\ell$. The Pythagorean theorem then yields for any $\theta \in \R^{s_d(N)}$,
\begin{align*}
    \norm{\mathbf Y_{\ell}-\mathbf{S}_N^{\mathcal{D}} \theta }_2^2 = \norm{\mathbf y_{\ell}-\mathbf{S}_N^{\mathcal{D}} \theta }_2^2 + \norm{\bm{\varepsilon}_\ell}^2 + 2 \langle \bm{\varepsilon}_\ell, y_{\ell}-\mathbf{S}_N^{\mathcal{D}} \theta \rangle.
\end{align*}
Applying this equation to $ \theta_{[\cdot],\ell}$ and $\widehat{\theta}_{[\cdot],\ell}$, we obtain
\begin{equation}\label{eqn:thm2step1}
   \frac{1}{2M}\norm{\mathbf y_{\ell}-\mathbf{S}_N^{\mathcal{D}} \widehat{\theta}_{[\cdot],\ell}}_2^2 \leq \frac{1}{2M}\norm{\mathbf y_{\ell}-\mathbf{S}_N^{\mathcal{D}} \theta_{[\cdot],\ell}}_2^2 + \frac{1}{M} \langle \bm{\varepsilon}_{\ell} , \mathbf{S}_N^{\mathcal{D}} (\widehat{\theta}_{[\cdot],\ell} - \theta_{[\cdot],\ell}) \rangle +  \Omega(\theta_{[\cdot],\ell}) -  \Omega(\widehat{\theta}_{[\cdot],\ell}).
\end{equation}
We now work at each layer of the signature matrix $\mathbf{S}_N^{\mathcal{D}}$. Towards that end, we rewrite 
\begin{equation*}
    \mathbf{S}_N^{\mathcal{D}} \big(\widehat{\theta}_{[\cdot],\ell} - \theta_{[\cdot],\ell} \big) = \sum_{k=0}^N \mathbf{S}_{\cdot,[k]}^{\mathcal{D}} \big(\widehat{\theta}_{[k],\ell} - \theta_{[k], \ell }\big),
\end{equation*}
and bound
\begin{equation*}
    \big\langle \bm{\varepsilon}_{\ell} , \mathbf{S}_N^{\mathcal{D}} (\widehat{\theta}_{[\cdot],\ell} - \theta_{[\cdot],\ell}) \big\rangle = \sum_{k=0}^N \big\langle \bm{\varepsilon}_{\ell} , \mathbf{S}_{\cdot, [k]}^{\mathcal{D}} (\widehat{\theta}_{[k],\ell)} - \theta_{[k],\ell}) \big\rangle \leq \sum_{k=0}^N \| \bm{\varepsilon}_{\ell}^\top \mathbf{S}_{\cdot,[k]}^{\mathcal{D}} \|_{\infty} \|\widehat{\theta}_{[k],\ell} - \theta_{[k],\ell}\|_1
\end{equation*}
by $\ell_1 - \ell_\infty$ norms duality. We fix $k$ and study the term $\| \bm{\varepsilon}_{\ell}^\top \mathbf{S}_{\cdot,[k]}^{\mathcal{D}} \|_{\infty}$. 
Lemma~\ref{lemma:word_bound} ensures that 
each of the words of the signature layer of depth $k$ is bounded by $1/k!$. As a consequence, by Lemma \ref{lemma:sum_subgaussian}, under Assumption~\ref{assump:target_noise}, every element of the vector $\bm{\varepsilon}_{\ell}^\top \mathbf{S}_{\cdot,[k]}^{\mathcal{D}}$ is $v_\varepsilon M / k!^2$-subgaussian. It follows that, for any real number $\mu>0$,
\begin{equation*}
   \mathbb P \Big( \| \bm{\varepsilon}_{\ell}^\top \mathbf{S}_{\cdot,[k]}^{\mathcal{D}} \|_{\infty}> \mu \Big)\leq 2 d^k \exp \Big( - \frac{(k!)^2 \mu^2 }{v_\varepsilon M}\Big).
\end{equation*}
We furthermore place ourselves on $A_\varepsilon(\Bar{\delta})$ defined 
by
\begin{align*}
    A_\varepsilon(\Bar{\delta})  =  \bigcap_{\ell=1}^p \bigcap_{k=0}^N \Big\{ \| \bm{\varepsilon}_{\ell}^\top \mathbf{S}_{\cdot,[k]}^{\mathcal{D}} \|_{\infty} \leq \frac{1}{k!}\sqrt{v_\varepsilon M \log(2pNd^k/\Bar{\delta})}\Big\}.
\end{align*}
We have just seen that, under Assumption \ref{assump:target_noise} (and still conditionally on $A_\xi(\delta)$), one has $ \mathbb{P}(A_\varepsilon(\Bar{\delta})) \geq 1-\Bar{\delta}$. Putting  together all terms in Equation~\eqref{eqn:thm2step1} and plugging the definition of $\Omega$ given in Equation~\eqref{eq:penalty}, we obtain that, on the set $A_\varepsilon(\Bar{\delta}) \cap A_\xi(\delta)$, for all $\theta \in \mathbb{R}^{s_d(N)\times p}$,
\begin{align*}
   \frac{1}{2M}\norm{\mathbf y-\mathbf{S}_N^{\mathcal{D}} \widehat{\theta}_{N,M}}_\textnormal{F}^2 &\leq \frac{1}{2M}\norm{\mathbf y-\mathbf{S}_N^{\mathcal{D}} \theta}_\textnormal{F}^2 \\
   & \quad + \sum_{\ell=1}^p\sum_{k=0}^N \Big( \frac{1}{M} \| \bm{\varepsilon}_{\ell}^\top \mathbf{S}_{\cdot,[k]}^{\mathcal{D}} \|_{\infty} \|\widehat{\theta}_{[k],\ell} - \theta_{[k],\ell}\|_1 + \frac{C_k(\bar{\delta})}{k!\sqrt{M}}
   \big( \|\theta_{[k],\ell}\|_1  -  \|\widehat{\theta}_{[k],\ell}\|_1 \big) \Big)\\
   & \leq \frac{1}{2M}\norm{\mathbf y-\mathbf{S}_N^{\mathcal{D}} \theta}_\textnormal{F}^2 \\
   & \quad +  \sum_{\ell=1}^p\sum_{k=0}^N  \frac{1}{k!\sqrt{M}}\sqrt{v_\varepsilon \log(2pNd^k/\Bar{\delta})} \big(\|\widehat{\theta}_{[k],\ell} - \theta_{[k],\ell}\|_1 +   \|\theta_{[k],\ell}\|_1  -  \|\widehat{\theta}_{[k],\ell}\|_1 \big).
\end{align*}
Choosing $\theta=\theta_N^\ast$, by the triangular inequality,
\begin{align*}
    \|\widehat{\theta}_{[k],\ell} - \theta^\ast_{[k],\ell}\|_1 +   \|\theta^\ast_{[k],\ell}\|_1  -  \|\widehat{\theta}_{[k],\ell}\|_1  \leq 2 \|\theta^\ast_{[k],\ell}\|_1,
\end{align*}
which finally gives us
\begin{align*}
   \frac{1}{2M}\norm{\mathbf y-\mathbf{S}_N^{\mathcal{D}} \widehat{\theta}_{N,M}}_\textnormal{F}^2 
    &\leq \frac{1}{2M}\norm{\mathbf y-\mathbf{S}_N^{\mathcal{D}} \theta_N^\ast}_\textnormal{F}^2 + \frac{2}{\sqrt{M}}\sqrt{v_\varepsilon \log(2pNd^N/\Bar{\delta})}\sum_{\ell=1}^p\sum_{k=0}^N \frac{ \| \theta^*_{[k],\ell}\|_1 }{k!}  
\\  &\leq \frac{1}{2M}\norm{\mathbf y-\mathbf{S}_N^{\mathcal{D}} \theta_N^\ast}_\textnormal{F}^2 + \frac{2p}{\sqrt{M}}\sqrt{v_\varepsilon \log(2pNd^N/\Bar{\delta})}\sum_{k=0}^N \frac{d^k \Lambda_k(\mathbf F)}{k!}\\
& = \frac{1}{2M}\norm{\mathbf y-\mathbf{S}_N^{\mathcal{D}} \theta_N^\ast}_\textnormal{F}^2 + \frac{2pC_N(\Bar{\delta})}{\sqrt{M}}\sum_{k=0}^N \frac{d^k \Lambda_k(\mathbf F)}{k!},
\end{align*}
where the second inequality comes from Proposition~\ref{prop:norm_theta_star}. To conclude the proof, we just need to compute the probability of the set $A_\xi(\delta) \cap A_\varepsilon(\Bar{\delta})$. It is an immediate consequence of Lemma \ref{lemma:concentration_max_subgaussian} that $\mathbb P(A_\xi(\delta)) \geq 1 - \delta$, and we have seen that $\mathbb P (A_\varepsilon(\bar{\delta}) | A_\xi(\delta)) \geq 1 - \bar{\delta}$, which yields that
\[ \mathbb P(A_\xi(\delta) \cap A_\varepsilon(\bar{\delta})) \geq (1-\bar{\delta})(1-\delta).\]

\end{proof}

\subsection{Proof of Lemma \ref{lemma:bias_bound}}
\label{proof:bias_bound}

This proof relies on bouding the remainder of the Taylor expansion of the CDE. 

\begin{proof}
For every $i=1,\dots,n$ and a given point $t_i \in \Bar{D}^i$, one has, using the upper bound of the approximation error of a CDE by its Taylor expansion provided by \citet[Proposition 4]{fermanian2021framing}
\[
\norm{y^i_{t_i} - S_N(x^i_{[0,t_i]}) \theta_N^*} \leq \frac{d^{N+1}\Lambda_{N+1}(\mathbf{F})}{(N+1)!}.
\]
This immediately gives
\begin{align*}
 \frac{1}{M} \norm{\mathbf{y}-\mathbf{S}_N\theta_N^*}^2_{\textnormal{F}} = \frac{1}{M} \sum\limits_{i=1}^n  \sum\limits_{t_i \in \Bar{D}^i} \norm{y^i_{t_i} - S_N(x^i_{[0,t_i]}) \theta_N^*}^2 & \leq \frac{1}{M} \sum\limits_{i=1}^M \Big(\frac{d^{N+1}\Lambda_{N+1}(\mathbf{F})}{(N+1)!}\Big)^2 = \Big(\frac{d^{N+1}\Lambda_{N+1}(\mathbf{F})}{(N+1)!}\Big)^2, 
\end{align*}
which concludes the proof.
\end{proof}

\subsection{A layer-wise bound on the signature}

We now prove that signature layers are locally Lipschitz mappings. We start with the following proposition.

\begin{proposition}
\label{lemma:1varklayer}
Let $x \in C_L^{\textnormal{1-var}}([0,1],\mathbb{R}^d)$. Then for all $t \in [0,1]$, the path $t \mapsto \mathbb{X}^k_{[0,t]}$ has $1$-variation bounded by
\begin{align*}
    \norm{\mathbb{X}^k}_{\textnormal{1-var}, [0,t]} \leq \frac{L^k}{k!}.
\end{align*}

\end{proposition}

\begin{proof}
By definition of the total variation, 
\[
\norm{\mathbb{X}^k}_{\textnormal{1-var}, [0,t]} = \sup_D \sum\limits_{i=1}^m \norm{\mathbb{X}^k_{[0,t_{i+1}]}-\mathbb{X}^k_{[0,t_{i}]}}_{(\mathbb{R}^d)^{\otimes k}} = \sup_D \sum\limits_{i=1}^m \norm{\mathbb{X}^k_{[t_i,t_{i+1}]}}_{(\mathbb{R}^d)^{\otimes k}}, 
\]
since $\mathbb{X}^k_{[0,t]} = \int_0^t dx_{u_1}\otimes \dots \otimes dx_{u_k}$, and where the supremum is taken over finite dissections $D=\{0=t_1,\dots,t_m = 1\}$ of $[0,1]$. Notice that the signature layer of depth $k$ is here written as an element of $(\mathbb{R}^d)^{\otimes k}$, which is more convenient for this proof. Then
\[
\sup_D \sum\limits_{i=1}^m \norm{\mathbb{X}^k_{[t_i,t_{i+1}]}}_{(\mathbb{R}^d)^{\otimes k}} \leq \sup_D \sum\limits_{i=1}^m \frac{\norm{x}_{\textnormal{1-var},[t_i,t_{i+1}]}^k}{k!} \leq \frac{1}{k!} \sup_D \Big(\sum\limits_{i=1}^m \norm{x}_{\textnormal{1-var},[t_i,t_{i+1}]} \Big)^k = \frac{1}{k!} \sup_D \norm{x}_{\textnormal{1-var},[0,1]}^k \leq \frac{L^k}{k!},
\]
where the second inequality follows from the multinomial theorem and the last equality comes from the fact that for all $s<u<t$, $\norm{x}_{\textnormal{1-var},[s,u]}+ \norm{x}_{\textnormal{1-var},[u,t]} = \norm{x}_{\textnormal{1-var},[s,t]}$. This ends our proof.
\end{proof}

We now state a bound on the difference between the $k$-th layer of the signatures of two different paths.

\begin{theorem}
\label{thm:layerbound}
Let $x,z \in C_L^{\textnormal{1-var}}([0,1],\mathbb{R}^d)$. Then for all $k \geq 2$, the difference in supremum norm between the paths $t \to \mathbb{X}^k_{[0,t]}$ and $t \to \mathbb{Z}^k_{[0,t]}$ is bounded by
\[
\norm{\mathbb{X}^k - \mathbb{Z}^k}_{\infty,[0,t]} \leq 2L^{k-1} \sum\limits_{j=1}^{k-1} \frac{1}{j!}\norm{x-z}_{\infty,[0,t]} \leq 2eL^{k-1}\norm{x-z}_{\infty,[0,t]} 
\]
and 
\[
\norm{ \mathbb{X}^1_{[0,t]}-\mathbb{Z}^1_{[0,t]}} \leq  2 \norm{x-z}_{\infty,[0,t]}.
\]
\end{theorem}

\begin{proof}
    Our proof works by induction. Let $x,z \in C_L^{\textnormal{1-var}}([0,1],\mathbb{R}^d)$, and for $t \in [0,1]$ denote by $\mathbb{X}^k_{[0,t]}$ (resp. $\mathbb{Z}^k_{[0,t]}$) the $k$-th layer of the signature of $x$ (resp. $z$). For $k=1$ and $t \in [0,1]$, remark that 
    \[
    \mathbb{X}^1_{[0,t]}-\mathbb{Z}^1_{[0,t]} = \int_0^t d(x_u - z_u) = x_t-z_t -(x_0 - z_0)
    \]
    such that
    \[
    \norm{ \mathbb{X}^1_{[0,t]}-\mathbb{Z}^1_{[0,t]}} \leq \norm{x-z}_{\infty,[0,t]} + \norm{x_0-z_0} \leq 2 \norm{x-z}_{\infty,[0,t]}. 
    \]

    Consider now $k \geq 2$. We have
    \[
    \mathbb{X}^k_{[0,t]} - \mathbb{Z}^k_{[0,t]} = \int_0^t \mathbb{X}^{k-1}_{[0,s]} \otimes dx_s - \int_0^t \mathbb{Z}^{k-1}_{[0,s]} \otimes dz_s = \int_0^t \mathbb{X}^{k-1}_{[0,s]} \otimes d(x_s-z_s+z_s) - \int_0^t \mathbb{Z}^{k-1}_{[0,s]} \otimes dz_s,
    \]
    and thus 
    \[
    \mathbb{X}^k_{[0,t]} - \mathbb{Z}^k_{[0,t]} = \int_0^t \mathbb{X}^{k-1}_{[0,s]} \otimes d(x_s-z_s) + \int_0^t \big(\mathbb{X}^{k-1}_{[0,s]} -\mathbb{Z}^{k-1}_{[0,s]} \big)\otimes dz_s.
    \]
    We now bound each of these terms separately. First, 
    \[
    \norm{\int_0^t \big(\mathbb{X}^{k-1}_{[0,s]} -\mathbb{Z}^{k-1}_{[0,s]} \big)\otimes dz_s}_{(\mathbb{R}^d)^{\otimes k}} \leq \norm{\mathbb{X}^{k-1} -\mathbb{Z}^{k-1}}_{\infty,[0,t]} \norm{z}_{\textnormal{1-var},[0,t]} \leq \norm{\mathbb{X}^{k-1} -\mathbb{Z}^{k-1}}_{\infty,[0,t]}L.
    \]
    Moving to the first integral, integration by parts yields 
    \[
    \int_0^t \mathbb{X}^{k-1}_{[0,s]} \otimes d(x_s-z_s) = \mathbb{X}^k_{[0,t]}\otimes(x_t-z_t) -\mathbb{X}^k_{[0,0]}\otimes(x_0-z_0) - \int_0^t (x_s-z_s) \otimes d\mathbb{X}^{k-1}_{[0,s]}. 
    \] 
    We stress that Proposition \eqref{prop:int_by_part_riemann} applies since the integral over the tensor product is taken coordinate-wise. Since $\mathbb{X}^{k-1}_{[0,0]} = 0$, we are left with 
    \[
    \int_0^t \mathbb{X}^{k-1}_{[0,s]} \otimes d(x_s-z_s) = \mathbb{X}^k_{[0,t]}\otimes(x_t-z_t) -\int_0^t (x_s-z_s) \otimes d\mathbb{X}^{k-1}_{[0,s]}. 
    \]
    Using Lemma \ref{lemma:1varklayer} and submultiplicativity of the tensor norms, this can thus be bounded by 
    \begin{align*}
    \norm{\int_0^t \mathbb{X}^{k-1}_{[0,s]} \otimes d(x_s-z_s)}_{(\mathbb{R}^d)^{\otimes k}} & \leq \norm{\mathbb{X}^{k-1}_{[0,t]}}_{(\mathbb{R}^d)^{\otimes (k-1)}}\norm{x-z}_{\infty,[0,t]} + \norm{x-z}_{\infty,[0,t]} \norm{\mathbb{X}^{k-1}}_{\textnormal{1-var}, [0,t]}\\
    & = \frac{2L^{k-1}}{(k-1)!}  \norm{x-z}_{\infty,[0,t]}.
    \end{align*}
    Finally, we are left with
    \begin{align*}
         \norm{\mathbb{X}^k - \mathbb{Z}^k}_{\infty,[0,t]} \leq \frac{2L^{k-1}}{(k-1)!}  \norm{x-z}_{\infty,[0,t]} + \norm{\mathbb{X}^{k-1} -\mathbb{Z}^{k-1}}_{\infty,[0,t]}L,
    \end{align*}
    which can be recursively bounded by
    \[
    \norm{\mathbb{X}^k - \mathbb{Z}^k}_{\infty,[0,t]} \leq 2L^{k-1}\norm{x-z}_{\infty,[0,t]} \sum\limits_{j=1}^{k-1} \frac{1}{j!} \leq 2L^{k-1}e \norm{x-z}_{\infty,[0,t]}. 
    \]
\end{proof}

Note that this inequality implies that if $z$ is chosen as the linear interpolation of a discretization of $x$ on a grid $D$, and if the grid gets finer, all signature layers converge at speed $\norm{x-z}_{\infty,[0,t]}$ but the multiplicative constant increases with depth (if $L \geq 1$). Figure \ref{fig:sig_convergence} illustrates this phenomenon.

\begin{figure}[ht]
    \centering
    \includegraphics[width=0.3\textwidth]{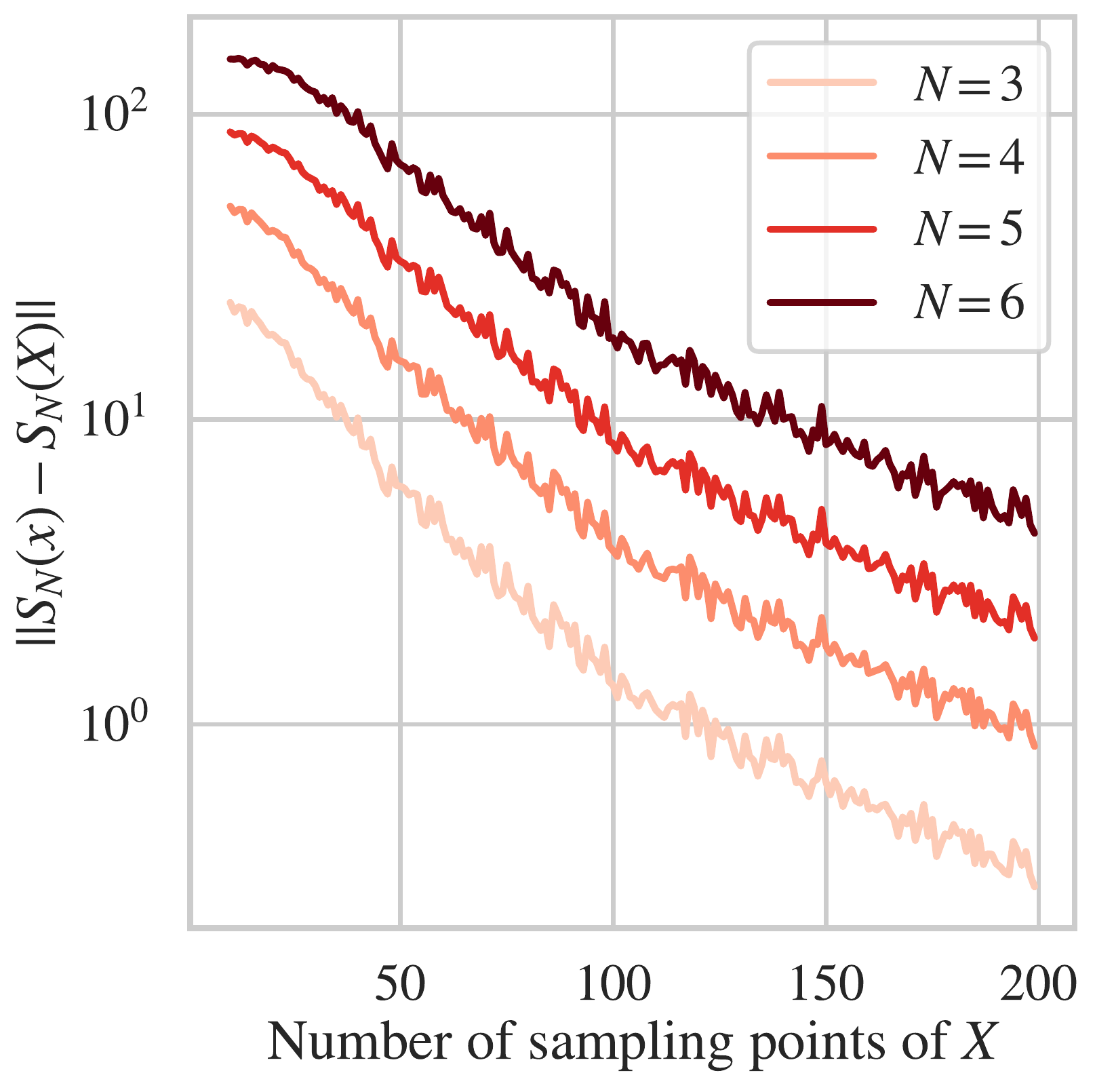}
    \includegraphics[width=0.3\textwidth]{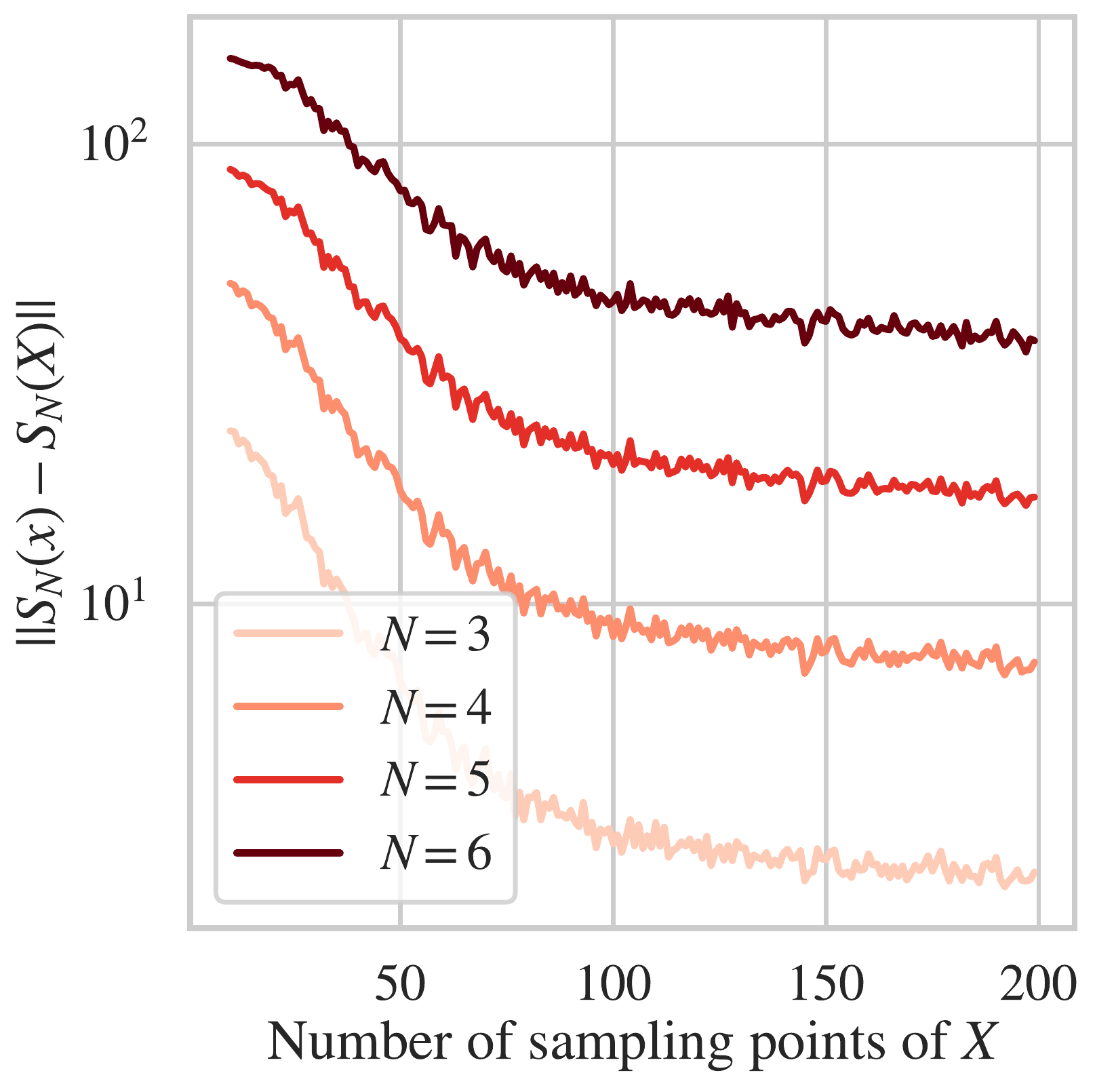}
    \includegraphics[width=0.3\textwidth]{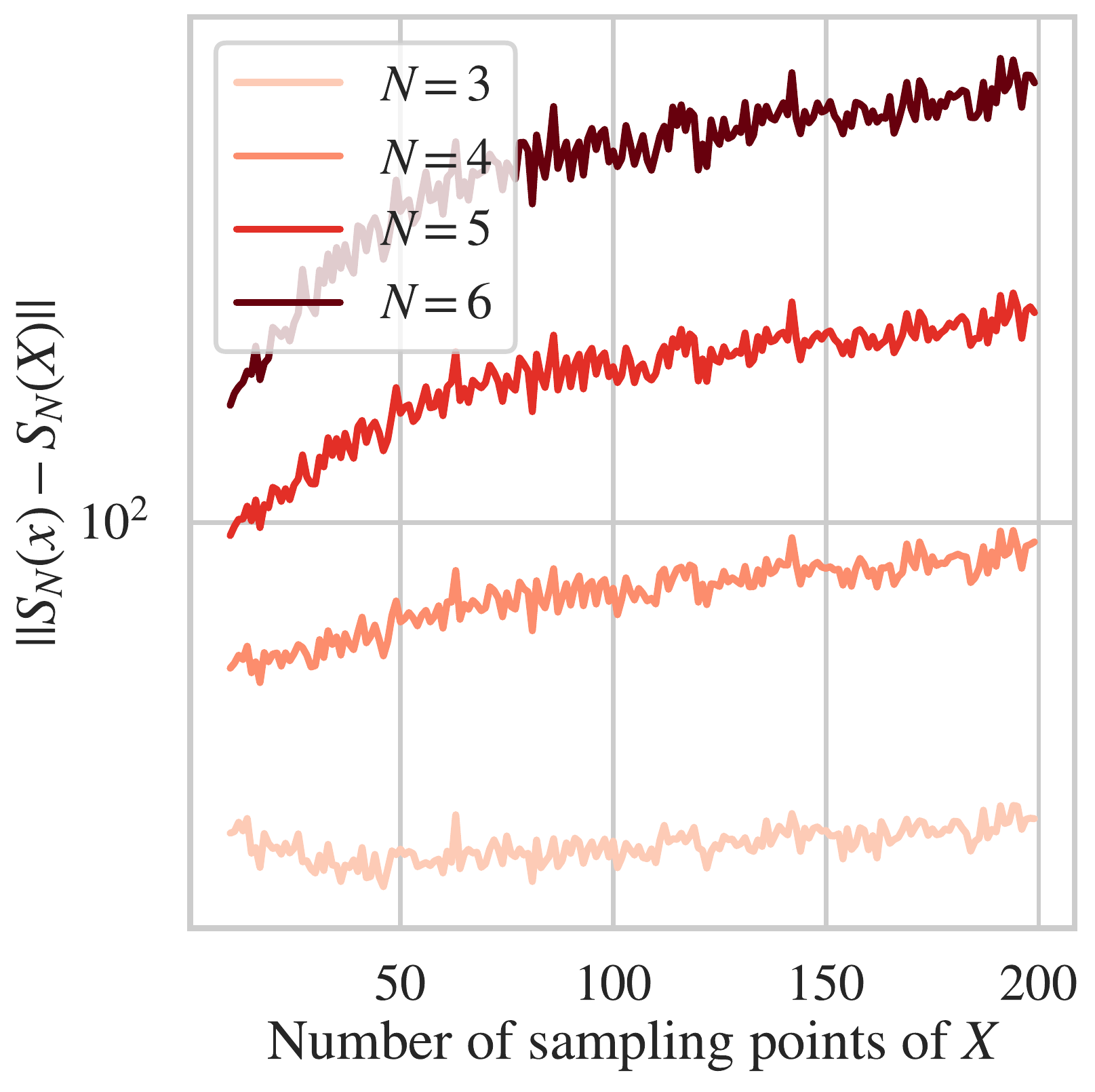}
    \caption{Difference between the signature of a continuous path $x$ and the signature of its discretized and noisy counterpart $X$, without noise on the discretization points (left), with noise of variance $v_\xi = 0.08^2$ (middle) and with noise of variance $v_\xi = 0.5^2$. For every number of sampling points, we average the distance between the two signature over $50$ randomly chosen discretizations of the interval $[0,1]$. The discretized path is generated as in the well-specified setting (see Appendix \ref{appendix:details_well_specified}).}
    \label{fig:sig_convergence}
\end{figure}

\subsection{Proof of Lemma \ref{lemma:disc_error}}

\label{proof:disc_error}

First, recall that for a generic path $x:[0,1] \rightarrow \mathbb{R}^d$, a modulus of continuity is a continuous function $\omega_x:\mathbb{R}_{\geq 0} \rightarrow \mathbb{R}_{\geq 0}$ vanishing at $0$ such that for all $s, t \in [0,1]$
\[
\norm{x_t-x_s} \leq \omega_x(\abs{t-s}).
\]

Also recall that by Heine's theorem, we can define such a modulus of continuity for every continuous mapping $[0,1]$ to $\mathbb{R}^d$. 

We start by giving a general lemma that bounds the difference between the signature layers of a path and its discretized version. Its proof is based on the results of the previous section.

\begin{lemma}
\label{lemma:distance_two_sigs}
Let $x \in C_L^{\textnormal{1-var}}([0,1],\mathbb{R}^d)$ and $\omega_x:\mathbb{R}_{\geq 0} \to \mathbb{R}_{\geq 0}$ its modulus of continuity. Let $x^D:[0,1] \rightarrow \mathbb{R}^d$ be the path obtained by linear interpolation of the discretization of $x$ on a grid $D$ corrupted by additive noise $\xi$. Let $\mathbb{X}^{k}_{[0,t]}$ and $\mathbb{X}^{k,D}_{[0,t]}$ be their respective $k$-th layers of signature on $[0,t]$. Then for all $k\geq 2$
\[
\norm{\mathbb{X}^{k} - \mathbb{X}^{k,D}}_{\infty, [0,1]} \leq 2L^{k-1}\sum\limits_{j=1}^{k-1}\frac{1}{j!} \Big(\max\limits_{0 \leq s \leq \abs{D}} \omega_x(s) + \max\limits_{t \in D} \norm{\xi_t}\Big),
\]
and for $k=1$
\[
\norm{\mathbb{X}^{1} - \mathbb{X}^{1,D}}_{\infty,[0,1]} \leq 2\Big(\max\limits_{0 \leq s \leq \abs{D}} \omega_x(s) + \max\limits_{t \in D} \norm{\xi_t}\Big).
\]

\end{lemma}

\begin{proof}
Theorem \ref{thm:layerbound} yields for $k \geq 2$
\[
\norm{\mathbb{X}^{k} - \mathbb{X}^{k,D}}_{\infty, [0,1]} \leq 2L^{k-1} \sum\limits_{j=1}^{k-1}\frac{1}{j!} \norm{x-x^D}_{\infty,[0,t]}
\]

Now, remark that 
\begin{align}
 \norm{x-x^D}_{\infty,[0,1]} \leq \norm{x - \Tilde{x}}_{\infty,[0,1]} + \max\limits_{t \in D} \norm{\xi_t}
\end{align}
from the triangular inequality, where $\Tilde{x}$ is the piecewise linear path obtained by linear interpolation of $x_{0},x_{t_1},\dots, x_{t_j}$. Now, since the paths $x$ and $\Tilde{x}$ coincide on $0,t_1,\dots,t_j$, we have
\begin{align*}
    \norm{x - \Tilde{x}}_{\infty,[0,1]} = \max\limits_{i=0,\dots,j-1} \norm{x-\Tilde{x}}_{\infty,[t_i,t_{i+1}]} \leq \max\limits_{i=0,\dots,j-1} \omega_x\big(\abs{t_{i+1} - t_{i}}\big) = \max\limits_{0 \leq s \leq \abs{D}} \omega_x(s).
\end{align*}

This gives us 
\begin{align*}
\norm{\mathbb{X}^{k}- \mathbb{X}^{k,D}}_{\infty,[0,1]} \leq 2L^{k-1} \sum\limits_{j=1}^{k-1} \frac{1}{j!}\Big(\max\limits_{0 \leq s \leq \abs{D}} \omega_x(s) + \max\limits_{t \in D} \norm{\xi_t}\Big).
\end{align*}

For the case $k=1$, we immediately get 
\[
\norm{\mathbb{X}^{1} - \mathbb{X}^{1,D}}_{\infty,[0,1]} \leq 2\Big(\max\limits_{0 \leq s \leq \abs{D}} \omega_x(s) + \max\limits_{t \in D} \norm{\xi_t}\Big)
\]
using the same technique as above.
\end{proof}

This result is illustrated in Figure \ref{fig:sig_convergence}. One can notice that as predicted by our theoretical bounds, the convergence of signature of high order happens at the same rate than the convergence of signatures of lower order. However, the multiplicative constant controlling the tightness of the bound increases with $N$, leading to a slower convergence when $N$ increases. Strong noise hinders the convergence of the signature of the discretized path since in this case, the noise's variance is independent of the number of sampling points : adding more sampling points means adding more noise. There are therefore two trade-offs when learning with signatures. A first trade-off is between sampling frequency and order: with paths sampled at low resolution, one should prefer lower order signatures, which trade model complexity against precise features. A second trade-off is between sampling and noise: if the feature time series are very noisy, the precision of the features increases up to a certain point, past which noise prevails.    

With this result in hand, we can now prove Lemma \ref{lemma:disc_error}.

\begin{proof}
We restrict ourselves to the $\omega$-Lipschitz case.

In our setup, after linearly interpolation the time series to obtain $x^D$, we normalize it by its total variation $\norm{x^D}_{\textnormal{1-var},[0,1]}$, which is a standard practice when learning with signatures \citep{morrill2020generalised}. This means that we compute the signature of the path 
\begin{align}
\frac{1}{\norm{x^D}}_{\textnormal{1-var},[0,t_j]}x^D.
\end{align}

Theorem \ref{thm:layerbound} gets us for $k \geq 2$ 
\begin{align}
\norm{\mathbb{X}^{k} - \mathbb{X}^{k,D}}_{\infty, [0,1]} & \leq 2L^{k-1} \sum\limits_{j=1}^{k-1}\frac{1}{j!} \norm{x- \frac{1}{\norm{x^D}}_{\textnormal{1-var},[0,1]} x^D}_{\infty,[0,1]}\\
& \leq 2L^{k-1} \sum\limits_{j=1}^{k-1}\frac{1}{j!} \norm{x- x^D}_{\infty,[0,1]} + 2L^{k-1} \sum\limits_{j=1}^{k-1}\frac{1}{j!} \norm{x^D- \frac{1}{\norm{x^D}}_{\textnormal{1-var},[0,1]}x^D}_{\infty,[0,1]}.
\end{align}

The first term can be bounded by using the fact that in our setting, $\omega_x(s) = \omega s$, and we thus get
\begin{align}
    2L^{k-1} \sum\limits_{j=1}^{k-1}\frac{1}{j!} \norm{x- x^D}_{\infty,[0,1]} \leq 2L^{k-1} \sum\limits_{j=0}^{k-1} \frac{1}{j!} \Big( \omega \abs{D}+ \max\limits_{t \in D} \norm{\xi_t}\Big) \leq 2L^{k-1}e  \Big( \omega \abs{D}+ \max\limits_{t \in D} \norm{\xi_t}\Big) 
\end{align}

The second term can be bounded by
\begin{align}
    2L^{k-1} \sum\limits_{j=1}^{k-1}\frac{1}{j!} \norm{x^D- \frac{1}{\norm{x^D}}_{\textnormal{1-var},[0,1]}x^D}_{\infty,[0,t]} & \leq  
    2L^{k-1} \sum\limits_{j=1}^{k-1}\frac{1}{j!} \norm{\Big(1- \frac{1}{\norm{x^D}}_{\textnormal{1-var},[0,1]}\Big)x^D}_{\infty,[0,t]}\\
    & \leq 2L^{k-1} \sum\limits_{j=1}^{k-1}\frac{1}{j!} \Big|1-\frac{1}{\norm{x^D}_{\textnormal{1-var},[0,1]}}\Big| \norm{x^D}_{\infty,[0,t]}.
    \end{align}

In order to bound 
\[
\Big|1-\frac{1}{\norm{x^D}_{\textnormal{1-var},[0,1]}}\Big| = \Bigg|\frac{\norm{x^D}_{\textnormal{1-var},[0,1]}-1}{\norm{x^D}_{\textnormal{1-var},[0,1]}}\Bigg|,
\]
we need both an upper and a lower bound on $\norm{x^D}_{\textnormal{1-var},[0,1]}$. 

Remark that 
\begin{align}
    \norm{x^D}_{\textnormal{1-var},[0,t_j]} = \sum\limits_{t_u,t_{u-1} \in D} \norm{x_{t_u}+ \xi_{t_{u}}-x_{t_{u-1}}-\xi_{t_{u-1}}}.
\end{align}
Recall that we assume the path $(x_t)$ to be time-augmented, and that the measurement times are not noisy. This means that 
\begin{align}
    \sum\limits_{t_u,t_{u-1} \in D} \norm{x_{t_u}+ \xi_{t_{u}}-x_{t_{u-1}}-\xi_{t_{u-1}}} \geq \sum\limits_{t_u,t_{u-1} \in D} \abs{t_u-t_{u-1}} \geq t_2-t_1 = t_2
\end{align}
since $t_1 = 0$ and Assumption \ref{assump:sampling_grid} guarantees that there are at least two sampling points in every grid. This gives us that 
\begin{align}
    \frac{1}{\norm{x^D}}_{\textnormal{1-var},[0,1]} \leq \frac{1}{t_2} \leq \frac{1}{\eta},
\end{align}
since we have required that the last sampling time is at least $\eta$ in Assumption \ref{assump:sampling_grid}. Turning to the upper bound, we get that 
\begin{align*}
\Big| 1-\norm{x^D}_{\textnormal{1-var},[0,1]} \Big| \leq 1 - L + \sum\limits_{t_u, t_{u-1}\in D} \norm{\xi_u - \xi_{u-1}}
\end{align*}
by definition of the total variation of a piecewise linear path. Finally, 
\begin{align}
    \sum\limits_{t_u, t_{u-1}\in D} \norm{\xi_u - \xi_{u-1}} \leq 2 \# D \max_{t \in D} \norm{\xi_t}.
\end{align}

Putting everything together gives us 
\begin{align}
    \Big|1-\frac{1}{\norm{x^D}_{\textnormal{1-var},[0,1]}}\Big| \leq \frac{1 - L +  2 \# D \max\limits_{t \in D} \norm{\xi_t}}{\eta}.
\end{align}

In the end, we get that 
\begin{align}
    2L^{k-1} \sum\limits_{j=1}^{k-1}\frac{1}{j!} \Big|1-\frac{1}{\norm{x^D}_{\textnormal{1-var},[0,1]}}\Big| \norm{x^D}_{\infty,[0,1]} \leq 2L^{k-1} e\frac{1 - L +  2 \# D \max\limits_{t \in D} \norm{\xi_t}}{\eta} \norm{x^D}_{\infty,[0,1]}.
\end{align}

Now, remark that since the path $x^D$ is piecewise linear,
\begin{align}
    \norm{x^D}_{\infty,[0,1]}  = \norm{x^D-x_0 + x_0}_{\infty,[0,1]} & \leq \max\limits_{t \in D} \norm{x_t + \xi_t - x_0} + \norm{x_0}\\
    & \leq \max\limits_{t \in D} \norm{x_t - x_0} + \max\limits_{t \in D} \norm{\xi_t} + \norm{x_0} \\
    & \leq \norm{x_0} + L + \max\limits_{t \in D} \norm{\xi_t}
\end{align}

where the inequality 
\[
\max\limits_{t \in D} \norm{x_t - x_0} \leq L
\]
follows from the definition of the total variation.

This means that 
\begin{align}
    2L^{k-1} \sum\limits_{j=1}^{k-1}\frac{1}{j!} \Big|1-\frac{1}{\norm{x^D}_{\textnormal{1-var},[0,1]}}\Big| \norm{x^D}_{\infty,[0,1]} \leq 2L^{k-1} e\frac{1 - L +  2 \# D \max\limits_{t \in D} \norm{\xi_t}}{\eta} \Big(\norm{x_0} + L + \max\limits_{t \in D} \norm{\xi_t}\Big).
\end{align}

We have written these inequalities for a generic random variable. Let us now consider individual observations of our dataset.

On the set $A_\xi(\delta)$, one has
\begin{align}
\max\limits_{i=1,\dots,n,t\in D^i}\norm{\xi^i_t} \leq v_\xi \sqrt{d} + v_\xi  \sqrt{c^{-1}\log(\delta^{-1}\# \mathcal{D})}.
\end{align}

To simplify notations, let us write
\begin{align}
    C_\delta := v_\xi \sqrt{d} + v_\xi  \sqrt{c^{-1}\log(\delta^{-1}\# \mathcal{D})}.
\end{align}

We get 
\begin{align}
2L^{k-1} \sum\limits_{j=1}^{k-1}\frac{1}{j!} \norm{x^i- x^{i,D}}_{\infty,[0,1]} \leq 2L^{k-1}e  \Big( \omega \abs{\mathcal{D}}+C_\delta \Big), 
\end{align}
where we recall that $\mathcal{D}$ is the collection of individual grids, and $\abs{\mathcal{D}}$ is the biggest sampling gap among individuals. Similarly,
\begin{align}
    2L^{k-1} \sum\limits_{j=1}^{k-1}\frac{1}{j!} \Big|1-\frac{1}{\norm{x^D}_{\textnormal{1-var},[0,1]}}\Big| \norm{x^D}_{\infty,[0,1]} \leq 2L^{k-1}e \frac{1-L+2\#\mathcal{D}C_\delta}{\eta}\big(\norm{x_0}+L+C_\delta\big)
\end{align}

Now moving to the feature matrices, we have
\begin{align*}
    \frac{1}{M}\norm{(\mathbf{S}_N - \mathbf{S}^\mathcal{D}_N)\theta^\ast_N}^2_F & \leq \frac{1}{M} \sum\limits_{i=1}^n \sum\limits_{k=0}^N \norm{(\mathbf{S}_{i,[k]} - \mathbf{S}^\mathcal{D}_{i,[k]})\theta^*_{[k],\cdot}}_F^2\\
    & \leq \frac{1}{M} \sum\limits_{i=1}^n \sum\limits_{t \in \bar{D}^i} \sum\limits_{k=0}^N d^k \Lambda_k(\mathbf{F})^2\big(2eL^{k-1}\big(\omega \abs{\mathcal{D}} + C_\delta + \frac{1-L+ 2\# \mathcal{D} C_\delta}{\eta}\big(\norm{x_0}+L+C_\delta\big)\big)\Big)^2\\
    & \leq 4 e^2 \Bigg(\omega \abs{\mathcal{D}} + C_\delta + \frac{1-L+ 2\# \mathcal{D} C_\delta}{\eta}\big(\norm{x_0}+L+C_\delta\big)\Bigg)^2 L^2 \sum\limits_{k=0}^N \frac{d^k \Lambda_k(\mathbf{F})^2}{k!^2} \times k!^2\\
    & \leq 4 e^2 N!^2 \Bigg(\omega \abs{\mathcal{D}} + C_\delta + \frac{1-L+ 2\# \mathcal{D} C_\delta}{\eta}\big(\norm{x_0}+L+C_\delta\big)\Bigg)^2 L^2  \sum\limits_{k=0}^N \frac{d^k \Lambda_k(\mathbf{F})^2}{k!^2}. 
\end{align*}

Writing 
\[
C_{\mathcal{D},N}(\delta) = 4e^2 L^2 N!^2 \Bigg(\omega \abs{\mathcal{D}} + C_\delta + \frac{1-L+ 2\# \mathcal{D} C_\delta}{\eta}\big(\norm{x_0}+L+C_\delta\big)\Bigg)^2,  
\]
one finally gets with probability $1-\delta$ that 
\begin{align*}
    \frac{1}{M}\norm{(\mathbf{S}_N - \mathbf{S}^\mathcal{D}_N)\theta^\ast_N}^2_F \leq C_{\mathcal{D},N}(\delta) \sum \limits_{k=0}^N  \frac{d^k \Lambda_k(\mathbf{F})^2}{k!^2}. 
\end{align*}

\end{proof}

\subsection{Proof of the main Theorem}
\label{proof:oracle_bound}
 
We finally combine all Lemmas to obtain the desired oracle bound. 

\begin{proof}
First, we have from Lemma \ref{lemma:estimation_error} that on $A_\varepsilon(\Bar{\delta})$,  
 \begin{align*}
     \frac{1}{2M}\norm{\mathbf{y} - \mathbf{S}_N^{\mathcal{D}}\widehat \theta_{N,M}}_\textnormal{F}^2 \leq \frac{1}{2M}\norm{\mathbf{y}-\mathbf{S}_N^{\mathcal{D}} \theta^\ast_N}_\textnormal{F}^2 + \frac{2pC_N(\Bar{\delta})}{\sqrt{M}}\sum_{k=0}^N \frac{d^k \Lambda_k(\mathbf F)}{k!}.
 \end{align*}
 The first term of the right-hand side of this inequality is bounded by 
 \begin{align*}
     \frac{1}{2M}\norm{\mathbf{y}-\mathbf{S}_N^{\mathcal{D}} \theta^\ast_N}_\textnormal{F}^2 \leq \frac{1}{M} \norm{\mathbf{y} - \mathbf{S}_N \theta^\ast_N}_\textnormal{F}^2 
    & + \frac{1}{M}\norm{\mathbf{S}_N \theta^\ast_N - \mathbf{S}_N^{\mathcal{D}} \theta^\ast_N }_\textnormal{F}^2.
 \end{align*}
 
 By Lemma \ref{lemma:bias_bound} and Lemma \ref{lemma:disc_error}, this can in turn be bounded on $A_\varepsilon (\Bar{\delta}) \cap A_\xi(\delta)$ by
 \begin{align*}
    \frac{1}{2M}\norm{\mathbf{y}-\mathbf{S}_N^{\mathcal{D}} \theta^\ast_N}_\textnormal{F}^2 \leq  \Bigg(\frac{d^{N+1}\Lambda_{N+1}(\mathbf{F})}{(N+1)!}\Bigg)^2 + C_{\mathcal{D},N}(\delta)  \sum \limits_{k=0}^N  \frac{d^k \Lambda_k(\mathbf{F})^2}{k!^2}
 \end{align*}
 Combining all the pieces, this finally gives us, on $A_\varepsilon (\Bar{\delta}) \cap A_\xi(\delta)$,  
\begin{align*}
\frac{1}{2M}\norm{\mathbf{y} - \mathbf{S}_N^{\mathcal{D}}\widehat \theta_{N,M}}_\textnormal{F}^2 & \leq \Bigg(\frac{d^{N+1}\Lambda_{N+1}(\mathbf{F})}{(N+1)!}\Bigg)^2 \\
& +C_{\mathcal{D},N}(\delta)  \sum \limits_{k=0}^N  \frac{d^k \Lambda_k(\mathbf{F})^2}{k!^2} \\
& + \frac{2pC_N(\Bar{\delta})}{\sqrt{M}}\sum_{k=0}^N \frac{d^k \Lambda_k(\mathbf F)}{k!}.
\end{align*}
\end{proof}

\subsection{Asymptotics}

\label{appendix:assymptotics}

We briefly discuss the asymptotic behaviour of the upper bound of the oracle inequality. 
 
\paragraph{Truncation depth $N$.} A natural question is whether the bias of our estimator vanishes as $N \to \infty$. If we have perfect sampling, i.e. the limit case where $\mathcal{D}=0$ and $v_\xi =0$, our bound on the prediction error becomes on $A_\varepsilon(\Bar{\delta})$
\begin{align*}
\frac{1}{2M}\norm{\mathbf{y} - \mathbf{S}_N^{\mathcal{D}}\widehat \theta_{N,M}}_\textnormal{F}^2 & \leq \Bigg(\frac{d^{N+1}\Lambda_{N+1}(\mathbf{F})}{(N+1)!}\Bigg)^2 \\
& + \frac{2pC_N(\Bar{\delta})}{\sqrt{M}}\sum_{k=0}^N \frac{d^k \Lambda_k(\mathbf F)}{k!}.
\end{align*}

The first term of this bound vanishes as an immediate consequence of Assumption \ref{assump:decay_derivatives_F}, while the second term is a statistical error term that behaves like $\frac{\sqrt{\log (Nd^N)}}{\sqrt{M}}$. In order to obtain an asymptotic convergence, we thus need that $N\log(dN) = o(M)$.

In the more realistic setting where $\abs{\mathcal{D}} > 0$, the discretization bias behaves like $L^{N-1} N! \abs{\mathcal{D}}$. It is thus sufficient to assume that $\abs{D} =o(1/N!)$. If $v_\xi >0$, our estimator is durably biased due to the measurement noise, and this bias increases with $N \to \infty$. This is due to a "propagation of chaos" phenomenon: the difference between the unobserved feature path and the interpolated feature time series is amplified by taking the successive iterated integrals that define the signature. This advocates for using simple, low-order signature models in the presence of noise, as the gain in precision obtained when taking higher $N$ and reducing the truncation bias will at some point be lost because of the amplified noise. 

\paragraph{Dimension $p$ of the target path.} Our oracle bound only depends on $p$ through the statistical error term. This term is proportional to $p \sqrt{\log p}$, which is expected in multitask regression. 

\paragraph{Dimension $d$ of the feature path.} Our oracle bound exhibits multiple dependencies in $d$. First, the truncation bias grows polynomially with $d$. Similarly, the discretization bias also depends polynomially on $d$. Finally, the statistical error term is proportional to $\log d$ times a polynomial term.

\section{Algorithms, experiments, and supplementary results}
\label{appendix:experiments}

\subsection{Implementation details}
\label{appendix:details_algo}

\label{appendix:layer_penalty}

Recall that the SigLasso estimator $\widehat{\theta}_{N,M}$ is defined as  
\begin{align*}
     \widehat  \theta_{N,M} \in \argmin\limits_{\theta \in \mathbb{R}^{ s_d(N)\times p}} \frac{1}{2M}\norm{\mathbf{Y}-\mathbf{S}_N^{\mathcal{D}}\theta}_{\textnormal{F}}^2 + \Omega(\theta),
 \end{align*}
 
 where 

\begin{align*}
    \Omega(\theta) = \sum\limits_{k=0}^N \frac{M^{\frac{1}{2}}C_k(\bar{\delta})}{k!} \norm{\theta_{[k],\cdot}}_1,
\end{align*}

and 
\begin{align*}
C_k(\bar{\delta}) = \sqrt{v_\varepsilon \log(2pNd^k / \Bar{\delta})} 
\end{align*}

for $\bar{\delta} \in [0,1]$. Our goal is first to rewrite the penalty $\Omega(\theta)$ as 
\[
\Omega(\theta) = C \sum\limits_{k=0}^N \lambda_k \norm{\theta_{[k],\cdot}}_1,
\]
such that training will only require to scale each layer of $\theta$ and to crossvalidate the multiplicative constant $C$. Since for $k \geq 1$,
\begin{align*}
    C_k(\bar{\delta}) = \sqrt{k} \times \sqrt{v_\varepsilon(\bar{\delta}/k + \log (pN)/k + \log d )} \leq \sqrt{k} \times \sqrt{v_\varepsilon(\bar{\delta} + \log (pN) + \log d )},
\end{align*}
we let $\lambda_k = \frac{\sqrt{k}}{k!}$. 

We now show that the minimization problem with layer-specific penalty can be written as a standard regression problem with $\ell_1$ penalization by rescaling the feature matrix, that is, multiplying $\mathbf{S}_N^\mathcal{D}$ by a well-chosen diagonal matrix. Consider the $\ell_1$-penalized problem \begin{align*}
    \min\limits_{\theta \in \mathbb{R}^{s_d(N) \times p}} \frac{1}{2M}\norm{\mathbf{Y}-\mathbf{S}_N^{\mathcal{D}}\theta}_{\textnormal{F}}^2 + C \sum\limits_{k=0}^N \lambda_k \norm{\theta_{[k],\cdot}}_1,
\end{align*}
where $C >0$ controls the strength of the penalization. 

Making the change of variable $$\Tilde{\theta} = \text{diag}(1,\underbrace{\lambda_1,\dots,\lambda_1}_{d \text{ repetitions}}, \underbrace{\lambda_2,\dots,\lambda_2}_{d^2 \text{ repetitions}},\dots,\underbrace{\lambda_k,\dots,\lambda_k}_{d^N \text{ repetitions}}) \theta,$$
which is equivalent to
 $$\theta = \text{diag}(1,1/\lambda_1,\dots,1/\lambda_1,\dots,1/\lambda_k,\dots,1/\lambda_k) \Tilde{\theta},$$
 and denoting by $W$ this last weight matrix,
 we get the equivalent minimization problem
 \begin{align*}
    \min\limits_{\Tilde{\theta} \in \mathbb{R}^{ s_d(N) \times p}} \frac{1}{2M}\norm{\mathbf{Y}-\mathbf{S}_N^{\mathcal{D}}W\Tilde{\theta}}_{\textnormal{F}}^2 + C \sum\limits_{k=0}^N \norm{\Tilde{\theta}_{[k],\cdot}}_1.
 \end{align*}
 
 We can thus obtain the SigLasso estimator by (i) multiplying the feature matrix $\mathbf{S}_N^\mathcal{D}$ by $W$ and solving the associated  $\ell_1$-penalized problem (ii) multiplying the obtained solution by $W$.  

The Learn-And-Reconstruct algorithm is the generic algorithm used in our work. It is applicable for a wide variety of tasks such as missing values inference, trajectory reconstruction, forecasting and many more. It is described in Algorithm \ref{alg:learn_and_reconstruct}.

\begin{algorithm}[h!]
   \caption{Learn-and-Reconstruct Algorithm. The algorithm infers for every individual in the test set a reconstructed time series $\Hat{Y}^i_t$.}
   \label{alg:learn_and_reconstruct}
\begin{algorithmic}
    \STATE {\bfseries 1. Learn the dynamics}
   \STATE {\bfseries Input:} train dataset of normalized paths $(\mathbf{X}^1,\mathbf{Y}^1),\dots,(\mathbf{X}^n,\mathbf{Y}^n)$ sampled on $(D^1,\Bar{D}^1),\dots,(D^n,\Bar{D}^n)$. 
   \STATE {\bfseries Construct the feature matrix} $\mathbf{S}_N^\mathcal{D}$ {\bfseries and the target vector} $\mathbf{Y}$
   \FOR{$i=1$ {\bfseries to} $n$}
        \FOR {$t$ {in} $\Bar{D}^i$}
   \STATE $\mathbf{S}_N^\mathcal{D}$ $\leftarrow $ Append $S_N(X^i_{[0,t]})$ 
   \STATE $\mathbf{Y}$ $\leftarrow $ Append $Y^i_{t}$
   \ENDFOR
   \ENDFOR
   \STATE{\bfseries Compute } $\hat{\theta}_{N,M}$ by solving \eqref{eq:optimization_problem} with $\mathbf{Y},\mathbf{S}_N^\mathcal{D}$ using coordinate descent.
   \STATE{\bfseries 2. Reconstruct trajectories}
   \STATE {\bfseries Input:} test dataset $\Tilde{\mathbf{X}}^1,\dots,\Tilde{\mathbf{X}}^n$ sampled on $\Tilde{D}^1,\dots,\Tilde{D}^n$.
   \FOR{$i=1$ {\bfseries to} $n$}
   \FOR{$t$ {\bfseries in} $\Tilde{D}^i$}
   \STATE $\Hat{Y}^i_t = \Hat{\theta}_{N,M}S_N(\Tilde{X}^i_{[0,t]})$
   \ENDFOR 
   \ENDFOR
\end{algorithmic}
\end{algorithm}
 
 \subsection{Assessing feature importance}
 
 We two metrics used to asses the importance of the different dimensions of the feature path.
 
 Given a truncation depth $N$ and a dimension $i \in \left\{1,\dots, d \right\}$, we define its pure feature importance (PFI) as the sum of the norm of the coefficients (or vectors in the case the target is multivariate) of $\widehat{\theta}_{N,M}$ that are associated to signatures taken on the words $I_1 = (i)$, $I_2 = (i,i)$, and so forth until $I_N =(i,\dots,i)$. Mathematically, 
 \[
 PFI(i) = \frac{1}{N}\Big( \norm{\theta^{I_1}}_2 + \norm{\theta^{I_2}}_2 + \dots + \norm{\theta^{I_N}}_2\Big).
 \]

Since signatures also capture interactions between dimensions of the feature path, we also define the cross feature importance (CFI) as the sum of norms of the coefficients (or vectors) of $\widehat{\theta}_{N,M}$ that are associated to signatures coefficients of words of length $\leq N$ in which the letter $i$ appears. Mathematically, 
\[
CFI(i) = \frac{1}{s_d(N) - s_{d-1}(N)}\sum\limits_{I \text{ s.t. } i \,\in\, I} \norm{\theta^I}_2.
\]

For a given truncation depth $N$, note that there are $s_d(N) - s_{d-1}(N) = \sum\limits_{k=0}^N d^k - \sum\limits_{k=0}^N (d-1)^k$ terms in the last sum, which justifies our choice of normalization. 

\subsection{Details on model implementation and evaluation}

\paragraph{SigLasso.} The SigLasso model is implemented using the \texttt{CVLasso} class in \texttt{scikit-learn} \citep{pedregosa2011scikit}. This implementation optimises the objective function using coordinate descent and features automatic cross-validation of the penalty strength. We use \texttt{iisignature} \citep{reizenstein2021algorithm} to compute the signature of the feature time series. Every time series is standardized prior to this through division by its own total variation, as suggested by \citet{morrill2020generalised}. The depth of the signature is a hyperparameter chosen between $2$ and $9$ or $6$ depending on the experiment. An intercept is added.

\paragraph{GRU.} The GRU is of width $128$ and systematically trained with $100$ epoches using a learning rate of $0.001$.

\paragraph{Neural CDE.} We use the implementation of Neural CDE provided by \texttt{torchcde} \citep{kidger2020neural}. We use the \href{https://github.com/patrick-kidger/torchcde/blob/master/example/irregular_data.py}{original vector field} described in the documentation of this package, with the small tweak that we use a smoother non-linearity ($\text{tanh}$ instead of $\text{ReLU}$). We observed that using the \texttt{rk4} solver instead of \texttt{dopri5} significantly accelerates the training time of the Neural CDE without affecting the model's performances. The learning rate is hand-tuned to either $0.001$ or $0.0001$ depending on the experiment. We train the model for $100$ epochs and asses its convergence by using a standard stopping criteria. 

\textbf{Metrics.} The MSE is computed in a classical fashion. To compute the integrate MSE, we compute the $L^2$ distance between the piecewise constant interpolations of the true $y_t$ and the predicted $\widehat{y}_t$.

\subsection{Details on the well specified model}
\label{appendix:details_well_specified}

\paragraph{Generation of the training data.} We generate a two-dimensional feature path by interpolating for every dimension $15$ points in $[0,1]$, each of them being draw randomly for a normal distribution $\mathcal{N}(0,1)$. The interpolation is done with Hermite cubic splines with backward differences using the package \texttt{torchcde} ~\citep{kidger2020neural}. Time is added as a supplementary channel, which is a standard practice when learning with signatures and Neural CDEs. These paths are then downsampled by randomly drawing sampling points for the target and the feature time series specific to every individual. The target path is the solution of a CDE of the form 
\[
dy_t = \sigma\Big(Ay_t\Big)dx_t
\]
where $\sigma:\mathbb{R}^p \to \mathbb{R}^p$ is the hyperbolic tangent, $A \in \mathbb{R}^{d \times p}$ is a matrix drawn randomly from $\mathcal{N}(0,I_{d \times p})$ and $x_t$ is the feature path constructed as above. The solution of this CDE is computed using \texttt{torchcde}.

\paragraph{Generation of the test data.} We generate the test data in the same way than the training data. However, this data is not downsampled as we wish to assess the generalization capacities of our model---i.e., is our model capable of approximating the dynamics and extrapolating to continuous feature paths.  

\subsection{Details on Ornstein-Uhlenbeck experiment }
\label{appendix:details_OU}

We take $(x_t)_{t \in [0,1]}$ to be a 1-dimensional Brownian motion with variance $\sigma^2 = 0.1$, and generate $(y_t)$ as a $1$-dimensional Ornstein-Uhlenbeck process driven by $(x_t)$, that is, for all $t \in [0,1]$
\[
dy_t = \theta(\mu-y_t)dt + dx_t.
\]

Simulation of $(y_t)$ is done using a standard Euler-Maruyama simulation scheme. We let $\theta=3$ and $\mu=1$. The training data is then downsampled as in the well specified experiment.

\subsection{Details on the tumor growth experiment}
\label{appendix:details_tumor_growth}

We consider the following tumor growth model taken from \citep{simeoni2004predictive}. Let $x \in C^{1-var}([0,1],\mathbb{R})$. The weight $y \in C([0,1],\mathbb{R}^+)$ under the concentration of a treatment drug $x$ is governed by the differential system
\begin{align*}
    & du^1_t = \Big[\Big(\lambda_0 u_t^1 \big[1+(\frac{\lambda_0}{\lambda_1}y_t )^\psi \big]\Big)^{-1/\psi} - k_2x_tu_t^1 \Big]dt\\
    & du^2_t = \Big[k_2x_tu_t^1-k_1u_t^2\Big]dt\\
    & du^3_t = \Big[k_1(u_t^2-u_t^3) \Big]dt\\
    & du^4_t = \Big[k_1(u_t^3-u_t^4) \Big]dt\\
    & y_t = u_t^1 + u_t^2 + u_t^3 + u_t^4
\end{align*}

with initial condition $(u_0^1,u_0^2,u_0^3,u_0^4,y_0) = (2,0,0,0,2)$ and parameters $(k_1,k_2,\lambda_0,\lambda_1,\psi) = (10,0.5,0.9,0.7,20)$. The concentration $(x_t)$ is chosen to be the squared value of the paths used for the well-specified experiment. Notice that this system is non-linear w.r.t. $x$. Indeed, writing $dy_t = \mathbf{G}(y_t,x_t)dt$, one has, for $\alpha \in \mathbb{R}$, $G(y_t,\alpha x_t) \neq  \alpha G(y_t, x_t)$. The training data is then downsampled as in the well specified experiment.

\subsection{Supplementary results}
\label{appendix:supplementary_table}
\begin{table*}[h!]
	\centering
 \begin{tiny}
	\caption{Performance of SigLasso, GRU, Neural CDE, RNN and LSTM in different simulation settings, averaged over 10 iterations. In every setting, $n=50$,  $\# \bar{D}^i = 5$ for all $i=1, \dots, n$ (and therefore $M=250$).}
    \label{table:supplementary_results}
	\begin{tabular}{lcccccccccc}
	\toprule
    & \multicolumn{5}{c}{$L_2$ error}& \multicolumn{5}{c}{MSE on last point} \\
    \cmidrule(lr){2-6}
    \cmidrule(lr){7-11}
    Setting &  SigLasso & GRU & Neural CDE & RNN & LSTM &  SigLasso & GRU & Neural CDE & RNN & LSTM \\
    \midrule 
    Well-specified & \textbf{0.13 $\pm$ 0.07} & 1.05 $\pm$ 0.42 & 0.61 $\pm$ 0.38 & 1.16 $\pm$ 0.45 & 0.87 $\pm$ 0.67 & \textbf{0.73 $\pm$ 0.56} & 3.32 $\pm$ 1.60 &  1.46 $\pm$ 1.20 & 3.56 $\pm$ 1.43 & 2.41 $\pm$ 1.75  \\
    Ill-specified & \textbf{0.15 $\pm$ 0.02} & 0.24 $\pm$ 0.11 & 0.29 $\pm$ 0.15 & 0.18 $\pm$ 0.006 & 0.20 $\pm$ 0.01 & \textbf{0.09 $\pm$ 0.05} & 0.19 $\pm$ 0.09 &  0.22 $\pm$ 0.15 & 0.18 $\pm$ 0.05 & 0.10 $\pm$ 0.03\\
    OU  & \textbf{0.01 $\pm$ 0.02} & 0.05 $\pm$ 0.06 & 0.17 $\pm$ 0.12 & 0.11 $\pm$ 0.09 & 0.46 $\pm$ 0.48 & 0.018 $\pm$ 0.025 & 0.014 $\pm$ 0.020 & \textbf{0.013 $\pm $ 0.016 }& 0.02 $\pm$ 0.02 & 4.41 $\pm$ 3.77 \\
    Tumor growth & \textbf{0.16 $\pm$ 0.02} & 0.66 $\pm$ 0.09 & 5.29 $\pm$ 1.38 & 0.75 $\pm$ 0.03 & 0.69 $\pm$ 0.04 & \textbf{0.35 $\pm$ 0.12 } & 2.00 $\pm$ 0.38 & 8.76 $\pm$ 9.26 & 2.72 $\pm$ 0.24 & 2.25 $\pm$ 0.29\\
	\bottomrule           
	\end{tabular}
 \end{tiny}
\end{table*}

\subsection{Details on the French Covid experiment}
\label{appendix:french_covid}

 We illustrate the performance of our method and competitors on French Covid data from 2021-03-31 to 2021-07-07 available on \href{https://gitlab.pasteur.fr/mmmi-pasteur/COVID19-Ensemble-model}{Gitlab}. Hospital data was obtained from the SI-VIC database, the national inpatient surveillance system.  
 
 \paragraph{Target path.} Following~\citet{paireau2022ensemble}, we chose the predict the growth rate of incident hospitalisations in each of the 9 metropolitan regions of France. The exponential growth rate was computed from raw data using a 2 days rolling window and then smoothed using local polynomial regression as in~\citet{paireau2022ensemble}. Mathematically, our target time series is the $\mathbb{R}$-valued growth rate, and we fit a different model for every of the 12 regions. It is displayed for all 12 regions in Figure~\ref{fig:covid_target}.
 
 \paragraph{Feature path.} As in~\citet{paireau2022ensemble}, we consider a set of 12 time-dependant predictors of different types summarized in Table~\ref{tab:predictors_sources} and plotted in Figure~\ref{fig:covid_features}.  Both \href{https://sidep.gouv.fr/}{SIDEP} (''Syst\`eme d’Information de D\'epistage Populationnel'') and \href{https://www.data.gouv.fr/fr/datasets/donnees-relatives-aux-personnes-vaccinees-contre-la-covid-19-1/}{VAC-SI} datasets are publicly available.
 The mobility data was obtained from \href{https://www.google.com/covid19/mobility/}{Google}. The mobility-related predictors describe travel trends for different kind of public spaces such as such as shops and leisure spaces, food stores and pharmacies, parks, public transport stations, workplaces and residential areas. The meteorological data was obtained from Météo France.
 
 \begin{table}
     \centering
     \begin{tabular}{lccc}
     \toprule
      \textbf{Predictor}&\textbf{Type}&\textbf{Source}&\textbf{Description}\\
      \midrule
       \texttt{prop\_pos\_symp}& Epidemiological&SIDEP database& proportion of positive tests among symptomatics\\
       \texttt{P\_r}& Epidemiological&SIDEP database&growth rate of positive tests\\
       \texttt{couv-complet}& Epidemiological&VAC-SI database&proportion of vaccinated  \\
       \texttt{grocery\_and\_pharmacy}& Mobility&Google& visits to grocery and pharmacy stores \\
        \texttt{parks}& Mobility&Google& visits to parks\\
        \texttt{transit\_stations}& Mobility&Google& visits to transit stations\\
       \texttt{workplaces}& Mobility&Google& visits to workplaces\\\
       \texttt{residential}& Mobility&Google& visits to residential places\\
       \texttt{IPTCC}& Meteorological&Météo France&Index PREDICT of climatic transmissivity\\
       \texttt{temperature}& Meteorological&Météo France& temperature\\
       \texttt{rel\_humidity}& Meteorological&Météo France& relative humidity\\
       \texttt{abs\_humidity}& Meteorological&Météo France& absolute humidity\\
       \bottomrule
     \end{tabular}
     \caption{The set of time-dependant predictors used to predict the hospital admission growth rate}
     \label{tab:predictors_sources}
 \end{table}

\begin{figure}
    \centering
    \includegraphics[width=\textwidth]{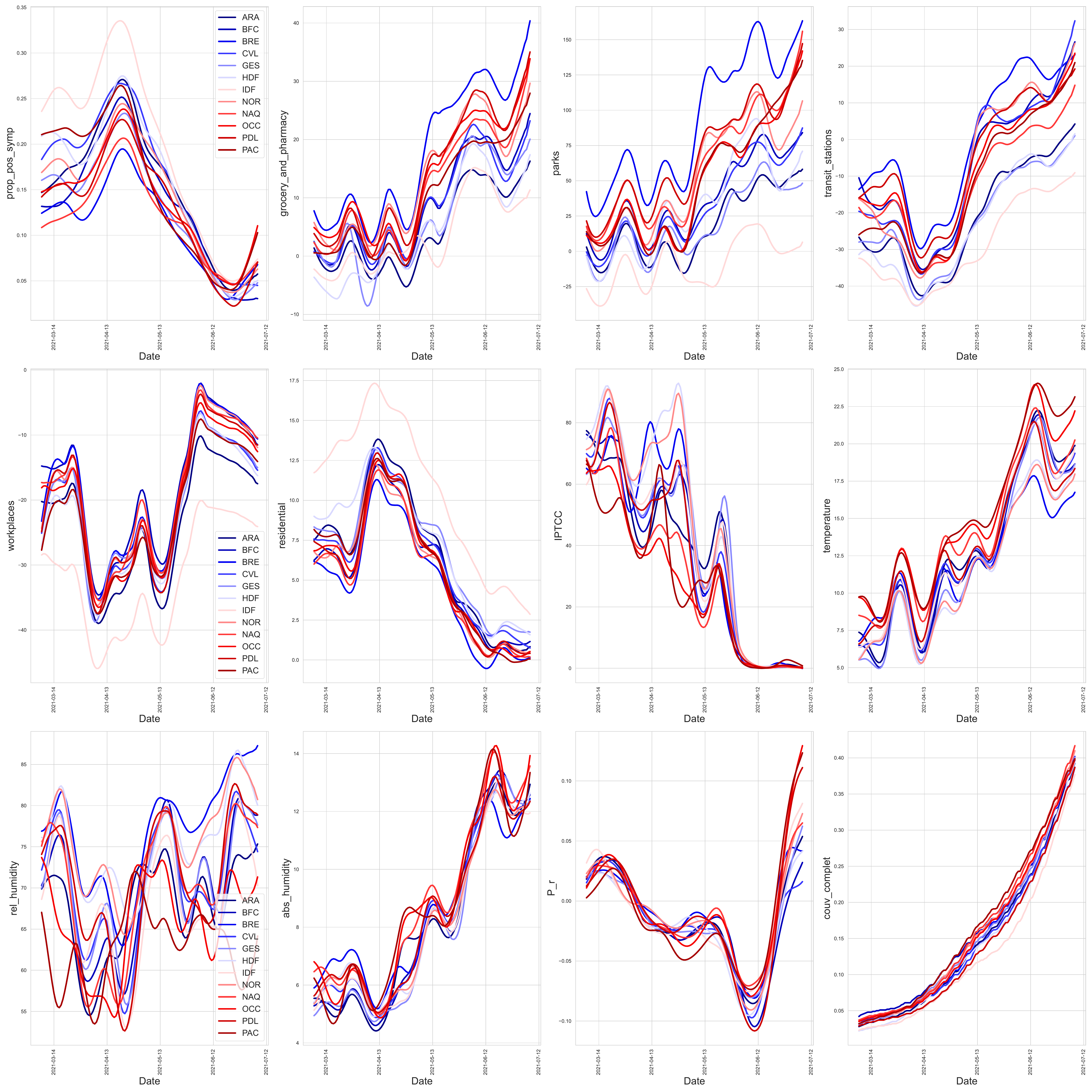}
    \caption{The 12 different feature time series used to forecast the hospitalization growth rate. Every different color corresponds to a given region of France.}
    \label{fig:covid_features}
\end{figure}

\begin{figure}
    \centering
    \includegraphics[width=\textwidth]{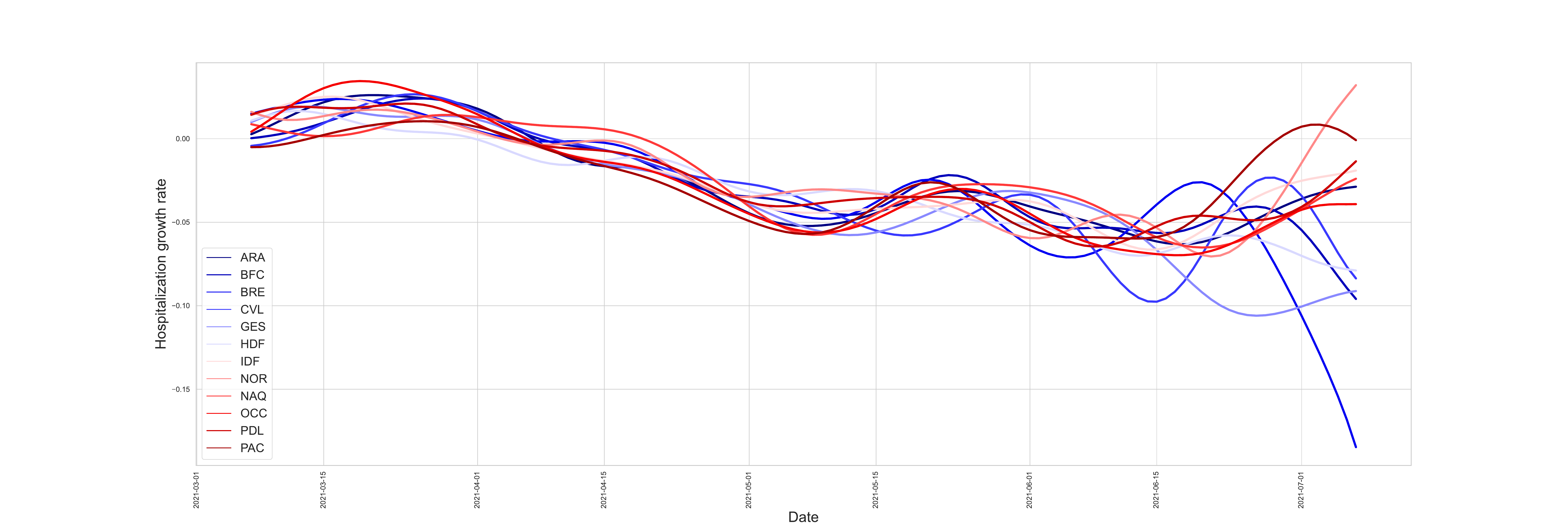}
    \caption{Hospitalization growth rate through time during the full period for the 12 different regions of France.}
    \label{fig:covid_target}
\end{figure}

\paragraph{Models.} SigLasso, NCDE and GRU algorithms were trained on the period from 2021-03-31 to 2021-06-23 and tested on the period from 2021-06-24 to 2021-07-07. We included a history of 10 days at each point and performed prediction for different horizons ranging from $1$ to $14$. In others words, at horizon $h$, features values from day $t-h-10$ to day $t-h$ to were used to compute the prediction at time $t$. All feature time series are normalized to have total variation equal to $1$.

\paragraph{Architectural details.} The GRU has width $128$ and is trained for $100$ epochs with a learning rate of $0.0001$. The NCDE is trained for $30$ epochs with a learning rate of $0.001$. It has $2$ hidden layers of width $128$, an intermediate $\textnormal{Tanh}(\cdot)$ non-linearity and a final linear readout. This architecture is identical to the one proposed in \cite{kidger2022neural}. Penalty strenght of the SigLasso is crossvalided using the internal implementation \texttt{LassoCV} of \texttt{scikit-learn} \cite{pedregosa2011scikit}. 


All details, in particular the features used for each individual prediction,  can be found in~\citet{paireau2022ensemble}. 


We refer to the supplementary information file of~\citet{paireau2022ensemble} and our code for more details.

\paragraph{Results.} Figure~\ref{fig:RMSE_covid_all_models} displays the RMSE (on all regions) of NCDE, SigLasso, GRU, and the Ensemble method for all prediction horizons $h=1,\ldots,14$. Figures \ref{fig:covid_display_weighted_siglasso}, \ref{fig:covid_display_gru} and \ref{fig:covid_display_ncde} display the obtained interpolation for SigLasso, GRU and NCDE at different horizons (corresponding to different line colors in \texttt{winter} matplotlib palette). The lighter the blue, the smaller the time horizon: the lightest curve corresponds to a time horizon equal to $h=1$. Truth is in red.

\begin{figure}
    \centering
    \includegraphics[width=0.78\textwidth]{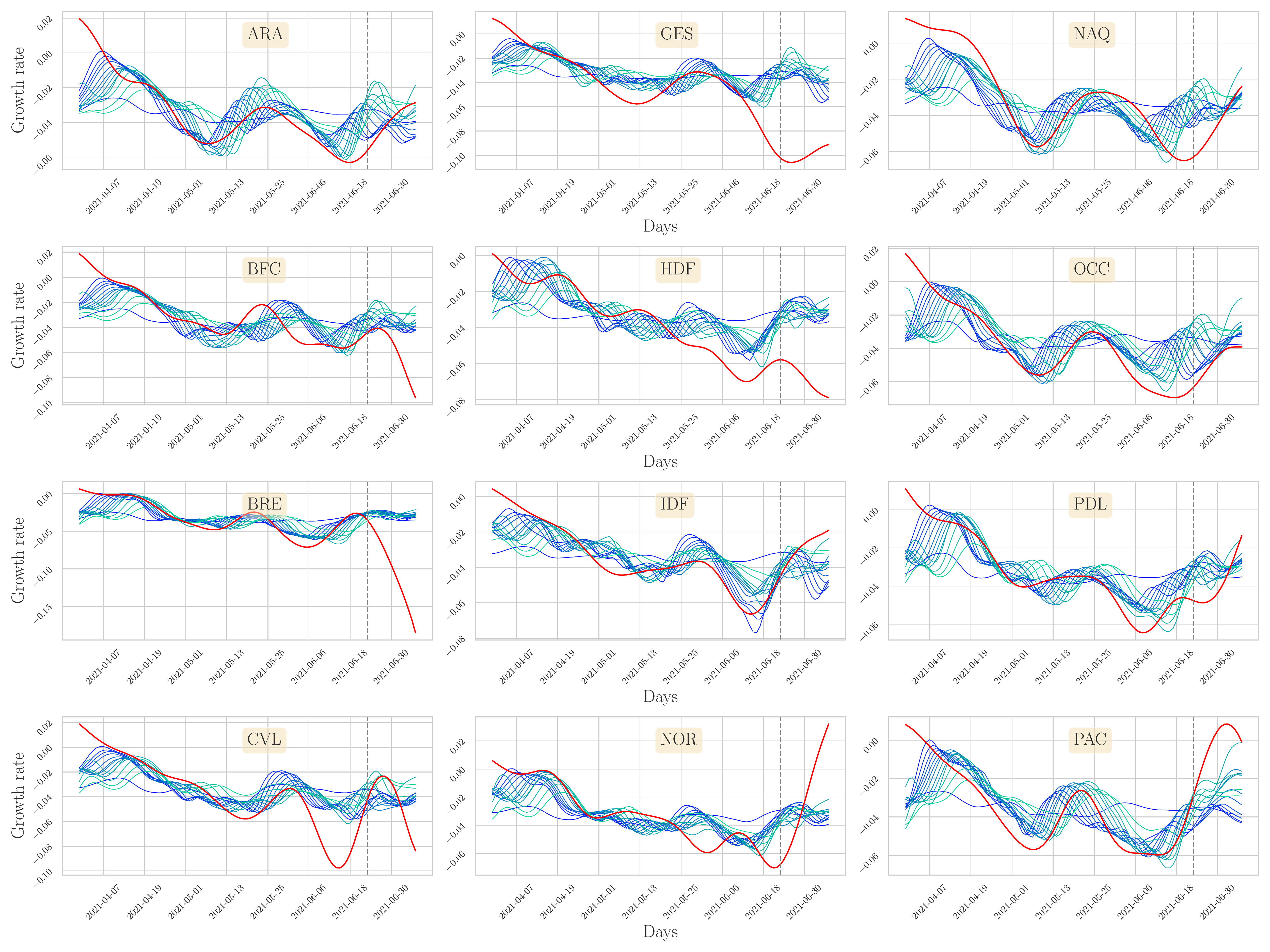}
    \caption{Interpolation (left of dotted line) and prediction (right of dotted line) of hospitalization growth rate for all 12 french regions using \textbf{SigLasso}.}
    \label{fig:covid_display_weighted_siglasso}
\end{figure}

\begin{figure}
    \centering
    \includegraphics[width=0.78\textwidth]{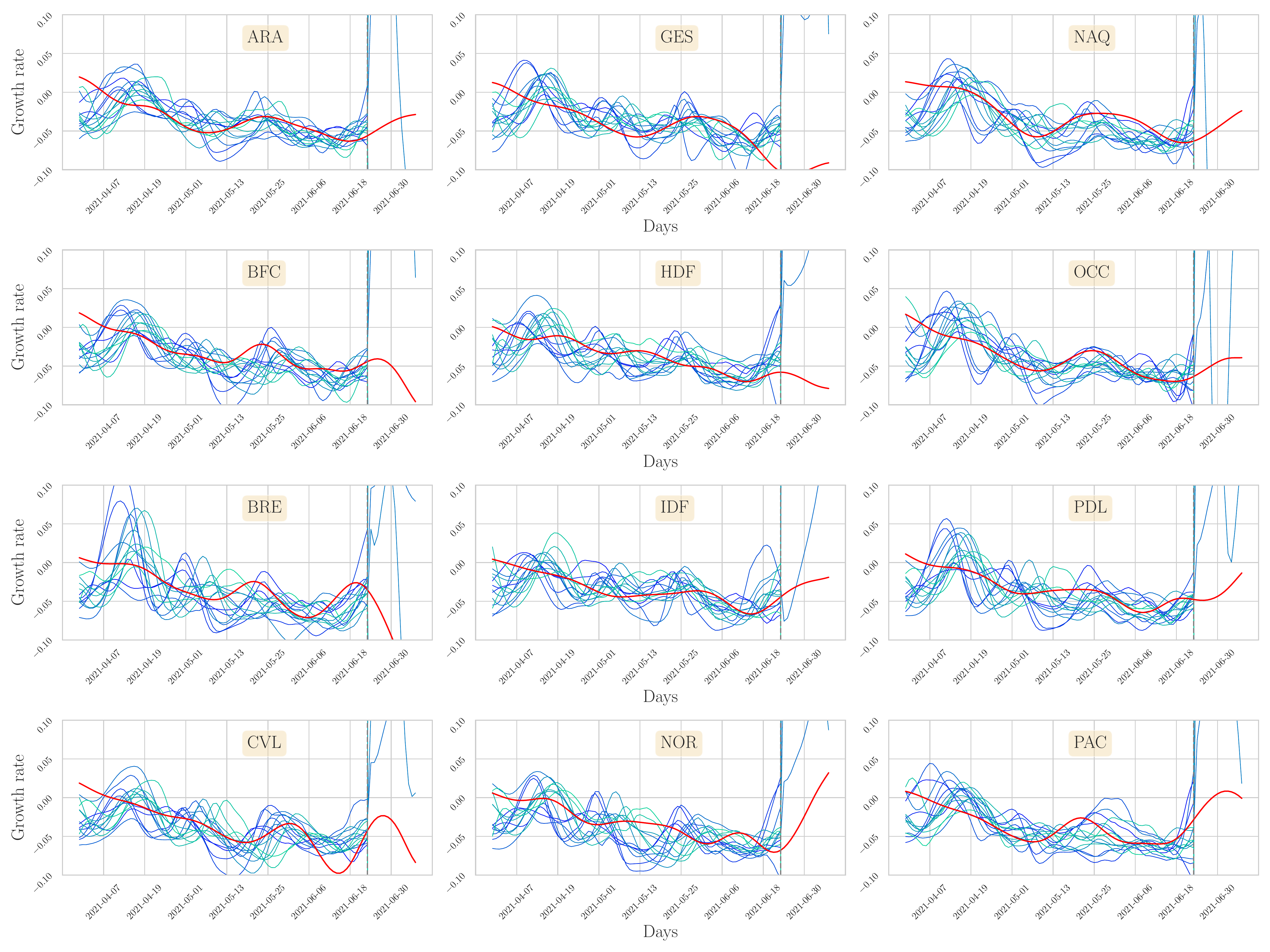}
    \caption{Interpolation (left of dotted line) and prediction (right of dotted line) of hospitalization growth rate for all 12 french regions using \textbf{NCDE}.}
    \label{fig:covid_display_ncde}
\end{figure}

\begin{figure}
    \centering
    \includegraphics[width=0.78\textwidth]{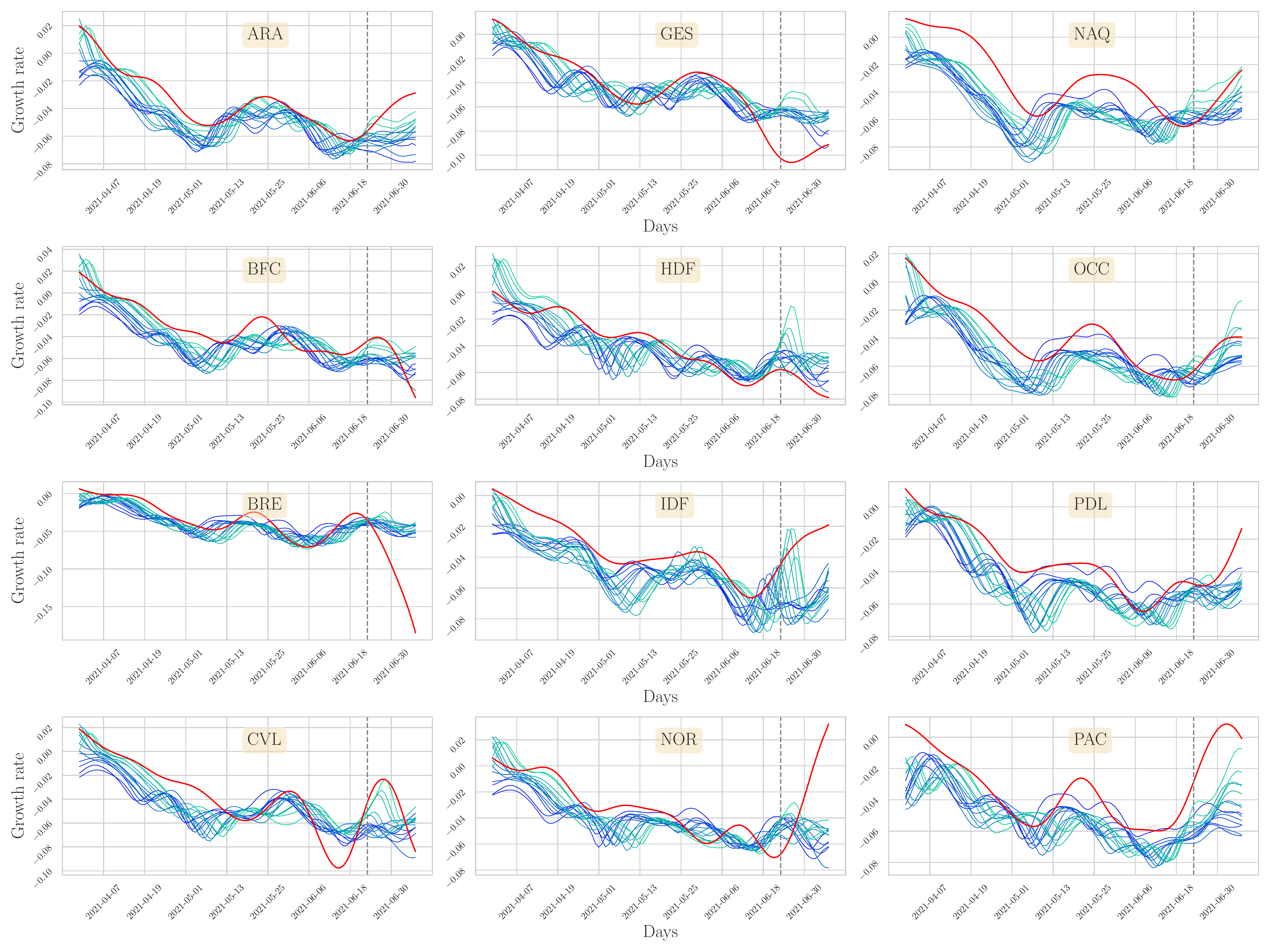}
    \caption{Interpolation (left of dotted line) and prediction (right of dotted line) of hospitalization growth rate for all 12 french regions using \textbf{GRU}.}
    \label{fig:covid_display_gru}
\end{figure}

\begin{figure}
    \centering
    \includegraphics[width=0.78\textwidth]{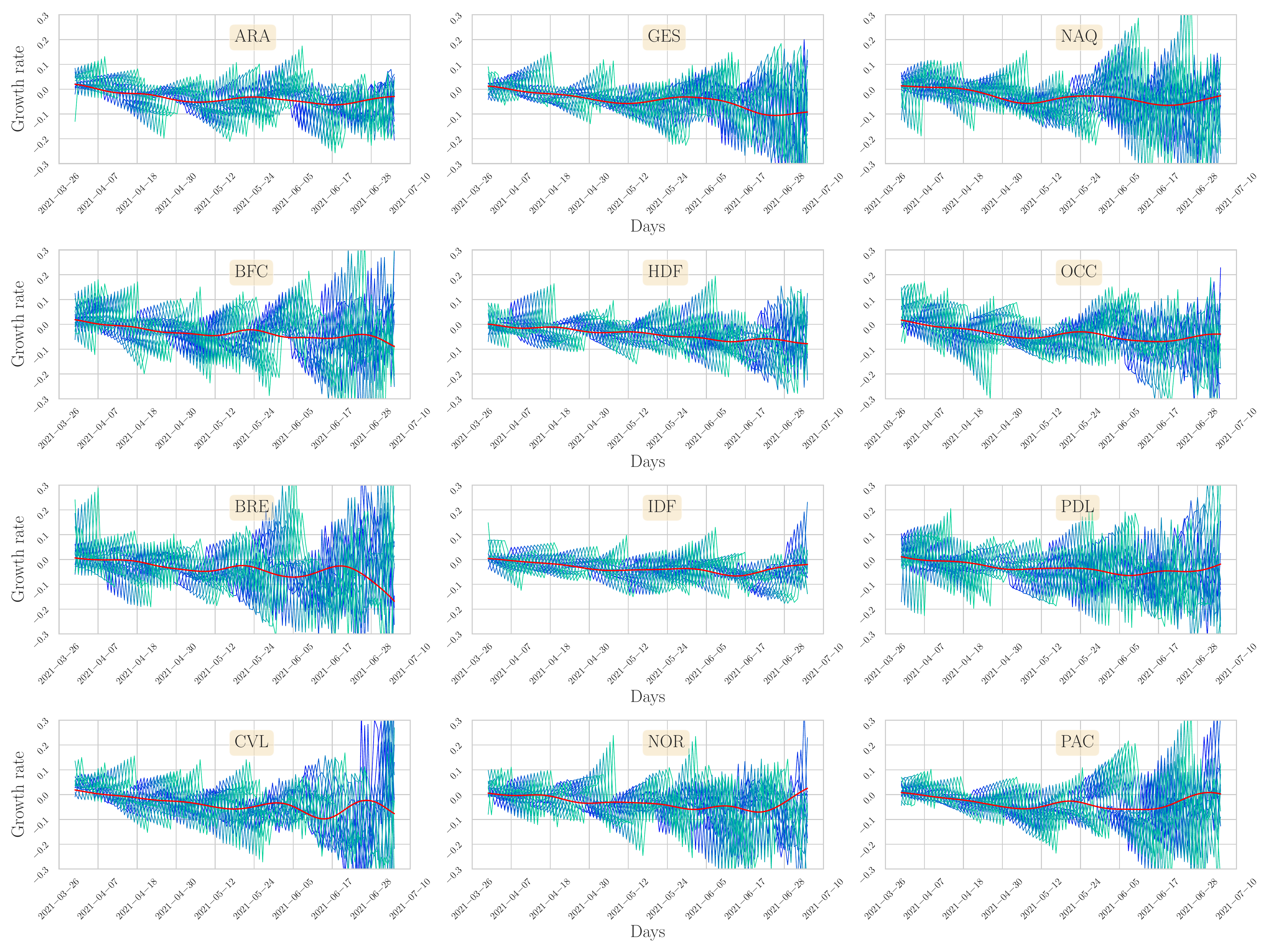}
    \caption{Interpolation (left of dotted line) and prediction (right of dotted line) of hospitalization growth rate for all 12 french regions using \textbf{ensemble methods} \citep{paireau2022ensemble}.}
    \label{fig:covid_display_pnas}
\end{figure}

\begin{figure}
    \centering
    \includegraphics[width=0.78\textwidth]{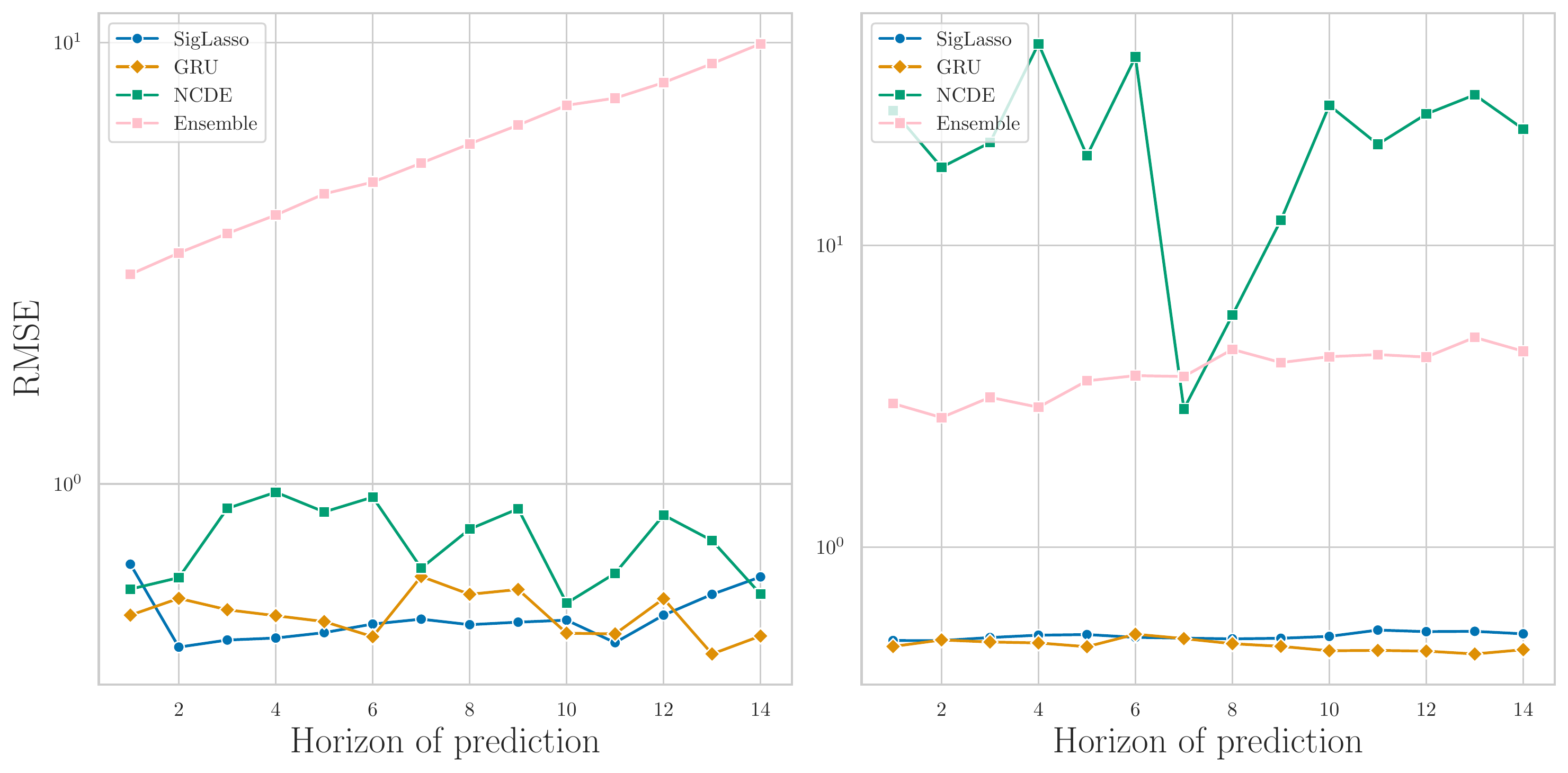}
    \caption{RMSE accross all regions on the \textbf{training period} (left) and the \textbf{testing period} (right) for the ensemble methode \citep{paireau2022ensemble}, NCDE, GRU, and SigLasso. See Figure \ref{fig:RMSE_covid_siglasso_gru} for a zoom-in on GRU and SigLasso performances.}
    \label{fig:RMSE_covid_all_models}
\end{figure}

\begin{figure}
    \centering
    \includegraphics[width=0.78\textwidth]{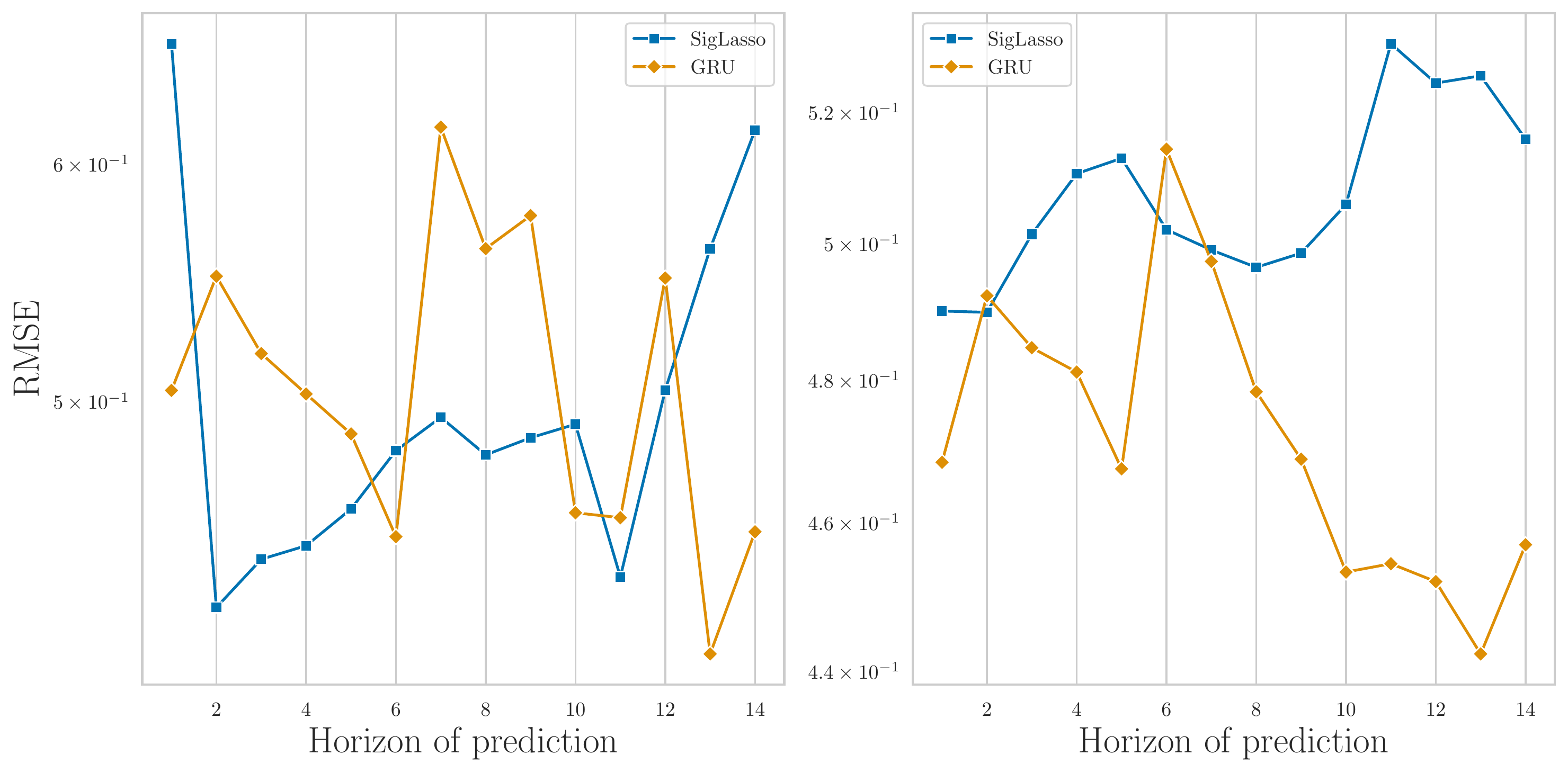}
    \caption{RMSE accross all regions on the \textbf{training period} (left) and the \textbf{testing period} (right) for GRU and SigLasso.}
    \label{fig:RMSE_covid_siglasso_gru}
\end{figure}

\subsection{Additional results}
\label{appendix:additional_results}

We give in Table \ref{tab:apx_time} some additional results on the experiments described above.

\begin{table*}[h!]
	\centering
	\caption{Training time of SigLasso, GRU and Neural CDE in different simulation settings, averaged over 10 iterations. In every setting, $n=50$,  $\# \bar{D}^i = 5$ for all $i=1, \dots, n$ (and therefore $M=250$).}
	\label{tab:apx_time}
	\begin{tabular}{lccc}
	\toprule
    & \multicolumn{3}{c}{Training time (s)} \\
    \cmidrule(lr){2-4}
    &  SigLasso & GRU & Neural CDE \\
    \midrule 
    Well-specified & 0.37 $\pm $ 0.23 &  269 $\pm $ 109 & 1754 $\pm$ 587 \\
    OU &  0.057 $\pm$ 0.005 & 27 $\pm$ 0.44 & 216 $\pm$ 2.7 \\
    Tumor growth & 0.056 $\pm$ 0.007 & 31 $\pm$ 3.5 & 250 $\pm$ 14 \\
	\bottomrule           
	\end{tabular}
\end{table*}

Moreover, we show in Figure \ref{fig:sampling_study_x_smooth} the results of the $L_2$ reconstruction error in the well-specified setting, when we vary the number of sampling points of the feature paths between 10 and $10^3$. We see that SigLasso always outperforms GRU and Neural CDE but that the difference of performance is more important when there are only a few sampling points. In this regime SigLasso is moreover more stable.

\begin{figure}[h]
    \centering
    \includegraphics[width=0.4\textwidth]{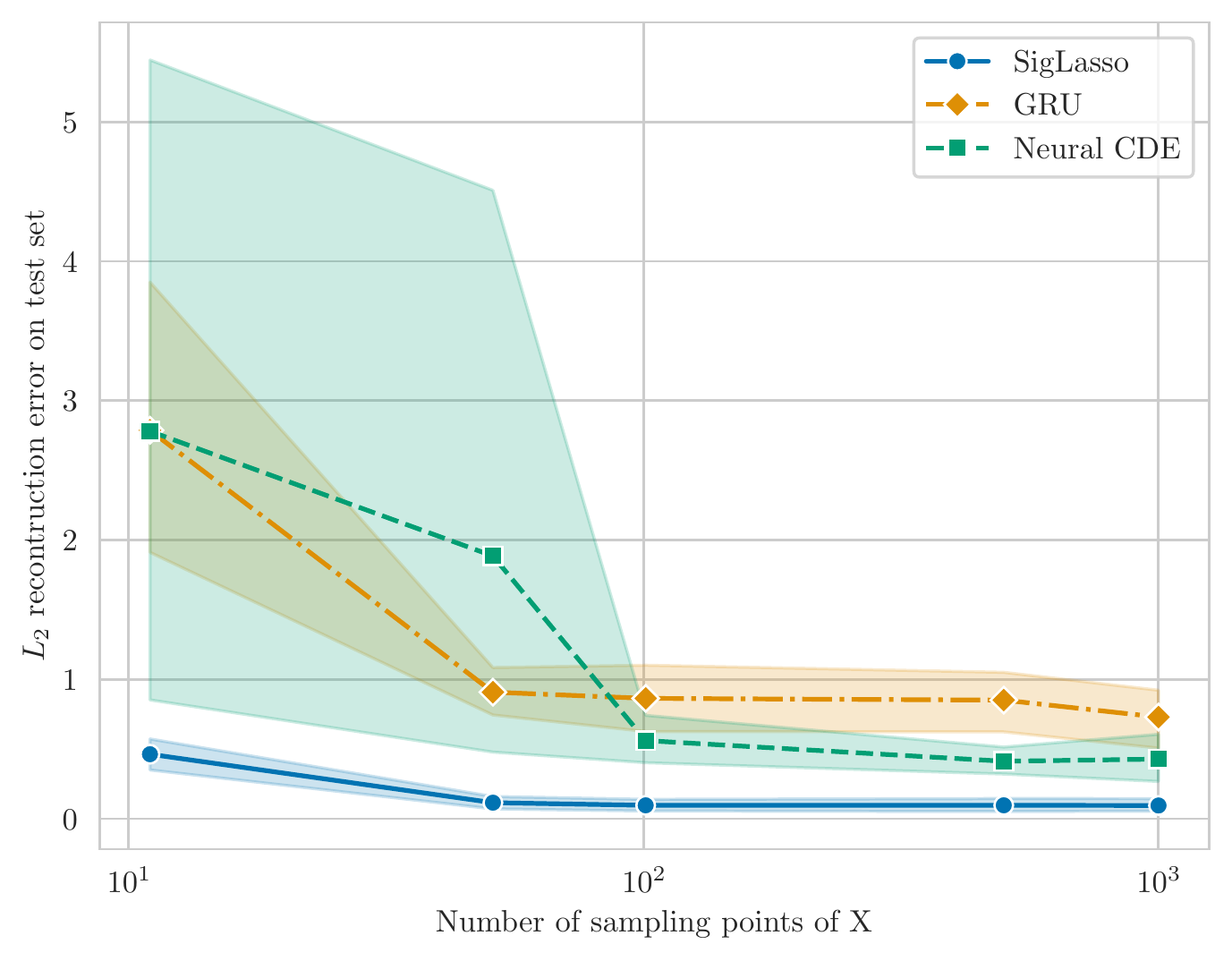}
    \caption{$L_2$ reconstruction error of SigLasso, GRU and Neural CDE in the well-specified setting, for varying number of feature samples.}
    \label{fig:sampling_study_x_smooth}
\end{figure}

\end{document}